\documentclass{article} % For LaTeX2e
\usepackage{iclr2026_conference,times}

\usepackage{amsmath,amsfonts,bm}

\def\eqref#1{equation~\ref{#1}}
\def\1{\bm{1}}

\DeclareMathAlphabet{\mathsfit}{\encodingdefault}{\sfdefault}{m}{sl}
\SetMathAlphabet{\mathsfit}{bold}{\encodingdefault}{\sfdefault}{bx}{n}

\newcommand{\E}{\mathbb{E}}

\DeclareMathOperator*{\argmax}{arg\,max}

\usepackage{hyperref}
\usepackage{booktabs,array}
\usepackage{multirow}
\usepackage{subcaption}
\usepackage{comment}
\usepackage{xcolor}
\usepackage[accsupp]{axessibility}  % Improves PDF readability for those with disabilities.
\usepackage{pifont}

\newcommand{\cmark}{\textcolor{teal}{\ding{51}}}%
\newcommand{\xmark}{\textcolor{purple}{\ding{55}}}%

\def\etal{\textit{et al.}\xspace}
\def\default{\texttt{default}\xspace}
\def\CIFARH{\texttt{CIFAR100}\xspace}
\def\CIFART{\texttt{CIFAR10}\xspace}

\def\TINY{\texttt{TinyImageNet}\xspace}
\def\TINYLT{\texttt{TinyImageNet-LT}\xspace}

\def\imageNetLT{\texttt{ImageNet-LT}\xspace}
\def\smce{\texttt{smCE}\xspace}
\def\refcal{\textit{RefCal}\xspace}
\def\CIFARHLT{\texttt{CIFAR100-LT}\xspace}
\def\CIFARTLT{\texttt{CIFAR10-LT}\xspace}
\def\CIFARHC{\texttt{CIFAR100-C}\xspace}
\def\auroc{\texttt{AUROC}\xspace}
\def\auc{\texttt{AUC}\xspace}
\newcommand{\mypara}[1]{\vspace{0.2em}\myfirstpara{#1}}

\def\AUCM{\texttt{AUCM}\xspace}
\def\Logitnorm{\texttt{LogitNorm}\xspace}
\def\CIFARH{\texttt{CIFAR100}\xspace}
\def\CIFART{\texttt{CIFAR10}\xspace}

\def\TINY{\texttt{TinyImageNet}\xspace}
\def\CIFARTLT{\texttt{CIFAR10-LT}\xspace}
\def\TINYLT{\texttt{TinyImageNet-LT}\xspace}
\def\CIFARTLT{\texttt{CIFAR10-LT}\xspace}
\def\CIFARHLT{\texttt{CIFAR100-LT}\xspace}
\def\STLT{\texttt{STL10}\xspace}
\def\imageNetLT{\texttt{ImageNet-LT}\xspace}
\def\sbf{\mathbf{s}}

\def\sbf{\mathbf{s}}

\def\bx{\mathbf{x}}

\def\sbf{\mathbf{s}}

\def\sh{\widehat{s}}
\def\X{{\mathcal{X}}}
\def\Y{{\mathcal{Y}}}
\def\D{{\mathcal{D}}}

\definecolor{LightCyan}{rgb}{0.88,1,1}
\definecolor{lightgrey}{rgb}{0.01,0.199,0.1}
\definecolor{cambridgeblue}{rgb}{0.64, 0.76, 0.68}

\def\mmce{\texttt{MMCE}\xspace}
\def\refcal{\texttt{RefCal}\xspace}
\def\dnn{\texttt{DNN}\xspace}
\def\dnns{\texttt{DNNs}\xspace}
\def\ece{\texttt{ECE}\xspace}
\def\ace{\texttt{ACE}\xspace}
\def\sce{\texttt{SCE}\xspace}
\def\smce{\texttt{smECE}\xspace}

\def\sota{\texttt{SOTA}\xspace}
\def\auroc{\texttt{AUROC}\xspace}
\def\roc{\texttt{ROC}\xspace}
\def\sota{\texttt{SOTA}\xspace}
\def\NLL{\texttt{NLL}\xspace}
\def\LS{\texttt{LS}\xspace}

\def\crl{\texttt{CRL}\xspace}
\def\ood{\texttt{OOD}\xspace}

\def\E{\mathbb{E}}
\def\X{\mathcal{X}}
\def\I{\mathcal{I}}
\def\Y{\mathcal{Y}}
\def\D{\mathcal{D}}

\def\P{\mathbb{P}}

\def\lsc{\mathcal{L}_{\text{SC}}}
\def\lref{\mathcal{L}_{\text{ref}}}
\def\eip{\text{exp}(z_i \cdot z_p)}
\def\eia{\text{exp}(z_i \cdot z_a)}
\def\ein{\text{exp}(z_i \cdot z_n)}

\def\sbf{\mathbf{s}}

\def\X{{\mathcal{X}}}
\def\D{{\mathcal{D}}}

\def\etal{{\emph{et al.}\xspace}}
\def\ood{{OOD\xspace}}

\def\id{{ID\xspace}}

\usepackage{enumitem}

\def\dnn{\textsc{dnn}\xspace}
\def\dnns{\textsc{dnn}s\xspace}

\def\sota{\textsc{sota}\xspace}

\def\sce{\textsc{sce}\xspace}
\def\ece{\textsc{ece}\xspace}

\def\ts{\textsc{ts}\xspace}
\def\ood{\textsc{ood}\xspace}

\def\mmce{\textsc{mmce}\xspace}

\def\sbf{\mathbf{s}}

\usepackage{amsthm}
\newcommand{\myfirstpara}[1]{\noindent \textbf{#1.}}
\usepackage{amsmath,amsfonts,bm}
\usepackage{mathtools}
\usepackage{amsmath,xspace}
\usepackage{amsthm}

\theoremstyle{definition}
\newtheorem{definition}{Definition}

\def\yh{\widehat{y}}

\definecolor{COLOR_ZS}{HTML}{d8ebf8}

\definecolor{maroon}{cmyk}{0,0.87,0.68,0.32}
\definecolor{ashgrey}{rgb}{0.7, 0.75, 0.71}
\definecolor{cadetblue}{rgb}{0.37, 0.62, 0.63}
\definecolor{cambridgeblue}{rgb}{0.64, 0.76, 0.68}
\definecolor{ceil}{rgb}{0.57, 0.63, 0.81}
\definecolor{darkseagreen}{rgb}{0.56, 0.74, 0.56}
\definecolor{junebud}{rgb}{0.74, 0.85, 0.34}

\definecolor{maroon}{cmyk}{0,0.87,0.68,0.32}
\definecolor{ashgrey}{rgb}{0.7, 0.75, 0.71}
\definecolor{darkseagreen}{rgb}{0.56, 0.74, 0.56}
\definecolor{junebud}{rgb}{0.74, 0.85, 0.34}

\usepackage{url}
\usepackage{amsthm}
\newtheorem{lemma}{Lemma}
\usepackage{tcolorbox}
\usepackage{wrapfig}
\usepackage{natbib}

\title{Enhancing Deep Neural Network Reliability with Refinement and Calibration}

\author{Ramya Hebbalaguppe\thanks{ Contact: ramya.murthy@gmail.com}  \quad Ajay Shastry \quad Soumya Suvra Ghosal \quad Chetan Arora \\
SIT, Indian Institute of Technology Delhi, New Delhi, India\\
\textbf{Project Webpage:} \url{https://github.com/rhebbalaguppe/RefCal}\\
}

\iclrfinalcopy % Uncomment for camera-ready version, but NOT for submission.
\begin{document}

\maketitle

\begin{abstract}
Although deep neural networks (\dnns) achieve high predictive accuracy, their confidence estimates are often unreliable, potentially compromising user trust in their decisions. This has driven research into \texttt{calibrated} models— where \texttt{calibration} measures the degree to which a model's predictive confidence matches the empirical probability of correctness. However, such a calibration metric can be changed by post-processing output to mimic the uncertainty of training time, without truly improving the model’s understanding. Hence, statisticians recommend a model to be \texttt{refined} as well as calibrated. Intuitively, a model is said to be more refined if there is a large difference in its predictive confidence for correct and incorrect predictions (sometimes called \texttt{sharpness}). We observe that common calibration methods often reduce model refinement. To address this, we propose: \textbf{(1)} a novel loss function that promotes refinement and is optimizable via supervised contrastive loss; \textbf{(2)} a unified training framework, \refcal, that jointly optimizes for calibration, refinement, and accuracy—enhancing DNN reliability. For eg., we report (accuracy$\uparrow$, refinement$\uparrow$, ECE$\downarrow$) of (58.81, 95.67, 0.08) on CIFAR-100-LT dataset (10\% class imbalance), surpassing (46.27, 93.7, 0.22) by the well known Correctness Ranking Loss. 
% Although such models may achieve superior performance metrics, their tendency to produce overconfident predictions during failure cases can substantially diminish user trust in the system.
%
%We observe that popular calibration techniques end up making a model less-refined. To this end, the primary contributions of this paper are: \textbf{(1.)} proposing a novel loss function to enforce refinement, which, we show mathematically, can be optimized using popular supervised contrastive learning loss; 
%\textbf{(2.)} This allows the proposal of a novel training framework that we term \refcal that optimizes for calibration, refinement, and accuracy, making a \dnn reliable. We observe that common calibration methods often reduce model refinement. To address this, we propose: (1) a novel loss function that promotes refinement and is optimizable via supervised contrastive loss; (2) a unified training framework, \refcal, that jointly optimizes for calibration, refinement, and accuracy—enhancing DNN reliability.
\end{abstract}

\section{Introduction}
\label{sec:intro}

%\myfirstpara{Generating Trust on an AI Model}
% 
Advances in datasets, model architectures, and computational resources have made it easier to achieve high accuracy with deep neural networks (\dnns). Consequently, research has shifted toward improving complementary performance metrics that enhance model trustworthiness, particularly reliability. This is crucial in safety-critical domains such as healthcare and autonomous vehicles, where models must not only be accurate but also properly quantify and communicate predictive uncertainty—exhibiting high confidence in correct predictions while reflecting uncertainty in cases of error. In this work, we examine reliability through both, calibration and refinement. Calibration measures how well predicted probabilities match true outcome frequencies, while refinement captures the sharpness of predictions. We formally define both dimensions and argue that both are essential for robust and trustworthy decision-making in \dnns.

\mypara{Calibration}
We consider a K-class classification problem, with $X \in \X$ as the input, $Y \in \Y$ as the target random variable, and $f(X)$ as the predicted confidence vector from a \dnn based classification model, $f$. Let $\mathcal{P}$ as the set of distributions on $\mathcal{Y}$. Since, our focus is on classification, we use $\mathcal{P}_K$ to denote the K-dimensional simplex of corresponding categorical distributions. We use $\mathbb{P} \in \mathcal{P}_K$, and $\mathbb{P}_{Y|f(X)} \in \mathcal{P}_K$ to denote the distribution of Y, and conditional distribution given X, respectively.

\begin{definition}[Calibration \cite{gruber2022better}]
	\label{defn:calibration}
	A model $f : \mathcal{X} \rightarrow \mathbb{P}$ is calibrated, if and only if, $f(X) = \mathbb{P}_{Y|f(X)}$.
\end{definition}
\noindent An alternative formulation focuses on calibrating only the \emph{top-label} prediction of a model \( f : \mathcal{X} \rightarrow \mathcal{P}_K \). Let \( C = \arg\max_{k} f_k(X) \), denote the predicted class with the highest confidence (top-label softmax probability). In this setting, the calibration condition is expressed as:
\[
f_C(X) = \mathbb{P}(Y = C \mid f_C(X)),
\]
which states that the confidence assigned to the top-label \( C \) should match the true probability of \( C \) being the correct class. %In other words, the model’s top-label confidence should reflect the empirical accuracy of its top prediction.

~\cite{guo2017calibration} showed that modern \dnns often produce overconfident but incorrect predictions, prompting various efforts to improve their calibration~\cite{platt1999probabilistic,Dirichlet,kumar2019verified,kumarpaper}. Train time methods typically reduce over/under-confidence \cite{kumarpaper}. %Consider a binary classification scenario (e.g., malignant vs. benign tumor) with a validation confusion matrix of $\{\{0.7, 0.3\}, \{0.2, 0.8\}\}$. On the test set, suppose we post-process the outputs such that all malignant predictions are assigned a confidence vector of $\{0.7, 0.3\}$ and benign predictions get $\{0.2, 0.8\}$. This model would appear well-calibrated, as predicted confidences match observed frequencies.

\newtcolorbox{insightbox}[1]{colback=gray!5,colframe=gray!50,title={#1}}

\begin{insightbox}{Pitfalls of Calibration Metrics}
In a binary classification setting with validation confusion matrix 
$\{\{0.7, 0.3\}, \{0.2, 0.8\}\}$, one can post-process test predictions by assigning a fixed confidence vector $\{0.7,0.3\}$ to all malignant predictions and $\{0.2,0.8\}$ to all benign ones. This yields good calibration scores since predicted confidences match observed frequencies. However, this uniform assignment reduces confidence differentiation between true and false positives, thereby weakening the model’s discriminative power while preserving accuracy. Post-hoc calibration can improve calibration metrics at the cost of predictive confidence differentiation.
\end{insightbox}

%\begin{figure*}[t]
%	\begin{minipage}{0.45\linewidth}
%	\centering
%	\includegraphics[width=\linewidth] {figs_plots/refcalteaser_final.pdf}
%	\end{minipage}
%	\quad
%	\begin{minipage}{0.5\linewidth}
%	\caption{\textbf{[Why RefCal?]}: Our proposed training regime, \refcal optimizes both, \textbf{Ref}inement and \textbf{Cal}ibration. \textbf{(a)} shows when we calibrate a \texttt{ResNet-50} model using \mmce\cite{kumarpaper} calibration on \CIFARHLT, it results in lower separation between the confidence values of the correct and incorrect predicted classes (red and blue). To this end, when we use the proposed proxy refinement loss along with calibration (For eg., using \mmce), it can be seen that the separability of correct and incorrect predictions is enhanced as shown in \textbf{(b)}, indicating better refinement. Quantitatively, (\auc($\uparrow$), \ece($\downarrow$)) for (refinement, calibration) improve from ($94.40$, $0.25$) for \mmce to ($95.03$, $0.07$) for \mmce + proposed refinement loss. \textbf{(c)} We provide Grad-CAM visualizations for \texttt{Resnet-18} trained on \texttt{ImageNet-LT}, using a particular calibration technique (each row, second column), and then by optimizing with the same calibration but adding our refinement loss (each, third column). As refinement loss forces a model to separate its confidence for correct, and incorrect predictions, it also leads the model to focus its attention on the salient object features, instead of the background.}
%	\label{fig:teaser}
%	\end{minipage}
%\end{figure*}

\begin{figure}[t]
    \centering
	\includegraphics[width=\linewidth]{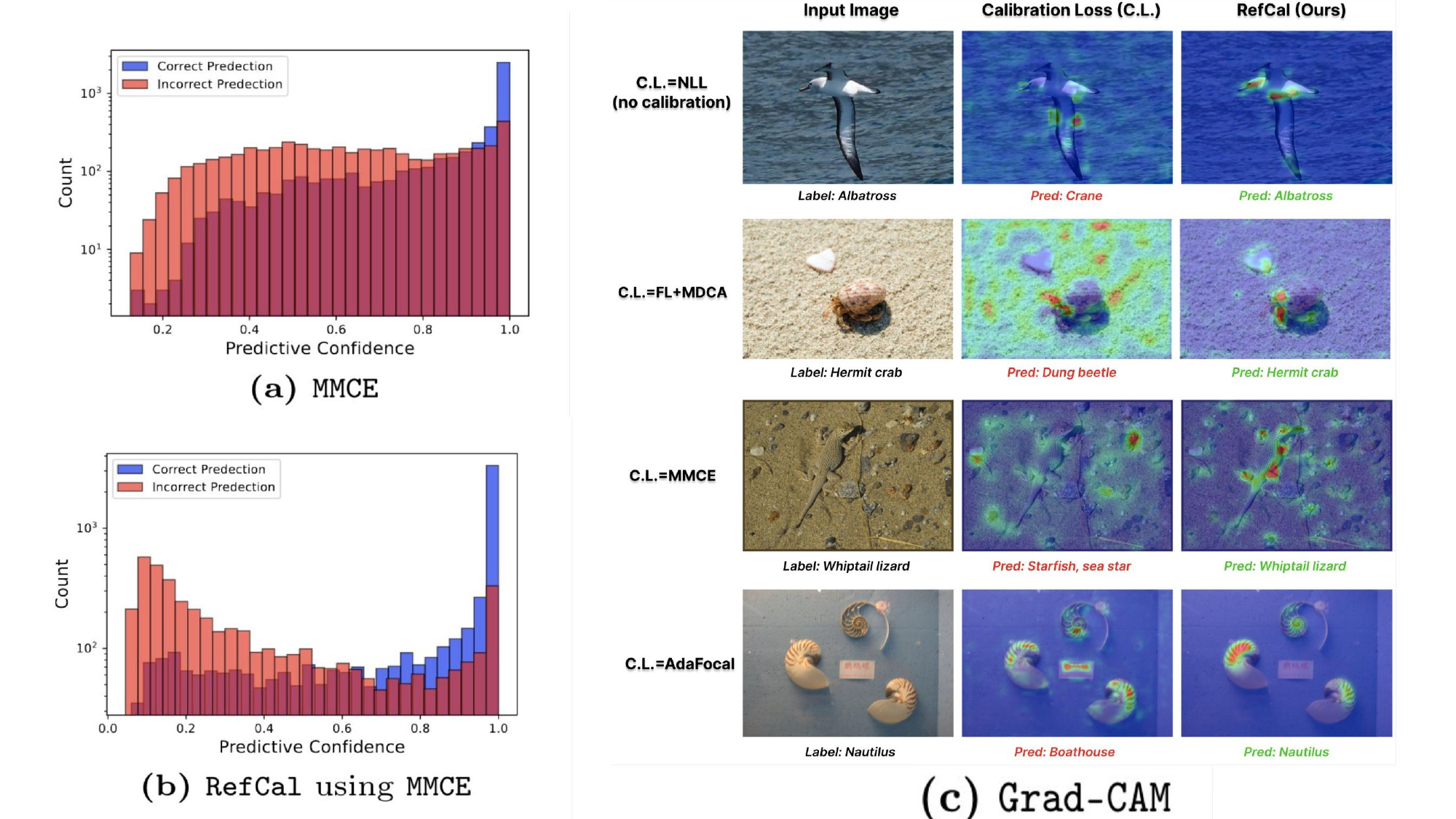}
    \vspace{-2em}
    %\includegraphics[width=.8\textwidth]{sections/RefCal_part1_teaser.pdf}
    %\caption{\mmce with and without \refcal}\label{one}
    %\bigbreak
    %\includegraphics[width=.8\textwidth]{sections/RefCal_part2_teaser.pdf}
    %\caption{GradCAM Viz.}\label{two}
	\caption{[Why \refcal?]: Our proposed training regime, \refcal optimizes both, \textbf{Ref}inement and \textbf{Cal}ibration. \texttt{(a)} shows when we calibrate a \texttt{ResNet-50} model using \mmce\cite{kumarpaper} calibration on \CIFARHLT, it results in lower separation between the confidence values of the correct and incorrect predicted classes (red and blue). To this end, when we use the proposed proxy refinement loss along with calibration (For eg., using \mmce), it can be seen that the separability of correct and incorrect predictions is enhanced as shown in \texttt{(b)}, indicating better refinement. Quantitatively, (\auc($\uparrow$), \ece($\downarrow$)) for (refinement, calibration) improve from ($94.40$, $0.25$) for \mmce to ($95.03$, $0.07$) for \mmce + proposed refinement loss. \texttt{(c)} We provide Grad-CAM visualizations for \texttt{Resnet-18} trained on \texttt{ImageNet-LT}, using a particular calibration technique (each row, second column), and then by optimizing with the same calibration but adding our refinement loss (each, third column). As refinement loss forces a model to separate its confidence for correct, and incorrect predictions, it also leads the model to focus its attention on the salient object features, instead of the background.}
    \label{fig:teaser}
\end{figure}
%Note that the model was expected to make $20\%$ mistakes on the test set for \emph{malignant tumour} (assuming same training/test distribution), but by outputting $0.8$ confidence for all predicted \emph{malignant tumour} samples it made sure to come good on calibration metric. However, this reduced the difference in predicted confidence for the \emph{malignant tumour} class, for the actual \emph{malignant tumour} (positive), and the \emph{benign tumour} (negative) samples. Note that such post-processing does not change the accuracy of the model; it simply trades off improved calibration for lower differences in predicted confidence vectors. However, such a trade-off makes it very hard for an end-user to trust a model's predictions.
% (all at $0.8$ now)

\mypara{Refinement}
%
%The ability of a predictive model to cheat calibration metrics has been known atleast since 1982, with 
%\texttt{Refinement}, a proposed strategy dates back to 1982 \cite{degroot1982assessing}. Refinement encourages a model to bring its predictive confidence closer to 0 and 1, which helps increase the difference between correctly and incorrectly predicted samples % , a concept 
Refinement was introduced by DeGroot and Fienberg in 1982, promotes sharper predictive distributions by encouraging confidence scores to be closer to 0 or 1  ~\cite{degroot1982assessing}. This behavior enhances the model’s ability to distinguish between correct and incorrect predictions, thereby improving its discriminative power(See Fig. \ref{fig:teaser}). The refinement trade-off was rediscovered recently, when \cite{moon2020confidence} proposed \emph{correctness ranking loss} (\crl)  to improve refinement by enforcing the following relationship for every pair of samples $(x_i,y_i)$ and $(x_j,y_j)$:
\begin{equation}
c_i \leq c_j \iff \P(\hat{y}_i=y_i \mid x_i) \leq \P(\hat{y}_j=y_j \mid x_j).
\end{equation}
Here $y$ denotes the target label, and $\hat{y}$ predicted label, for a sample $x$, and the predictive confidence, $c$. Recently, ~\cite{yuan2021large,yang2021learning} proposed a differentiable loss function that directly maximizes area under the \roc curve (\auc). They showed that for a two-class classification: 
\begin{equation}
	\auc = \mathbb{E}[\mathbb{I}(c_i > c_j)], \; \text{s.t.} \; x_i \in \mathcal{S}_p, \text{and } x_j \in \mathcal{S}_n,
	\label{eq:auroc}
\end{equation}
where $\mathcal{S}_p$ and $\mathcal{S}_n$ denote the set of correctly and incorrectly classified samples of a model on the test set. The result indicates that \auc can be used as metric to score a model's refinement performance. %Further, \cite{yuan2021large} proposed a new \auc margin loss as a proxy to maximize \auc (and improve refinement). 
% In this paper we propose a novel loss to improve refinement, and a technique for its efficient optimization. 
%show that it can be simply minimized using a version of popular supervised contrastive loss \cite{khosla2020supervised}. 
We show the improvement in refinement using our proposal with \auc as the metric.

\mypara{Contributions} The main contributions of this work are:

\begin{itemize}[leftmargin=*, itemsep=2pt, topsep=2pt]
    \item \myfirstpara{A surrogate loss for refinement} We propose a novel surrogate loss function that promotes refinement by encouraging sharper predictive distributions. We show, through mathematical analysis, that this loss can be minimized using the supervised contrastive loss~\cite{khosla2020supervised}, enabling the reuse of existing contrastive learning frameworks for refinement. On the STL10 dataset (in binary setting; see Tab. \ref{tab:bin}), our method achieves an AUC of 99.90, outperforming the state-of-the-art score of 98.86~\cite{yuan2021large}. \myfirstpara{Scalable refinement for multi-class classification} Existing refinement or AUC-based optimization approaches are typically limited to binary classification and require one-vs-all schemes for multi-class settings. In contrast, our formulation—based on supervised contrastive loss—extends naturally to multi-class classification. %On the imbalanced \CIFARH dataset (10\% minority class), our method achieves an AUC of 96.62 compared to 93.00 reported in~\cite{liu2022devil}.
    
    \item \myfirstpara{\refcal: A two-stage framework for  optimization of reliability objectives} We introduce \refcal, a two-stage training procedure that first optimizes for refinement using supervised contrastive learning while the second stage fine-tunes the model using standard calibration and accuracy losses. On the \CIFARTLT benchmark, \refcal achieves the best combined performance: \{ECE: 1.05, AUC: 99.04, Accuracy: 88.11\}, compared to the SOTA baseline~\cite{liu2022devil} with \{ECE: 7.05, AUC: 98.59, Accuracy: 87.82\}.
    
    \item \myfirstpara{Revisiting the calibration–refinement trade-off for \dnn classifiers} Prior literature~\cite{murphy1973new,singh2021dark} interprets Brier score decomposition to imply a trade-off between calibration and refinement. We clarify that this conflict arises only under the assumption of constant error. Our empirical results show that refinement can improve alongside calibration and accuracy (see Fig.~\ref{fig:teaser}). 
    %We will release code, models, and evaluation protocols upon acceptance.
\end{itemize}

\section{Related work}
\label{sec:relWork}

\myfirstpara{Refinement}
% a separate module, whereas 
Refinement in literature is studied in two different styles. The first one aims at the separability of correct and incorrect predictions with a margin based on the predicted confidence and the second tries to optimize the ordinal ranking relationship between correctly classified an incorrectly classified samples. 
%A prevalent approach for assessing whether to trust a classifier's judgment entails considering the confidence score reported by the classifier itself. This confidence score could encompass probabilities derived from a neural network's softmax layer, the proximity to a separating hyperplane in support vector classification, or the average class probabilities for decision trees in a random forest. Although using a model's own implied confidences seems logical, research has demonstrated that the raw confidence values provided by a classifier are often inaccurately calibrated \cite{guo2017calibration} and insufficient for reliable classification. To address this, Jiang \etal \cite{jiang2018trust} suggest using a trust score that assesses the agreement between the classifier and a modified nearest-neighbor classifier based on the testing example.  
\crl \cite{moon2020confidence} is a popular technique in the first category which learns ordinal relationship by introducing margin-based penalties. %ACLS \cite{park2023acls} show that \mdca and \cpc would only reduce the label of true class slightly when output predicted probability is exceedingly high. 
In the second style for refinement, researchers have exploited its relationship with \auc (Eq. \ref{eq:auroc}), and developed techniques to directly optimize \auc \cite{yang2022algorithmic}, or its proxy \cite{yuan2021large}. Margin maximization seems theoretically and practically superior, as it facilitates generalization error analysis and presents a clear geometric interpretation of the models being built. The resulting techniques have also been applied to large-scale medical image datasets \cite{yang2022optimizing}. However, the extensions to multi-class do not appear obvious, and calibration in conjunction with refinement has not been investigated systematically.

% CVPR review: added new references (1) Optimizing Two-way Partial AUC with an End-to-end Framework (2) 

% unsophisticated approach involves employing a surrogate loss that operates in pairs, relying on the AUC score's definition. Nevertheless, attempting to optimize such a general pairwise loss on training data encounters a significant scalability problem, rendering it impractical for deep learning on extensive datasets.

%They presented a sufficient condition for the solutions of regularized loss functions to converge to margin-maximizing separators as the regularization vanishes. 

\mypara{Calibration}
% Platt scaling [30], histogram binning [40], isotonic regression [41], and Bayesian binning and averaging [22, 21]. Temperature scaling has the desirable property that it can improve the calibration of a network without in any way affecting its accuracy. However, whilst its simplicity and effectiveness. have made it a popular network calibration method, it does have downsides. For example, whilst it scales the logits to reduce the network’s confidence in incorrect predictions, this also slightly reduces the network’s confidence in predictions that were correct [16] -- Aviral Kumar, Sunita Sarawagi, and Ujjwal Jain. Trainable Calibration Measures For Neural Networks From Kernel Mean Embeddings. In ICML, 2018
Train time loss functions such as \texttt{FL} \cite{mukhoti2020calibrating}, \texttt{MMCE} \cite{kumarpaper} \cite{Patra_2023_WACV}\cite{hebbalaguppe2024calibration}\cite{PnPOOD} \cite{relcal} \cite{hebbalaguppe2025prompting} etc., aim to alleviate the miscalibration in \dnns but do not address refinement. On the other hand, post-hoc calibration techniques \cite{local-TS,labelsmoothinghelp,Dirichlet}, tackle miscalibration by adding a parameterized calibration component to a \dnn, which can be fine-tuned using a separate validation set. Recent calibration metrics include \cite{unifying-theory-distance,gruber2022better,nixon2019measuring}. Post-hoc methods typically have fewer learnable parameters compared to train-time techniques, resulting in limited calibration capabilities prompting us to employ train-time calibration in the proposed approach \texttt{RefCal}.  % Recently, Błasiok et al. introduce consistent calibration metrics, such as the smooth calibration error (\smce) and interval calibration error, to overcome the limitations of \ece. We also use these metrics to assess calibration errors \cite{unifying-theory-distance,}.

\section{Proposed methodology}

\noindent We first design a surrogate loss for the refinement objective by proposing a new loss function optimized via supervised contrastive loss \cite{khosla2020supervised}. Next, we perform calibration by training a classifier on frozen refined representations using standard calibration losses. This two stage procedure is termed \texttt{RefCal}, which achieves state-of-the-art calibration and refinement simultaneously.

\subsection{Proposed refinement loss}
\label{subsec:theory}

\myfirstpara{Notation}
Let $i \in \mathcal{I}$ denote indices of all the samples in our dataset. Let $z_i=f(x_i)$ be representation/embedding for a sample $x_i$, that we wish to learn through a \dnn. We normalize $z_i$ such that it lies on a unit hypersphere. The normalization is important to enable the use of the inner product to measure distances in the projection space. We also call $z_i$ as the \emph{anchor}, and use $z_p$/$z_n$ to denote a sample with the same/different label as $z_i$. The set of all samples with the same label as $z_i$ is called a \emph{positive set} (denoted as $\mathcal{P}_i$). Similarly, the set of samples with a different label is called a \emph{negative set} (denoted as $\mathcal{N}_i$). We define $\mathcal{A}_i=\mathcal{P}_i \cup \mathcal{N}_i \backslash z_i$. We propose a new surrogate loss to improve the refinement as follows: 

\begin{definition}[Proposed refinement loss]
	\begin{align}
		\lref = \sum_{i \in \mathcal{I}} & \Bigg( \min_{p \in \mathcal{P}_i} \frac{1}{2}\|z_i - z_p\|^2   
		- \min_{n \in \mathcal{N}_i} \frac{1}{2}\|z_i - z_n\|^2 \Bigg).
		\label{eq:refinement_loss}
	\end{align}
\end{definition}
% ref - \texttt{ref}  \text{SC}
\noindent Next, we show in the section below that the proposed refinement loss is a lower bound to the following supervised contrastive loss \cite{khosla2020supervised}:
\begin{equation}
	\mathcal{L}_{SC} = \sum_{i \in \mathcal{I}} -\frac{1}{|\mathcal{P}_i|} \sum_{p \in \mathcal{P}_i} 
	\log \left( 
		\frac{\exp (z_i \cdot z_p)/\tau}{\sum_{a \in \mathcal{A}_i} \exp (z_i \cdot z_a)/\tau}
	\right).
\label{eq:supcon}
\end{equation}
Here, $|\mathcal{P}_i|$ denotes the size of set $\mathcal{P}_i$. Now we state an important result required to establish the relationship between $\lsc$ and $\lref$, 

\begin{lemma}
    \label{lemma:cos_sim}
    The cosine similarity between two $d$-dimensional vectors $z_1$ and $z_2$ with unit $\ell_2$ norm, can be written as $z_1 \cdot z_2 \triangleq \frac{1}{2}(2 - \|z_1 - z_2\|^2)$.
\end{lemma}
\begin{proof}
Since we normalize the embedding vectors to lie on a unit hypersphere, we have  $\|z_1\|^2=\|z_2\|^2=1$ by definition. Rewriting $z_1 \cdot z_2 = \frac{1}{2}(2 - 2 + 2 z_1 \cdot z_2)$ and using the expansion for $(z_1-z_2)\cdot(z_1 - z_2)$, we have $z_1 \cdot z_2 = \frac{1}{2}(2 - \|z_1\|^2 - \|z_2\|^2 + 2 z_1 \cdot z_2)$. This leads to $z_1 \cdot z_2 =  \frac{1}{2}(2 - \|z_1 - z_2\|^2)$.
\end{proof}

\begin{lemma}
\label{lemma:2}
The supervised contrastive loss, $\lsc$, as given by Eqn. \ref{eq:supcon}, is an upper bound to the refinement loss, $\lref$, as in Eqn. \ref{eq:refinement_loss}. That is: $\lsc > \lref$. 

% Here $C$ is a constant dependent on number of class-wise samples, but not their embedding.
\end{lemma}
\begin{proof}
Without loss of generality, but to keep the exposition simpler, we prove the result for temperature $\tau = 1$ in Eqn. \ref{eq:supcon}. The result for $\tau > 0$ can be shown similarly. Setting $\tau=1$ in Eqn. \ref{eq:supcon}, and using Jensen's inequality: $-\E[\log(X)] \geq -\log(\E[X])$, on the inner summation:
\begin{align}
\lsc & \geq -\sum_{i \in \I} \log \left( \frac{\sum_{p \in \mathcal{P}_i} \eip}{|\mathcal{P}_i|\sum_{a \in \mathcal{A}_i} \eia} \right),
\\
& = \sum_{i \in \I} \log \left( \frac{\sum_{a \in \mathcal{A}_i} \eia}{\sum_{p \in \mathcal{P}_i} \eip} \right) + \sum_{i \in \I} \log(|\mathcal{P}_i|).
\end{align}
%
% The second term is always positive and hence can be removed from the R.H.S. 
Splitting $\mathcal{A}_i$ into $\mathcal{P}_i$ and $\mathcal{N}_i$:
\begin{align}
\lsc & \geq \sum_{i \in \I} \log \left( 1 + \frac{\sum_{n \in \mathcal{N}_i} \ein}{\sum_{p \in \mathcal{P}_i} \eip} \right) + \sum_{i \in \I} \log(|\mathcal{P}_i|).
\intertext{Since $\log(1+x) > \log(x)$ for $x > 0$}
\lsc & > \sum_{i \in \I} \log \left( \frac{\sum_{n \in \mathcal{N}_i} \ein}{\sum_{p \in \mathcal{P}_i} \eip} \right) + \sum_{i \in \I} \log(|\mathcal{P}_i|).
\end{align}
\begin{align}
	\lsc > \sum_{i \in \I} \left( \log \sum_{n \in \mathcal{N}_i} \ein 
	- \log \sum_{p \in \mathcal{P}_i} \eip \right) + \sum_{i \in \I} \log(|\mathcal{P}_i|).\nonumber
\end{align}
Using the inequality: $\max_i(x_i) \leq \log \sum_{i=1}^m \exp (x_i) \leq \max_i(x_i)+ \log (m)$, we get:\\
\begin{align}
\lsc 
&> \sum_{i \in \I} \Big( 
    \max_{n \in \mathcal{N}_i} (z_i \cdot z_n) 
    - \max_{p \in \mathcal{P}_i} (z_i \cdot z_p) 
    - \log |\mathcal{P}_i| \Big)  + \sum_{i \in \I} \log |\mathcal{P}_i| \label{eq:bound1} \\
&> \sum_{i \in \I} \Bigg( 
    \max_{n \in \mathcal{N}_i} \tfrac{1}{2} (2 - \|z_i - z_n\|^2) 
    - \max_{p \in \mathcal{P}_i} \tfrac{1}{2} (2 - \|z_i - z_p\|^2) 
    \Bigg) \label{eq:bound2} \\
&> \sum_{i \in \I} \Bigg( 
    \min_{p \in \mathcal{P}_i} \tfrac{1}{2} \|z_i - z_p\|^2 
    - \min_{n \in \mathcal{N}_i} \tfrac{1}{2} \|z_i - z_n\|^2 
    \Bigg). \label{eq:bound3}
\end{align}

where we use Lemma \ref{lemma:cos_sim} to replace $z_i \cdot z_n$ and $z_i \cdot z_p$. Hence
$\lsc > \lref$. 

\end{proof}

\begin{table*}[]
\centering
\resizebox{\linewidth}{!}{
\begin{tabular}{lccrcccccccccc}
\toprule
\multirow{2}{*}{\textbf{Method}} & \textbf{Venue} & \textbf{Ref.} & \textbf{Cal.} 
& \multicolumn{5}{c}{\textbf{CIFAR100-LT}} 
& \multicolumn{5}{c}{\textbf{CIFAR100}} \\
\cmidrule(lr){5-9} \cmidrule(lr){10-14}
& & & & \textbf{Top-1 $\uparrow$} & \textbf{AUC $\uparrow$} & \textbf{ECE $\downarrow$} & \textbf{SCE $\downarrow$} & \textbf{ACE $\downarrow$} 
& \textbf{Top-1 $\uparrow$} & \textbf{AUC $\uparrow$} & \textbf{ECE $\downarrow$} & \textbf{SCE $\downarrow$} & \textbf{ACE $\downarrow$} \\
\midrule

NLL (CE) & - & \cmark & \xmark & 47.57 & 94.00 & 21.70 & 0.61 & 0.52 & 73.89 & 98.34 & 5.95 & 0.22 & \textbf{0.12} \\
\textbf{RefCal (Ours)} & Ours & \cmark & \cmark & \textbf{58.46} & \textbf{96.22} & \textbf{13.79} & \textbf{0.47} & \textbf{0.44} & \textbf{76.11} & \textbf{98.36} & \textbf{5.20} & \textbf{0.22} & 0.15 \\

\midrule
LS~\cite{originallabelsmoothing} & CVPR'15 & \xmark & \cmark & 47.84 & 92.11 & 9.21 & 0.44 & 0.55 & 74.66 & 98.30 & 11.02 & 0.31 & 0.40 \\
\textbf{RefCal (Ours)} & Ours & \cmark & \cmark & \textbf{58.90} & \textbf{96.65} & \textbf{8.76} & \textbf{0.40} & \textbf{0.47} & \textbf{75.81} & \textbf{99.13} & \textbf{9.15} & \textbf{0.29} & \textbf{0.25} \\

\midrule
CE + TS~\cite{guo2017calibration} & ICML'17 & \cmark & \cmark & 45.35 & 94.40 & 28.25 & 0.71 & 0.57 & 73.89 & 99.10 & 8.00 & 0.28 & 0.23 \\
\textbf{RefCal (Ours)} & Ours & \cmark & \cmark & \textbf{58.46} & \textbf{96.85} & \textbf{7.62} & \textbf{0.41} & \textbf{0.46} & \textbf{76.11} & \textbf{99.18} & \textbf{6.62} & \textbf{0.24} & \textbf{0.18} \\

\midrule
MMCE~\cite{kumarpaper} & ICML'18 & \xmark & \cmark & 49.11 & 94.40 & 25.90 & 0.65 & 0.50 & 72.68 & 98.10 & 8.67 & \textbf{0.26} & \textbf{0.13} \\
\textbf{RefCal (Ours)} & Ours & \cmark & \cmark & \textbf{55.90} & \textbf{95.03} & \textbf{7.91} & \textbf{0.54} & \textbf{0.56} & \textbf{75.04} & \textbf{98.35} & \textbf{8.43} & 0.30 & 0.25 \\

\midrule
MixUp~\cite{thulasidasan2019mixup} & NeurIPS'19 & \xmark & \cmark & 52.90 & 95.82 & \textbf{6.10} & \textbf{0.58} & \textbf{0.54} & \textbf{78.30} & \textbf{98.54} & \textbf{5.49} & \textbf{0.25} & \textbf{0.19} \\
\textbf{RefCal (Ours)} & Ours & \cmark & \cmark & \textbf{56.35} & \textbf{96.37} & 18.84 & 0.69 & 0.72 & 74.81 & 98.22 & 27.02 & 0.60 & 0.56 \\

\midrule
CRL~\cite{moon2020confidence} & ICML'20 & \cmark & \xmark & 46.27 & 93.70 & 22.03 & 0.63 & 0.54 & 73.89 & 98.29 & 5.94 & 0.22 & 0.12 \\
\textbf{RefCal (Ours)} & Ours & \cmark & \cmark & \textbf{58.46} & \textbf{96.22} & \textbf{13.80} & \textbf{0.47} & \textbf{0.44} & \textbf{76.11} & \textbf{98.36} & \textbf{5.21} & \textbf{0.22} & \textbf{0.12} \\

\midrule
FL + MDCA~\cite{mdca} & CVPR'22 & \xmark & \cmark & 46.17 & 94.16 & 11.32 & 0.52 & 0.48 & 73.46 & \textbf{98.34} & 5.71 & \textbf{0.22} & 0.14 \\
\textbf{RefCal (Ours)} & Ours & \cmark & \cmark & \textbf{58.70} & \textbf{96.23} & \textbf{10.81} & \textbf{0.45} & \textbf{0.42} & \textbf{75.86} & 98.29 & \textbf{5.25} & \textbf{0.22} & \textbf{0.16} \\

\midrule
AdaFocal~\cite{ghosh2023adafocal} & NeurIPS'22 & \xmark & \cmark & 47.68 & 95.49 & 29.00 & 0.72 & 0.55 & 68.49 & 98.70 & 16.04 & 0.22 & 0.14 \\
\textbf{RefCal (Ours)} & Ours & \cmark & \cmark & \textbf{58.05} & \textbf{96.62} & \textbf{8.41} & \textbf{0.45} & \textbf{0.44} & \textbf{76.20} & \textbf{99.10} & \textbf{5.40} & \textbf{0.22} & \textbf{0.14} \\

\midrule
MbLS~\cite{liu2022devil} & CVPR'22 & \xmark & \cmark & 48.10 & 93.00 & \textbf{8.36} & 0.45 & 0.47 & 75.92 & 98.46 & \textbf{4.80} & \textbf{0.20} & \textbf{0.13} \\
\textbf{RefCal (Ours)} & Ours & \cmark & \cmark & \textbf{58.79} & \textbf{96.62} & 8.60 & \textbf{0.41} & \textbf{0.47} & \textbf{75.93} & \textbf{99.11} & 8.46 & 0.28 & 0.23 \\

\midrule
LogitNorm~\cite{wei2022mitigating} & ICML'22 & \cmark & \xmark & 52.96 & 94.74 & 49.84 & \textbf{0.13} & 1.30 & 69.89 & 97.32 & 67.80 & 0.04 & 1.59\\
\textbf{RefCal (Ours)} & Ours & \cmark & \cmark & \textbf{55.89} & \textbf{96.30} & \textbf{25.97} & 0.68 & 0.97 & \textbf{72.10} & \textbf{97.55} & \textbf{19.63} & 0.43 & \textbf{0.19} \\

\bottomrule
\end{tabular}
}
\caption{Comparison of refinement and calibration performance on CIFAR100-LT and CIFAR100 using ResNet-50: RefCal (Ours) is our proposed method. We report \texttt{Top-1}, \texttt{AUC}, and calibration metrics: \texttt{ECE}, \texttt{SCE}, \texttt{ACE}. Bold entries denote the best scores. `Ref' = refinement method; `Cal' = calibration method. CIFAR100-LT is used with imbalance factor $0.1$. All models selected based on best Top-1 accuracy averaged over 3 runs. Calibration metrics follow conventions used in prior work: 15 bins for \ece/\sce and adaptive binning for \ace.  \NLL: Negative log likelihood (cross-entropy loss). For all methods, we have used the code provided by the authors. All models used for inference were chosen based on best top $1\%$ accuracy as per norm.}
\label{tab:CIFAR10_100_imb}
\end{table*}

\section{Optimizing for proposed refinement loss}

\textbf{Geometric Interpretation of Refinement Loss:}
The refinement loss $\lref$ in Eqn. \ref{eq:refinement_loss} promotes class-wise clustering by simultaneously: (a) Pulling the nearest positive neighbor closer (via the first term), thus tightening intra-class clusters; (b) Pushing the nearest negative neighbor farther (via the second term), widening inter-class margins. Because the margin between a sample and its nearest negative correlates with classifier performance, $\lref$ acts as a meaningful proxy for refinement. Furthermore, since $\lsc$ upper-bounds $\lref$ (cf. Lemma~\ref{lemma:2}), we minimize $\lsc$ using the supervised contrastive loss proposed in \cite{khosla2020supervised} as a tractable surrogate.

\noindent\textbf{On the Refinement Capabilities of Supervised Contrastive Loss:}
The central contribution of this paper lies in introducing a refinement loss whose effectiveness is validated by improved \auc scores as shown in the results. While supervised contrastive (SC) loss \cite{khosla2020supervised} is widely known for enhancing accuracy, we provide the first theoretical and empirical evidence that it also improves refinement—by showing that it upper bounds our proposed refinement loss. This connection enables the repurposing of SC optimization techniques for refinement tasks. Although we do not propose a novel calibration loss, our ability to leverage existing SC and calibration frameworks for refinement constitutes a key strength of this work.

%\noindent \textcolor{blue}{\textbf{Note 2: A well-refined model can also achieve low ECE} A popular understanding of Brier score indicates that refinement and calibration conflict, but this is only true when error (or accuracy) is taken as a constant\cite{ANewVectorPartitionoftheProbabilityScore}. Otherwise, a perfect classifier is also perfectly calibrated and perfectly refined at the same time. Hence, it is not surprising to see that proposed $\lref$, when used in conjunction with calibration loss, leads to a representation which is more accurate also (compared to the case when only calibration and accuracy loss are used).} % In Brier score terms, not only the two terms on the RHS decreasing, the LHS is also decreasing. So there is no discrepancy.}

%\vspace{-1.5em}
\begin{figure*}[]
	\centering
	\begin{subfigure}{0.32\textwidth}
		\centering
		\includegraphics[width=\linewidth]{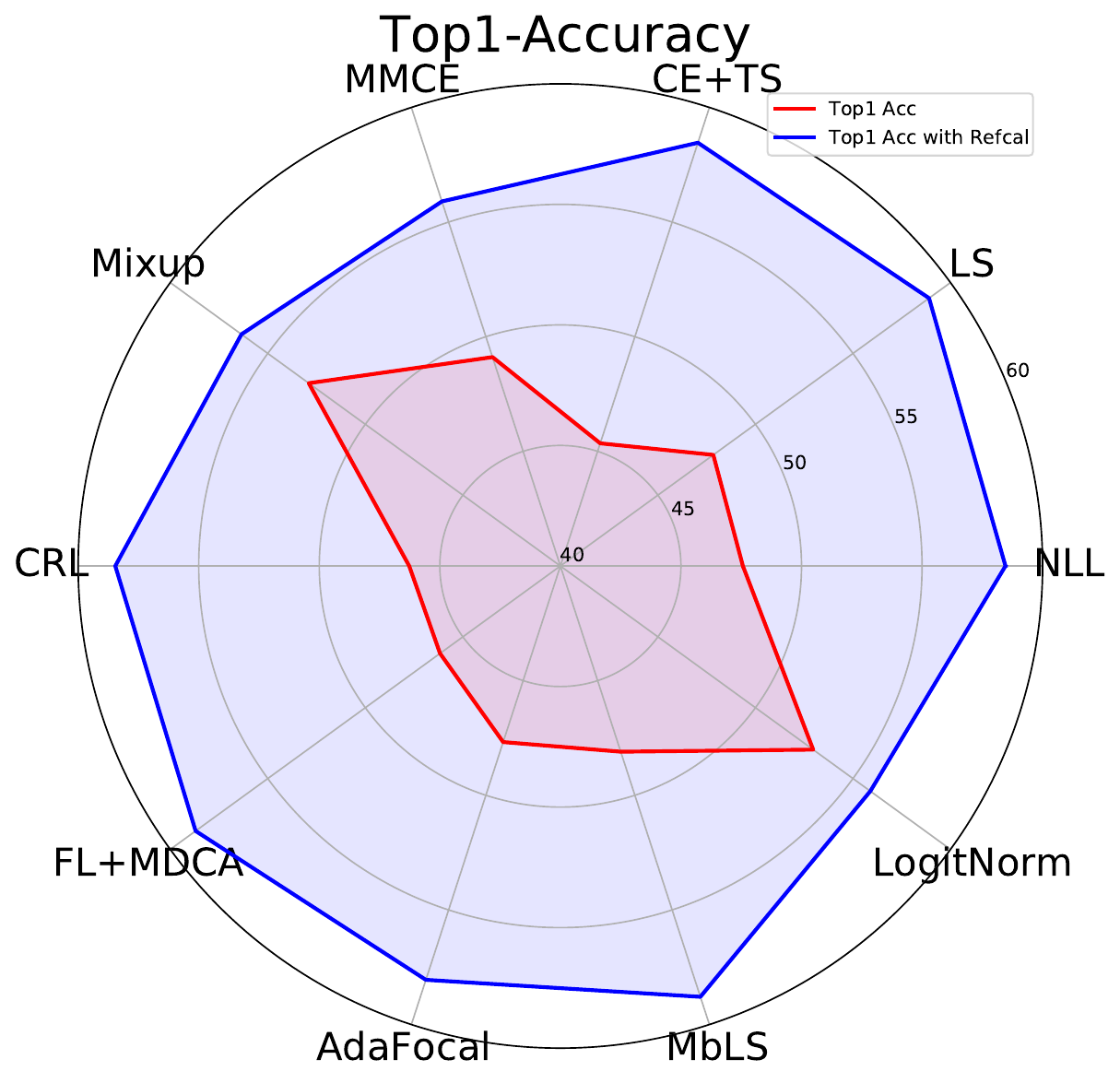}
		\caption{Top 1\% Accuracy}
	\end{subfigure}
	\begin{subfigure}{0.32\textwidth}
		\centering
		\includegraphics[width=\linewidth]{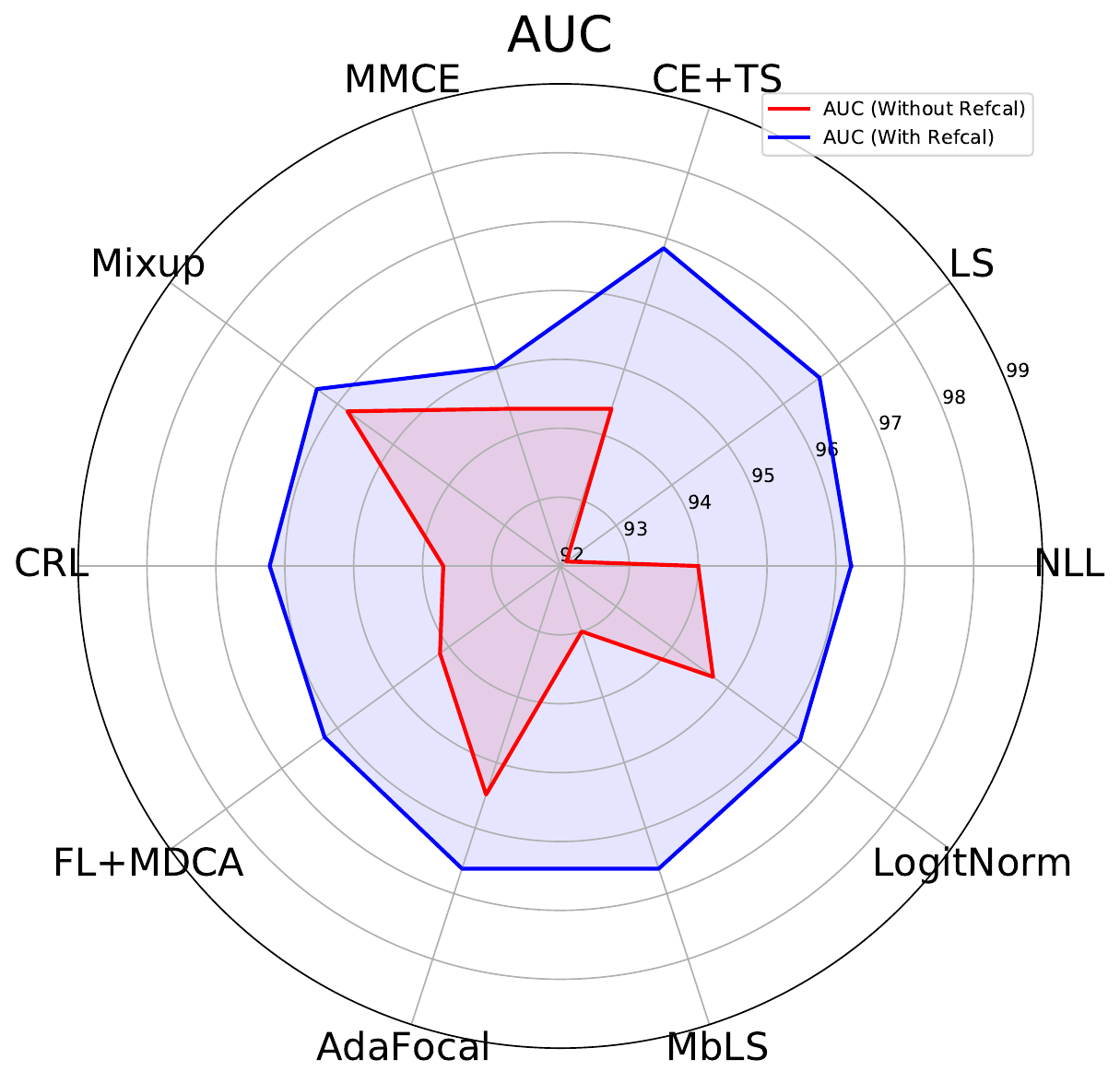}
		\caption{AUC}
	\end{subfigure}
	\begin{subfigure}{0.32\textwidth}
		\centering
		\includegraphics[width=\linewidth]{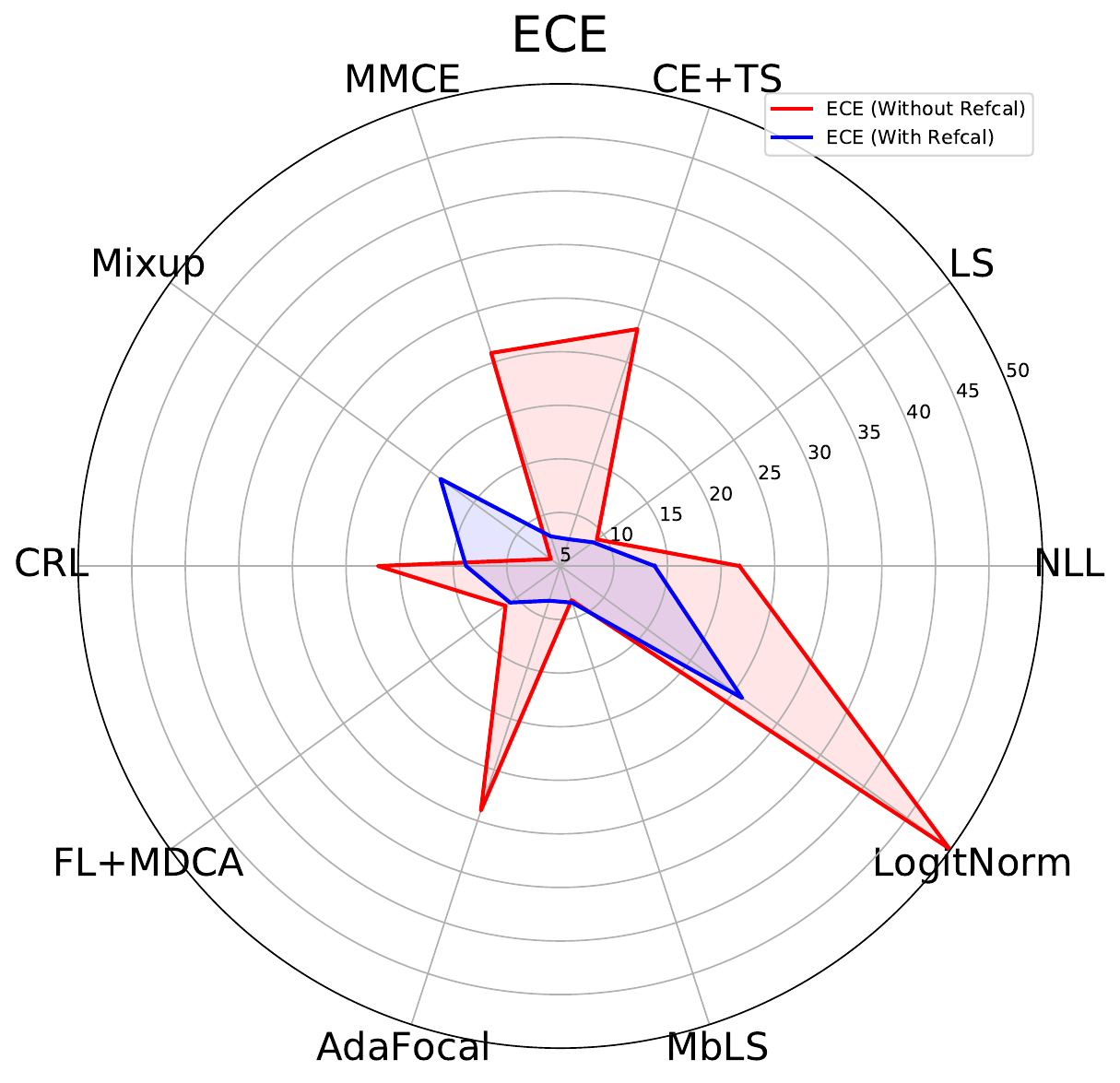}
		\caption{ECE}
	\end{subfigure}
	\caption{Effect of RefCal Training: Spider plot comparing Top-1\% Accuracy, \auc, and \ece with \textcolor{blue}{(Blue)} and without \textcolor{red}{(Red)} \refcal training when \texttt{ResNet-50} features extractor is used on \CIFARH-LT dataset. \refcal offers the best Top1\% accuracy, \auc, and \ece in majority cases. } % This suggests that \refcal enhances both predictive performance and calibration, making models more trustworthy and well-calibrated.}
% \textbf{Note:} The complete table for reliability metrics: Accuracy, Calibration, and Refinement for both \CIFARH and \CIFARLT are in the supplementary material which reports the average of 3 runs.
\label{fig:spiderplotvitch}
\end{figure*}

%\renewcommand{\thefigure}{S6}

\begin{comment}

\begin{figure*}[t]
	%\vspace{1cm}
	\centering
	\hspace*{-0.6cm}\includegraphics[angle=0, width=0.95\linewidth]{figs_plots/Top_1_10_nov_all.png}
\caption{[Scatter Error plot for comparative study of \textbf{Top 1\% accuracy vs. calibration trade-offs} associated with existing techniques and ours (Top-left is most preferred)]: The mean and one standard scatter error bars for \auc, \ece, \sce, \ace of \texttt{ResNet-50} trained on \texttt{CIFAR100-LT (imb factor 10\%} using \sota calibration techniques. Further, the lower variances emphasize the reliability of \refcal variants. All plots were generated by training \texttt{ResNet-50} models through every technique for $3$ runs for fair comparison.}
	\vspace{1em}
	\label{fig:error_bar}
	%  Calibration is measured in terms of ECE ($\downarrow$ is better). Our calibration technique offers the best values
\end{figure*}
\end{comment}

\subsection{RefCal training regime}
\label{subsec:tt_cal}
We propose a training strategy to develop reliable \dnn classifiers that simultaneously exhibit high refinement, high accuracy, and strong calibration (i.e., low calibration error).  It consists of two stages: (a) \textbf{Refinement Stage:} We pretrain an encoder with supervised contrastive loss \cite{khosla2020supervised} to approximate our refinement loss, promoting discriminative and robust class-boundary representations. (b) \textbf{Calibration Stage:} Keeping the encoder frozen, we fine-tune a linear classifier head using a combination of standard classification loss and calibration-specific losses. We experiment with multiple calibration losses (see Tab. \ref{tab:CIFAR10_100_imb}) to assess their impact. The resulting models demonstrate strong performance across all three axes: refinement, calibration, and accuracy. A detailed comparison using various loss combinations is reported in Tabs. \ref{tab:CIFAR10_100_imb},\ref{tab:bin}.

\section{Experimental design}

\myfirstpara{Datasets}
We utilize well-established classification datasets: \CIFART \cite{Krizhevsky_2009_17719-cifar10-100}, \CIFARH, \texttt{CIFAR100-LT}\cite{openlongtailrecognition}, and large-scale \TINY \cite{le2015tiny}, \texttt{ImageNet-1K}\cite{deng2009imagenet} and \imageNetLT \cite{openlongtailrecognition}. To evaluate the robustness of \refcal, we also present results on \CIFARHC \cite{hendrycks2018benchmarking} with varying degrees of corruption. For binary classification, we construct binary benchmarks from multi-class: \CIFART \cite{Krizhevsky_2009_17719-cifar10-100},\texttt{STL10} \cite{coates2011analysis}. %Please refer to the supplementary for more details. 

\mypara{Baseline approaches} 
In our analysis, we incorporate several baseline methods that serve as a basis for comparison. These include models trained using Negative Log Likelihood \NLL (also called Cross Entropy loss- \texttt{CE}), \LS\cite{originallabelsmoothing}, temperature scaling on top of \texttt{CE}, \texttt{TS+CE},  \texttt{MixUp}\cite{thulasidasan2019mixup}, \texttt{Adafocal} \cite{ghosh2023adafocal}, \texttt{MMCE} \cite{kumarpaper}, \texttt{CRL} \cite{moon2020confidence}, \texttt{MDCA} \cite{mdca},  \texttt{MbLS} \cite{liu2022devil},
\texttt{LogitNorm} \cite{wei2022mitigating}, and \texttt{AUCM} \cite{yuan2021large}.

\mypara{Metrics}
\auc: The area under the receiver operating characteristic curve measures separability between the classes and is a proxy for refinement. For measuring calibration, we use Expected Calibration Error (\ece), Static Calibration Error (\sce), and Adaptive Calibration Error (\ace), Smooth Calibration Error (\smce). 

\mypara{Implementation details} 
We implemented a multi-stage training methodology as outlined in \cite{khosla2020supervised}, this is used to minimize our surrogate loss for refinement. The training process involves two phases. In the first stage, the emphasis is on feature extraction, which is achieved using a \texttt{ResNet-50} \cite{he2016deep}/EfficientViT-M1\cite{cai2022efficientvit} network trained with the supervised contrastive loss for $1000$ epochs. Subsequently, after the initial training phase, we freeze the parameters of the Stage $1$ network, and introduce a classifier, constituting the second stage of our training procedure. In this second stage, a linear classifier was trained using a combination of classification and calibration losses for an additional $100$ epochs. %See supplementary for the details.

\begin{insightbox}{Relevance and Applications:} Although we evaluate on moderate-scale datasets, the considered architectures (ResNet, Efficient-ViT, and MobileNet) are widely used in real-world vision systems ranging from medical imaging to edge deployment. Thus, our findings are directly relevant to practical deployment scenarios where model reliability and confidence calibration are critical. Future work will explore fine-tuning VLMs under the \refcal regime to enhance reliability.
\end{insightbox}

\section{Results}
\label{sec:results}

\subsection{Multi-class classification}
\begin{table*}[t]
\centering
\resizebox{\linewidth}{!}{
\begin{tabular}{lcccccccccccc}
\toprule

\multirow{2}{*}{\large \textbf{Method}} & \multicolumn{6}{c}{{\textbf{\large (a) \texttt{ResNet18} on Imagenet-LT}}} &  \multicolumn{6}{c}{{\textbf{\large (b) \texttt{EfficientViT-M1} on \CIFARH}}} \\
%\midrule
 &  \textbf{Top1 (\%)~$\uparrow$} & \textbf{AUC~$\uparrow$}  & \textbf{ECE (\%) ~$\downarrow$}    & \textbf{SCE (\%)~$\downarrow$} & \textbf{ACE(\%)~$\downarrow$}  & \textbf{smECE(\%)~$\downarrow$} & \textbf{Top1 (\%)~$\uparrow$} & \textbf{AUC~$\uparrow$}  & \textbf{ECE (\%) ~$\downarrow$}    & \textbf{SCE (\%)~$\downarrow$} & \textbf{ACE(\%)~$\downarrow$}  & \textbf{smECE(\%)~$\downarrow$} \\
\bottomrule
\multicolumn{1}{l}{NLL (CE)} & 41.76 & \textbf{98.30} & 16.11 & 00.07 & 00.06 & 15.98 & 47.77 & 95.07 & 29.53 & 00.71 & 00.42 & 27.13 \\
 \textbf{\texttt{RefCal}} \textbf{(Ours)}& \textbf{42.18} & 98.10 & \textbf{10.82} & \textbf{00.05} & \textbf{00.04} & \textbf{10.38} & \textbf{56.78} & \textbf{96.00}    & \textbf{09.67}  & \textbf{00.28} & \textbf{00.19} & \textbf{09.17}  \\
\midrule
LS~\cite{originallabelsmoothing} & 41.73 & 97.70 & \textbf{06.06}  & 00.05 & 00.06 & \textbf{06.02}  & 48.97 & 92.79 & \textbf{04.11}  & 00.30  & 00.25 & \textbf{04.06}  \\
 \textbf{\texttt{RefCal}} \textbf{(Ours)}  & \textbf{42.25} & \textbf{98.10} & 08.46  & \textbf{00.05} & \textbf{00.04} & 08.19  & \textbf{56.83} & \textbf{96.20}  & 08.32  & \textbf{00.25} & \textbf{00.19} & 07.90   \\
\midrule
CE+TS~\cite{guo2017calibration} & 41.76     & \textbf{98.30}    & 09.13     & 00.06    & 00.05    & 08.98     & 47.77 & 95.40  & 15.25 & 00.43 & 00.35 & 14.99 \\
 \textbf{\texttt{RefCal}} \textbf{(Ours)}     & \textbf{42.18} & 98.20 & \textbf{06.14}  & \textbf{00.05} & \textbf{00.04} & \textbf{06.00}     & \textbf{56.78} & \textbf{96.30}  & \textbf{07.82}  & \textbf{00.25} & \textbf{00.26} & \textbf{07.55}  \\
\midrule
MMCE~\cite{kumarpaper} & 41.89 & 98.10 & 15.91 & 00.07 & 00.06 & 15.78 & 48.70  & 95.11 & 28.82 & 00.70  & 00.42 & 26.61 \\
 \textbf{\texttt{RefCal}} \textbf{(Ours)} & \textbf{42.22} & \textbf{98.10} & \textbf{10.62} & \textbf{00.05} & \textbf{00.04} & \textbf{10.25} & \textbf{56.74} & \textbf{96.00}    & \textbf{09.61}  & \textbf{00.28} & \textbf{00.19} & \textbf{09.11}  \\
\midrule
CRL~\cite{moon2020confidence} & 41.61 & 98.10 & 16.08 & 00.07 & 00.06 & 15.95 & 49.07 & 95.11 & 28.53 & 00.69 & 00.41 & 26.45 \\
 \textbf{\texttt{RefCal}} \textbf{(Ours)}    & \textbf{42.18} & \textbf{98.10} & \textbf{10.82} & \textbf{00.05} & \textbf{00.04} & \textbf{10.38} & \textbf{56.75} & \textbf{96.10}  & \textbf{09.24}  & \textbf{00.27} & \textbf{00.18} & \textbf{08.79}  \\
\midrule
FL+ MDCA  ~\cite{mdca}  & 40.68 & \textbf{98.40} & 08.04  & 00.06 & 00.06 & 08.03  & 47.25 & 95.25 & 17.76 & 00.51 & 00.37 & 17.69 \\
 \textbf{\texttt{RefCal} (Ours)}   & \textbf{42.19} & 98.20 & \textbf{07.30}   & \textbf{00.05} & \textbf{00.04} & \textbf{07.07}  & \textbf{56.69} & \textbf{96.60}  & \textbf{05.90}   & \textbf{00.24} & \textbf{00.20}  & \textbf{05.59}  \\
\midrule
AdaFocal~\cite{ghosh2023adafocal} & 31.18 & 95.60 & 34.01 & 00.10  & 00.07 & 30.26 & 52.39 & \textbf{96.63} & \textbf{14.35} & \textbf{00.43} & \textbf{00.29} & \textbf{14.36} \\
 \textbf{\texttt{RefCal}} \textbf{(Ours)}    & \textbf{41.83} & \textbf{97.80} & \textbf{19.25} & \textbf{00.06} & \textbf{00.04} & \textbf{18.83} & \textbf{53.10}  & 93.90  & 23.30  & 00.57 & 00.43 & 21.42 \\
\midrule
MbLS~\cite{liu2022devil} & 41.14 & 97.70 & \textbf{04.18}  & 00.06 & 00.05 & \textbf{04.16}  & 49.90  & 93.85 & 08.78  & 00.33 & 00.25 & 08.78  \\
 \textbf{\texttt{RefCal}} \textbf{(Ours)}  & \textbf{42.23} & \textbf{98.10} & 08.63  & \textbf{00.05} & \textbf{00.04} & 08.30   & \textbf{56.82} & \textbf{96.20}
  & \textbf{08.50}   & \textbf{00.26} & \textbf{00.19} & \textbf{08.07}  \\
\midrule
LogitNorm~\cite{wei2022mitigating} & 42.00    & 96.00   & 41.81 & 00.01    & 00.16 & 27.43 & 44.61 & 94.68 & \textbf{16.38} & \textbf{00.47} & 00.35 & \textbf{16.36} \\
 \textbf{\texttt{RefCal}} \textbf{(Ours)}  & \textbf{42.04} & \textbf{96.30} & \textbf{41.90}  & \textbf{00.01}    & \textbf{00.16} & \textbf{25.77} & \textbf{55.93} & \textbf{95.30}  & 25.93 & 00.60  & \textbf{00.34} & 22.92 \\
\bottomrule
\end{tabular}
}
\caption{(a) [Large Scale Experiments]:Comparison of reliability metrics of \refcal (Ours)  vs. \sota on \texttt{ImageNet-LT} using \texttt{ResNet18}. (b) [Experiments with Visual Transformers]: Comparison of reliability metrics of \refcal (Ours)  vs. \sota on \texttt{Cifar-100} using Memory Efficient Vision Transformer architecture \texttt{EfficientViT-M1} as feature extractor.}
\label{tab:imagenet_lt_and_efficient_vit}
\end{table*}

%\vspace{-2em}

\begin{wraptable}{r}{0.5\linewidth}
\centering
\footnotesize
\resizebox{\linewidth}{!}{%
\begin{tabular}{lrrrrrr}
\toprule
\textbf{Method} & \textbf{Top-1 (\%) $\uparrow$} & \textbf{AUROC $\uparrow$} & \textbf{ECE $\downarrow$} & \textbf{SCE $\downarrow$} & \textbf{ACE $\downarrow$} & \textbf{smECE $\downarrow$} \\
\midrule
NLL (CE) & 96.71 & 99.37 & 1.21 & 1.64 & 1.47 & 1.20 \\
\textbf{RefCal (Ours)} & \textbf{98.69} & \textbf{99.80} & \textbf{1.13} & \textbf{1.11} & \textbf{1.02} & \textbf{1.00} \\

\midrule
LS~\cite{originallabelsmoothing} & 96.64 & 98.92 & 9.23 & 9.25 & 9.23 & 9.32 \\
\textbf{RefCal (Ours)} & \textbf{98.58} & \textbf{99.90} & \textbf{3.02} & \textbf{3.21} & \textbf{2.97} & \textbf{3.03} \\

\midrule
CE + TS~\cite{guo2017calibration} & 96.71 & 99.40 & 0.41 & 1.21 & 1.16 & 0.66 \\
\textbf{RefCal (Ours)} & \textbf{98.69} & \textbf{99.90} & \textbf{0.26} & \textbf{0.51} & \textbf{0.46} & \textbf{0.47} \\

\midrule
MMCE~\cite{kumarpaper} & 96.19 & 99.30 & 1.22 & 1.38 & 1.19 & 1.22 \\
\textbf{RefCal (Ours)} & \textbf{98.74} & \textbf{99.88} & \textbf{1.13} & \textbf{1.16} & \textbf{1.00} & \textbf{0.98} \\

\midrule
MixUp~\cite{thulasidasan2019mixup} & 96.58 & 99.36 & 6.01 & 6.30 & 6.26 & 6.04 \\
\textbf{RefCal (Ours)} & \textbf{98.58} & \textbf{99.90} & \textbf{3.16} & \textbf{3.17} & \textbf{3.13} & \textbf{3.16} \\

\midrule
CRL~\cite{moon2020confidence} & 96.88 & 99.47 & 1.46 & 1.45 & 1.40 & 1.50 \\
\textbf{RefCal (Ours)} & \textbf{98.69} & \textbf{99.88} & \textbf{1.13} & \textbf{1.12} & \textbf{1.02} & \textbf{1.01} \\

\midrule
AUCM~\cite{yuan2021large} & 95.49 & 98.86 & 3.04 & 3.11 & 2.97 & 2.76 \\
\textbf{RefCal (Ours)} & \textbf{98.56} & \textbf{99.90} & \textbf{0.56} & \textbf{0.81} & \textbf{0.82} & \textbf{0.64} \\

\midrule
FL + MDCA~\cite{mdca} & 95.49 & 98.23 & \textbf{0.93} & 1.97 & 1.98 & \textbf{0.98} \\
\textbf{RefCal (Ours)} & \textbf{98.68} & \textbf{99.35} & 1.30 & \textbf{1.28} & \textbf{1.12} & \textbf{0.80} \\

\midrule
AdaFocal~\cite{ghosh2023adafocal} & 96.16 & \textbf{99.11} & \textbf{1.10} & 1.10 & 1.03 & 1.06 \\
\textbf{RefCal (Ours)} & \textbf{98.69} & 98.68 & 1.31 & \textbf{0.67} & \textbf{0.64} & \textbf{0.67} \\

\midrule
MbLS~\cite{liu2022devil} & 96.98 & 99.42 & 1.09 & 1.18 & 0.85 & 1.15 \\
\textbf{RefCal (Ours)} & \textbf{98.71} & \textbf{99.91} & \textbf{0.82} & \textbf{0.86} & \textbf{0.64} & \textbf{0.85} \\
\bottomrule
\end{tabular}
}
\caption{Comparison of reliability metrics for binary classification with \texttt{ResNet-50} \cite{he2016deep} on \texttt{STL10} dataset \cite{coates2011analysis}.}
\label{tab:bin}
\end{wraptable}

We conduct experiments on $8$ datasets(\texttt{ImageNet-1K}, \CIFARH,  \imageNetLT, \texttt{CIFAR100-LT}, \texttt{STL10},  \TINY, \CIFARHC (Corrupted \CIFARH)) and \CIFART). The results for the first 5 datasets are in the main text while the supplementary consists of \TINY, \CIFART and \CIFARHC results.   
Tab. \ref{tab:CIFAR10_100_imb} shows that training with \refcal achieves a significant improvement in refinement as measured by \auc and a reduction in calibration error (\ece, \sce, \ace) while also providing an accuracy increase. To understand the benefits of \refcal, observe the performance of our approach to non-calibrated/directly calibrated baseline models. Notice a jump in accuracy by over $10\%$ in every case, similarly on \auc we do see a consistent increase. The reduction in calibration error (\ece and \smce) is also seen in majority of the cases.

Fig. \ref{fig:error_bar} presents the error bars for $3$ runs and comparative study highlighting \auc vs. calibration trade-offs associated with existing techniques and \refcal (Top-left is the most desirable location on the chart suggesting higher \auc and lower calibration error). Specifically, we found the mean and one standard scatter error for \auc and \ece plot shows the lower variances in case of \refcal variants emphasizing its reliability in comparison with \sota. Note the variant \texttt{\refcal+CE+TS} (ours) offers the highest \auc-low calibration error with the next best being \texttt{AUCM}\cite{yuan2021large}. We also compare \texttt{Top $1\%$ accuracy} vs. calibration trade-off associated with contemporary techniques and \refcal. The mean and one standard scatter error bars for \texttt{Top $1\%$ accuracy} and \ece reveal that \refcal variants not only offer lower variance but also reasonably good top 1\% accuracy and calibration error trade-off. %Note \texttt{RefCal+CE+TS} performs the best. %See Supplementary for \texttt{CIFAR-10-LT} results. 

\begin{wraptable}{l}{0.5\linewidth}
\centering
\resizebox{\linewidth}{!}{%
\begin{tabular}{l|llllll}
\toprule
\multirow{3}{*}{\textbf{Method}} & && \textbf{\large CIFAR-10} \\
 &  \textbf{FPR@TPR95 ~$\downarrow$} & \textbf{Det Error ~$\downarrow$}  & \textbf{AUROC~$\uparrow$}    & 
 \textbf{AUPR In ~$\uparrow$} & \textbf{AUPR Out ~$\uparrow$}  \\ 
\bottomrule

NLL (CE) & 30.50 &08.80 & 95.90& 96.90 & 94.70 \\
\textbf{+  \refcal (Ours)} & \textbf{16.30} &\textbf{07.10} & \textbf{97.50} & \textbf{97.70} &\textbf{97.00} \\

\midrule
LS~\cite{originallabelsmoothing} & 29.30 & 09.30 & 92.60 & 86.80 & 92.40\\
\textbf{+  \refcal (Ours)} & \textbf{17.00} & \textbf{08.20} & \textbf{97.10} & \textbf{97.60} & \textbf{96.30}\\
\midrule
CE+TS~\cite{guo2017calibration} &  22.20 & 08.10 &97.00 &97.60 &96.40\\
\textbf{+  \refcal (Ours)} & \textbf{14.10} & \textbf{06.80} &\textbf{97.90}& \textbf{98.20} &\textbf{97.60}\\
\midrule
MMCE~\cite{kumarpaper} & \textbf{12.10} &\textbf{06.90} &\textbf{98.30} &\textbf{98.40} &\textbf{98.20} \\
\textbf{+  \refcal (Ours)} &20.80 &10.30&97.30&85.80&95.60  \\
\midrule
CRL~\cite{moon2020confidence} & 27.90&08.20&96.30&97.20&95.30 \\
\textbf{+  \refcal (Ours)}  & \textbf{16.30}&\textbf{07.10}&\textbf{97.50}&\textbf{97.70}&\textbf{97.00}  \\
\midrule
AUCM~\cite{yuan2021large} & 38.30&12.10&93.40&93.60&92.40  \\
\textbf{+  \refcal (Ours)}  & \textbf{10.20}&\textbf{06.60}&\textbf{98.20}&\textbf{98.50}&\textbf{98.00}  \\
\midrule
FL+ MDCA  ~\cite{mdca}  & 20.50&08.00& 97.10&\textbf{97.70}&\textbf{96.60}\\
\textbf{+  \refcal (Ours)}  & \textbf{18.00}&\textbf{07.20}&\textbf{97.30}&97.40&96.50  \\
\midrule
MbLS~\cite{liu2022devil} & 25.00&08.20&96.40&97.10&95.10  \\
\textbf{+  \refcal (Ours)}   &\textbf{16.90} &\textbf{07.80}&\textbf{97.20}&\textbf{97.70}&\textbf{96.60}  \\
\bottomrule
\end{tabular}
}
%\vspace{-1em}
\caption{[\textbf{Robustness to \ood data]}: Comparison of reliability metrics of \refcal (Ours) vs. \sota. We use \texttt{ResNet-50} backbone on in-distribution dataset \texttt{CIFAR10} and test on \ood dataset \texttt{SVHN}.}
\label{tab:ood}
\end{wraptable}

\mypara{Performance on various feature extractors and large scale datasets}
\refcal continues to perform superior on \texttt{ImageNet-LT}/\texttt{ImageNet} and even using a different backbone architecture \texttt{EfficientViT-M1}\cite{cai2022efficientvit} as evidenced in the Tab. \ref{tab:imagenet_lt_and_efficient_vit} in majority of the cases. Fig. \ref{fig:spiderplotvitch} shows the spider plot illustrating that applying RefCal (blue) consistently improves Top1\% Accuracy and \auc across different calibration methods compared to models without \refcal (red). Additionally, \ece is significantly reduced with \refcal, indicating better model reliability.

% \input{tables/ImageNet_balanced} - moved to suppl material

 %This suggests that \refcal enhances both predictive performance and calibration, making models more trustworthy and well-calibrated.}

\begin{figure}[]
	%\vspace{1cm}
	\centering
	\hspace*{-0.6cm}\includegraphics[angle=0, width=0.9\linewidth]{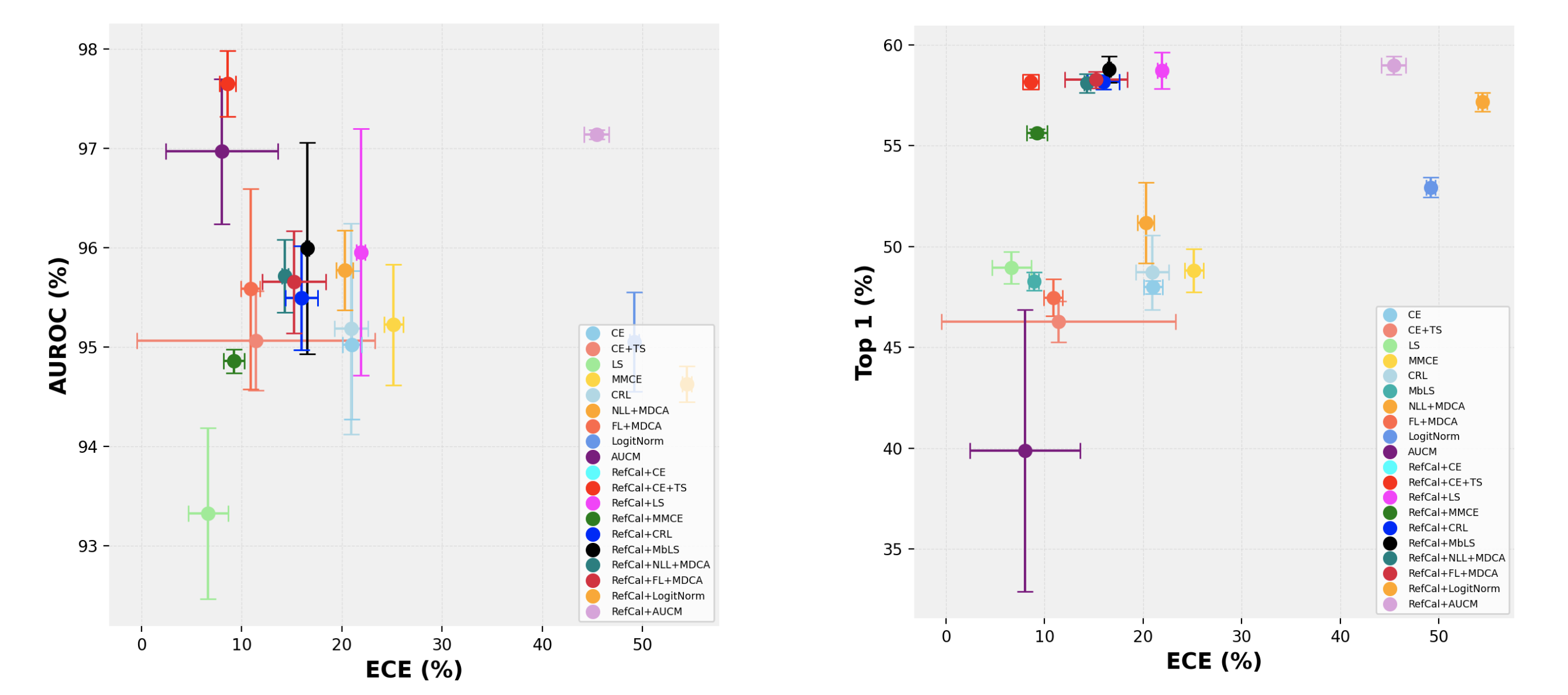}
% \vspace{-0.5cm}
\caption{(Left) \auc vs. \ece trade-off and (Right) Top $1\%$ accuracy vs. \ece trade-off for ResNet-50 on CIFAR100-LT (IF=10\%). Higher \auc and Top $1\%$ accuracy with lower \ece (top-left) are desirable. Lower variance in \auc and \ece highlights the reliability of \refcal variants.}
%Note the variant \texttt{RefCal+CE+TS} (ours) offers the highest \auc-low \ece with the next best being \texttt{AUCM}\cite{yuan2021large}. 
%Note that \texttt{RefCal+CE+TS} performs the best. We have run every technique $3$ times for fair comparison. See supplementary for zoomed-in view. } 
\label{fig:error_bar}
\end{figure}

\subsection{Binary classification} \label{subsec:binClass}

For binary classification, we employ multi-class classification datasets, transforming them into binary classification datasets. The categorization of classes within these datasets is done by grouping semantically relevant classes together. We construct binary benchmark dataset: \texttt{STL10} \cite{cao2019learning} following the settings mentioned in \cite{yuan2021large}. Tab. \ref{tab:bin} reports results on binary classification. We outperform/are comparable to the \sota on binary classification tasks in Top 1\% accuracy, \ece, and \auc.

%\input{tables/merged_imagenet_lt_and_efficient_vit}
%\vspace{-2em}

%\cref{fig:teaser} in main text and (\cref{fig:gradcamsuppl1},\cref{fig:gradcamsupp2},\cref{fig:gradcamsupp3}) show Grad-CAM visualization of \refcal vs. contemporaries.

\subsection{Robustness of \textbf{\refcal}}

\begin{comment}

\myfirstpara{Robustness to natural corruptions} 
%
\dnns lack robustness to out-of-distribution data or natural corruptions such as noise, blur, etc. However, Optimizing for accuracy, calibration and refinement using the proposed training regime leads to more robust models. In supplementary, Table \ref{tab:ood}, we report the results on \CIFARHC using various corruptions, with varying degrees of corruption. We observe that our training not only yields higher accuracy, \auc, but also favorable calibration values. We observe that \refcal surpasses the \sota in terms of evaluation metrics performance despite corruption showing relatively lower degradation compared to \sota. Training with proposed refinement loss \cite{khosla2020supervised} has many benefits including robustness to corruptions as demonstrated in Fig. \ref{fig:brightnessDist}, Fig. \ref{fig:frostDist}.
\end{comment}

\mypara{Robustness to out-of-distribution (\ood) samples (semantic shift)} 
A well-calibrated model should not only exhibit low confidence whenever it misclassifies but also in situations when it encounters data that belongs to a new/different class different than the training classes. We investigate if \refcal can acquire representations that demonstrate increased resilience to \ood samples by assigning them low confidence. Tab. \ref{tab:ood} summarizes the \ood detection performance of our model and highlights the capability to reject such unknown samples by increasing \auc between the in-distribution and \ood samples and decreasing both \texttt{FPR@0.95TPR}, and \texttt{Detection Error}.

\mypara{Grad-CAM visualization}
Figs. \ref{fig:teaser} and \ref{fig:addnl_gcam} compares Grad-CAM visualizations of \texttt{RefCal} and contemporary methods.
%, showing that refinement loss shifts attention from background to salient object features by separating confidence for correct and incorrect predictions.
%Fig. \ref{fig:teaser} present Grad-CAM visualizations comparing \texttt{RefCal} with contemporary methods. As refinement loss forces a model to separate its confidence for correct, and incorrect predictions, it also leads the model to focus its attention on the salient object features, instead of the background.

\begin{figure}
    \centering
    \includegraphics[width=0.9\linewidth]{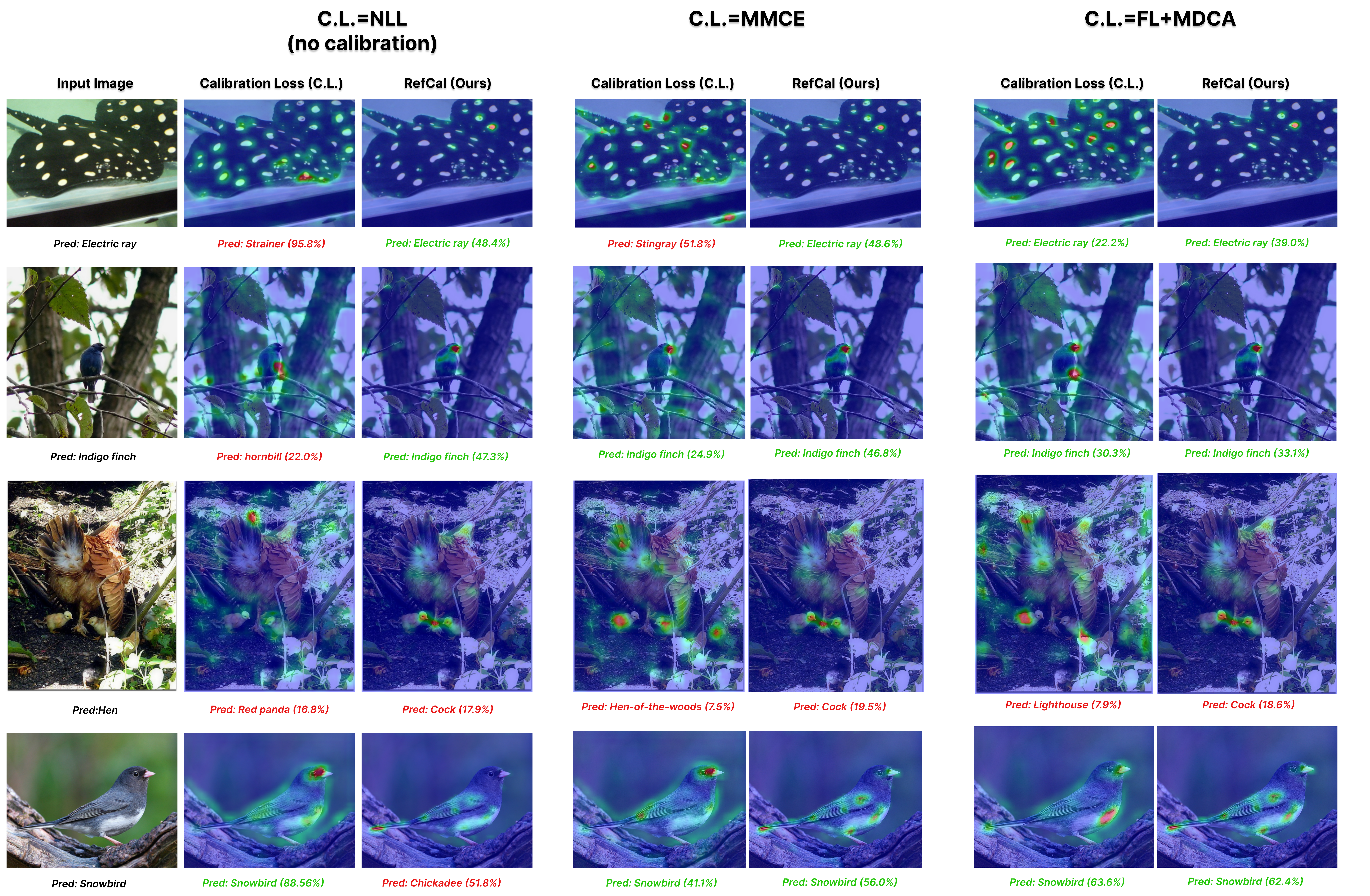}
\caption{\textbf{[Grad-CAM]}: Our proposed training regime, \refcal allows for joint optimization of calibration and refinement. Grad-CAM visualizations for \texttt{Resnet-18} trained on \texttt{ImageNet-LT}, using a particular calibration technique (column bearing title Calibration Loss (C.L.)), and then by jointly optimizing with the same calibration but adding our refinement loss (columns bearing title \refcal i.e., columns 3,5, and 7). Here, the Calibration Loss tested are: NLL (no calibration) in column 2, MMCE\cite{kumarpaper} in column 4, and FL+MDCA\cite{hebbalaguppe2022novel} in column 6. \textbf{Note:} Numbers in parentheses indicate predictive confidence for each method. Since, the refinement loss forces a model to separate its confidence for positive, and negative samples, it also leads the model to focus its attention on the salient object features, instead of the background.}
    \label{fig:addnl_gcam}
\end{figure}

\mypara{Limitations}
We observe that combining our loss $\lref$ with calibration methods such as AUCM~\cite{yuan2021large}, MixUp~\cite{mixup-augmentation}, and Adafocal~\cite{ghosh2023adafocal} can be challenging on some datasets. Although $\lref$ provably improves refinement, joint optimization with multiple losses via SGD introduces complex interactions. We conjecture that some calibration methods produce representations more compatible with $\lref$ than others, which we defer to future work. Our approach uses SupCon~\cite{khosla2020supervised} during training, incurring higher training cost than post-hoc calibration, but yielding substantial reliability gains with no additional inference overhead.

\section{Conclusions}
This paper demonstrates that existing calibration methods can improve calibration metrics while degrading prediction refinement. To address this, we propose a novel surrogate loss for refinement. We theoretically show that supervised contrastive loss upper-bounds our refinement loss, making it a suitable surrogate for minimization. Building on this, we introduce \refcal, a training framework that optimizes refinement, calibration, and accuracy. Experiments on standard classification benchmarks show that our approach outperforms or comparable to the \sota methods in producing accurate, well-calibrated, and refined predictions, thereby enhancing the classification reliability.

\section{Acknowledgments}
The authors would like to thank Harshad Kadilkar and Ranjitha Prasad for carefully reviewing the proof and for their suggestions, which helped make the proof more concise.

%In this paper we showed how current calibration techniques may end up taking a shortcut and achieve better calibration performance at the cost of reduced refinement. We suggested a novel surrogate loss to improve refinement. We also gave an attractive optimization technique to minimize the loss using widely available supervised contrastive learning. We showed mathematically how supervised contrastive loss is an upper bound to the proposed refinement loss, and hence, a surrogate to minimize the proposed loss. Finally, we have a new training regime \refcal, which can jointly optimize the proposed refinement loss, coupled with calibration, and accuracy.
%Finally, our experiments on standard benchmark datasets for classification reveal that this approach significantly surpasses \sota in generating highly accurate, refined, and confidence calibrated estimates that enhance reliability of \dnns.

 %This leads to a joint optimization framework to simultaneously achieve improved calibration, refinement as well as accuracy. 
%
%We present \textbf{(a)} \refcal, a training methodology that optimizes calibration and refinement of \dnns, enabling them to produce trustworthy predictions \textbf{(b)} We substantiate, both theoretically and empirically, how supervised contrastive loss can be used to minimize our proposed refinement loss and demonstrate the seamless extension of refinement to multi-class classification; \textbf{(c)}

\bibliography{iclr2026_conference}
\bibliographystyle{iclr2026_conference}

\appendix

\newpage
%\clearpage

\section{Appendix}
We supplement the main text with additional results that comprise the following details.
\begin{itemize}

\item \textbf{Reproducibility Checklist} is attached as a separate pdf along with the supplementary material.

	\item \textbf{Source Code}: The supplementary materials include the source code along with a \texttt{readme.html} file, enclosed within the provided zip file. 

\item \textbf{Additional Results} We strengthen the \refcal methodology with additional experiments, covering refinement and calibration results on large-scale datasets such as runs on \TINY and \TINYLT in addition to \texttt{ImageNet-1K} in the main text. Supplemental material also consists of \texttt{CIFAR10-LT} experiments. Additionally, we include experiments of robustness to corruption and Out-of-distribution data.

 \item \textbf{Grad-CAM visualization}  We provide Grad-CAM visualizations on additional examples illustrating both the efficacy and shortcomings of the \refcal training regimen. Additionally, we include zoomed-in Grad-CAM visualizations of models trained with and without the \refcal procedure, as depicted in the main body of the text.

	\item \textbf{Training and Compute Details}: We furnish comprehensive information on the (a) datasets to keep the paper self-contained; (b) baselines; and (c) specifics of training methodology with details on how to reproduce each table. Finally, we mention the compute/resources employed in our experiments.

	\item \textbf{Limitations and Broader Impact}: This section delves into the limitations of our research and contemplates its broader impact on the field.
\end{itemize}

%\section{Reproducibility Checklist}

%\input{ReproducibilityChecklist_AAAI}

\section{Additional Results}
\label{sec:supplAdditionalResults}
%\input{tables/Suppl1_ls_Tiny}

%\subsection{Performance on Large scale dataset} 

%\input{tables/Suppl1_ls_Tiny}
\begin{table*}[t]
\caption{[\textbf{Comparison of reliability metrics in multi-class classification: Accuracy, Calibration, and Refinement of our approach, \refcal  vs. \sota}] on \texttt{CIFAR10-LT} \cite{openlongtailrecognition} using \texttt{ResNet-50} \cite{he2016deep} feature extractor. Notice that \refcal variants outperform contemporary methods in majority of cases.}
%Additional binary classification experiments are a part of supplementary.  } %\textcolor{red}{CPC and ACLS if time permits. Yet to check bold in entries in row-pairs. Rationale for choosing imbalanced dataset??}}
\centering
\resizebox{\linewidth}{!}{%
\begin{tabular}{l|cccccc}
\toprule

\multirow{3}{*}{\large \textbf{Method}} &  & &  \textbf{\large CIFAR-10-LT} &&&   \\
%\midrule
 &  \textbf{Top1 (\%)~$\uparrow$} & \textbf{AUROC~$\uparrow$}  & \textbf{ECE (\%) ~$\downarrow$}    & \textbf{SCE (\%)~$\downarrow$} & \textbf{ACE(\%)~$\downarrow$}  & \textbf{smECE(\%)~$\downarrow$} \\ 
\bottomrule
%NLL (CE) & \cmark         &  \xmark   &          &       &           &          &       &       &   \\
NLL (CE) &  85.30  & 98.40  & 09.35  & 02.22 & 02.02 & 08.91   \\
\textbf{\refcal} \textbf{(Ours)}&  \textbf{89.70}  & \textbf{98.82} & \textbf{08.60}   & \textbf{01.83}  & \textbf{01.45}  & \textbf{06.85}\\

\midrule
LS~\cite{originallabelsmoothing} & 86.58 & 97.33 & 05.00 & 01.47 & 01.64 & \textbf{03.92}\\
\textbf{\refcal} \textbf{(Ours)}    &  \textbf{91.17}  & \textbf{99.28}  & \textbf{04.84} & \textbf{01.03}  & \textbf{00.70}  & 04.55\\
\midrule
CE+TS~\cite{guo2017calibration} & 85.30  & 98.90  & 03.42  & 01.84 & 01.86 & 03.41 \\
\textbf{\refcal} \textbf{(Ours)}     & \textbf{89.70}  & \textbf{99.12} & \textbf{03.09}  & \textbf{01.04}  & \textbf{01.01}   & \textbf{02.96}\\
\midrule
MMCE~\cite{kumarpaper} & 86.47 & 98.55 & \textbf{09.30}   & 02.08 & 01.85 & 08.47\\
\textbf{\refcal} \textbf{(Ours)} & \textbf{90.02} & \textbf{98.88} & 09.71  & \textbf{01.84}  & \textbf{01.22}  & \textbf{05.62}   \\
\midrule
MixUp~\cite{thulasidasan2019mixup} & 88.92 & \textbf{98.75} & \textbf{07.97}  & 02.54 & 02.50  & 07.51\\
\textbf{\refcal} \textbf{(Ours)}    &  \textbf{89.49} & 98.69 & 10.16 & \textbf{02.54}  & \textbf{02.42}  & \textbf{07.42}  \\
\midrule
CRL~\cite{moon2020confidence} & 83.16 & 97.71 & 10.90  & 02.33 & 02.08 & 10.41\\
\textbf{\refcal} \textbf{(Ours)}    & \textbf{89.70}  & \textbf{98.82} & \textbf{08.59}  & \textbf{01.83}  & \textbf{01.45}  & \textbf{06.85} \\
\midrule
AUCM~\cite{yuan2021large} & 88.35 & 98.62 & 05.75  & \textbf{01.38} & \textbf{01.11} & 05.79 \\
\textbf{\refcal (Ours)}     &  \textbf{89.39} & \textbf{99.10}  & \textbf{05.41}  & 01.91  & 01.89  & \textbf{05.41}\\
\midrule
FL+ MDCA  ~\cite{mdca}  & 87.60  & 98.41 & 05.45  & 01.42 & 01.28 & 05.40 \\
\textbf{\refcal (Ours)}   &  \textbf{89.12} & \textbf{99.11} & \textbf{03.16}  & \textbf{01.36}  & \textbf{01.28}  & \textbf{03.14}\\
\midrule
AdaFocal~\cite{ghosh2023adafocal} & 81.51 & 98.30  & 11.26 & 02.54 & 02.35 & 11.05\\
\textbf{\refcal} \textbf{(Ours)}    & \textbf{89.57} & \textbf{98.49} & \textbf{10.43} & \textbf{01.06}  & \textbf{01.02}  & \textbf{05.27}  \\
\midrule
MbLS~\cite{liu2022devil} & 87.82 & 98.59 & 07.05  & 01.58 & 01.24 & 07.05\\
\textbf{\refcal} \textbf{(Ours)}  &  \textbf{91.16} & \textbf{99.28} & \textbf{05.52}  & \textbf{01.16}   & \textbf{00.78}  & \textbf{05.09}         \\
\midrule

LogitNorm~\cite{wei2022mitigating} & 87.13 & \textbf{99.02} & \textbf{03.14}  & \textbf{01.31} & \textbf{01.21} & \textbf{03.11}\\

\textbf{\refcal} \textbf{(Ours)} & \textbf{89.02} & 98.50  & 74.19 & 10.24 & 15.03 & 50.29 \\
% \midrule
%ACLS~\cite{park2023acls} & CVPR'23& \cmark          & \cmark &  &         &           &  &        &           &  &         \\

%\textbf{ \refcal} &  \textbf{Ours}  &   \cmark       & \cmark &          &       &           &          &       &       &   \\
%\rowcolor{LightCyan}\textbf{\emph{RefCal} (SupCon with AdaFocal) }     &    \cmark      & \cmark &          &       &           &          &       &       &   \\
\bottomrule
\end{tabular}
}

\label{tab:cifar10lt}
\end{table*}

% {\hrefhttps://crossminds.ai/video/self-distillation-as-instance-specific-label-smoothing-606fec3ff43a7f2f827c10c9/}

\subsection{Performance on Large scale dataset} 

\begin{table*}[]
\centering

\resizebox{\linewidth}{!}{%
\begin{tabular}{l|c|cc|llllll|llllll}
\toprule
% \multicolumn{12}{l}{\textbf{\texttt{ResNet50} | \texttt{CIFAR100-LT dataset (imbalance factor=0.1)} \quad \quad\quad\quad \quad  \texttt{ResNet50}   \texttt{CIFAR10-LT dataset (imbalance factor=0.1)}}}   
\multirow{3}{*}{\large \textbf{Method}} &  & & &&& \textbf{\large \TINY}&&&& &&\textbf{\large \TINYLT}  \\
%\midrule
 & \textbf{Venue}  &\textbf{Ref.} & \textbf{Cal.} & \textbf{Top1 $(\%)~\uparrow$} & \textbf{AUROC~$\uparrow$}  & \textbf{ECE $(\%) ~\downarrow$}    & \textbf{SCE $(\%)~\downarrow$} & \textbf{ACE$(\%)\downarrow$}  & \textbf{smECE $(\%)~\downarrow$} &  \textbf{Top1 $(\%)~\uparrow$}  & \textbf{AUROC~$\uparrow$}  & \textbf{ECE $(\%) ~\downarrow$}    & \textbf{SCE $(\%)~\downarrow$} & \textbf{ACE$(\%)~\downarrow$}  & \textbf{smECE$(\%)~\downarrow$} \\ 
 \bottomrule
NLL (CE) & - &\cmark         &  \xmark   & 	64.64 & 98.30  & 07.34  & \textbf{00.15} & \textbf{00.09} & 07.13  & 47.66                         & 96.00                            & 11.10                         & \textbf{00.23}                         & 00.20                          & 11.05                         \\
\textbf{\refcal} &  -   &   \cmark       & \cmark & \textbf{64.74} & \textbf{98.64} & \textbf{06.21}  & \textbf{00.15} & 00.10 & \textbf{05.97}  & \textbf{47.96} & \textbf{96.37} & \textbf{01.82}  & \textbf{00.23} & \textbf{00.19} & \textbf{01.96}\\

\midrule
LS~\cite{originallabelsmoothing} & CVPR'15 & \xmark &  \cmark         & \textbf{65.22} & 97.95 & 22.36 & 00.25 & 00.33 & 21.82 & \textbf{48.95}                         & 95.10                          & 21.56                         & 00.26                         & 00.40                          & 21.10      \\
\textbf{\refcal} & \textbf{Ours}    &   \cmark       & \cmark &   64.30  & \textbf{98.59} & \textbf{07.56}  & \textbf{00.17} & \textbf{00.13} & \textbf{07.32}  & 47.15 & \textbf{96.40}  & \textbf{04.69}  & \textbf{00.22} & \textbf{00.20}  & \textbf{04.65}\\
\midrule
CE+TS~\cite{guo2017calibration} & ICML'17   & \cmark &  \cmark                            & 64.64 & 98.30  & 27.02 & 00.28 & 00.31 & 25.66 & 47.66                         & 95.30                          & 27.42                         & 00.28                         & 00.41                         & 25.86                         \\
\textbf{\refcal} & \textbf{Ours}     &   \cmark       & \cmark &  \textbf{64.74} & \textbf{98.56} & \textbf{23.87} & \textbf{00.26} & \textbf{00.25} & \textbf{23.00} & \textbf{47.96} & \textbf{96.30}  & \textbf{14.90}  & \textbf{00.26} & \textbf{00.24} & \textbf{14.72}\\
\midrule
MMCE~\cite{kumarpaper} & ICML'18 &    \xmark      &     \cmark             & 62.49 & 97.94 & 09.55  & 00.17 & \textbf{00.09} & 09.47  & 45.22                         & 95.40                          & 16.10                          & 00.26                         & 00.21                         & 15.97                         \\
\textbf{\refcal} & \textbf{Ours}      &   \cmark       & \cmark &   \textbf{64.34} & \textbf{98.59} & \textbf{05.66}  & \textbf{00.16} & 00.12 & \textbf{05.53}  & \textbf{48.04} & \textbf{96.47} & \textbf{02.21}  & \textbf{00.23} & \textbf{00.19} & \textbf{02.18}\\
\midrule
MixUp~\cite{thulasidasan2019mixup} & NeurIPS'19                  & \xmark         & \cmark & \textbf{66.93} & \textbf{98.35} & \textbf{15.64} & \textbf{00.21} & \textbf{00.20}  & \textbf{15.59} & \textbf{50.39}                         & \textbf{96.60}                          & \textbf{11.36} & \textbf{00.30} & \textbf{00.28} & \textbf{11.35} \\
\textbf{\refcal}  & \textbf{Ours}    &   \cmark       & \cmark &  63.15 & 98.26 & 18.93 & 00.25 & 00.23 & 18.74 & 47.99 & 96.29 & 17.26 & 00.28 & 00.26 & 17.14\\
\midrule
CRL~\cite{moon2020confidence} & ICML'20 &                      \cmark        &   \xmark     &  64.60  & 98.21 & 06.63  & \textbf{00.15} & \textbf{00.09} & 06.43  & 47.66                         & 96.00 & 11.10  & \textbf{00.23} & 00.20 & 11.05\\
\textbf{\refcal}  & \textbf{Ours}    &   \cmark       & \cmark & \textbf{64.74} & \textbf{98.64} & \textbf{06.21}  & \textbf{00.15} & 00.10  & \textbf{05.97}  & \textbf{47.96} & \textbf{96.37} & \textbf{01.82}  & \textbf{00.23} & \textbf{00.19} & \textbf{01.96} \\
\midrule
%Trust Score~\cite{jiang2018trust} (NeurIPS'19)                    &  \cmark        &   \xmark     & &          &       &           &          &       &       &   \\
%\textbf{\emph{RefCal}  +  TrustScore}    &   \cmark       & \cmark &          &       &           &          &       &       &   \\
%\midrule
%PSKD ~\cite{Kim_2021_ICCV} (ICCV'21)                   &    \xmark       &    \cmark       & &          &       &           &          &       &       &   \\
AUCM~\cite{yuan2021large} & ICCV'21   &   \cmark        &    \xmark                         & 31.23 & 96.32 & \textbf{17.26} & \textbf{00.23} & \textbf{00.20}  & \textbf{17.11} & 22.63 & \textbf{94.25} & \textbf{08.84}  & 00.17 & \textbf{00.22} & \textbf{08.83}\\
\textbf{\refcal}    & \textbf{Ours} &   \cmark        &    \xmark &  \textbf{56.90}  & \textbf{96.46} & 55.91 & \textbf{00.23} & 00.78 & 42.45 & 34.66 & 93.79 & 33.72 & \textbf{00.06} & 00.72 & 30.26\\
\midrule
FL+ MDCA  ~\cite{mdca}  & CVPR'22                 &  \xmark         & \cmark & \textbf{64.82} & 98.50  & \textbf{03.22}  & \textbf{00.15} & \textbf{00.09} & \textbf{03.26}  & \textbf{47.80}    & \textbf{96.60} & \textbf{02.47} & \textbf{00.23} & \textbf{00.19} & \textbf{02.44}\\
\textbf{\refcal}   & \textbf{Ours}   &   \cmark       & \cmark &   63.74 & \textbf{98.58} & 09.99  & 00.19 & 00.15 & \textbf{09.95}  & 47.56 & \textbf{96.60}  & 11.80  & 00.25 & 00.22 & 11.68\\
\midrule
AdaFocal~\cite{ghosh2023adafocal}  & NeurIPS'22                   & \xmark &    \cmark       &   60.75 & 82.41  & 39.43 & 00.29 & 00.27 & 21.86 & 40.22 & 94.70  & 28.89 & 00.38 & 00.26 & 26.91           \\
\textbf{\refcal}  & \textbf{Ours}    &   \cmark       & \cmark & \textbf{61.05} & \textbf{83.22} & \textbf{38.73} & \textbf{00.22} & \textbf{00.17} & \textbf{20.15} & \textbf{47.67} & \textbf{96.61} & \textbf{08.55}  & \textbf{00.25} & \textbf{00.22} & \textbf{08.30} \\
\midrule
%CPC~\cite{cpc} & CVPR'22   &   \xmark        & \cmark                                       & &          &       &           &          &       &       &   \\
%\textbf{\refcal}  & \textbf{Ours}    &   \cmark       & \cmark &          &       &           &          &       &       &   \\
%\midrule
MbLS~\cite{liu2022devil} & CVPR'22   &   \cmark        &    \cmark &  \textbf{65.36} & 98.05 & 10.97 & 00.19 & 00.18 & 10.97 & \textbf{48.96} & 95.50 & 10.25 & 00.23 & 00.27 & 10.24 \\
\textbf{\refcal}   & \textbf{Ours}   &   \cmark       & \cmark &  64.31 & \textbf{98.59} & \textbf{07.03}  & \textbf{00.17} & \textbf{00.12} & \textbf{06.81}  & 47.62 & \textbf{96.40}  & \textbf{04.11}  & \textbf{00.22} & \textbf{00.20}  & \textbf{04.08} \\
\midrule
LogitNorm~\cite{wei2022mitigating} & ICML'22 &   \xmark  &  \cmark                             &       63.03 & 96.35 & 62.24 & \textbf{00.01} & 00.80  & 45.30  & 47.88 & 94.11 & 47.04 & \textbf{00.01} & 00.73 & 38.06 \\
\textbf{\refcal}  & \textbf{Ours} &  \cmark & \cmark        &	\textbf{63.88} & \textbf{98.48} & \textbf{03.35}  & 00.15 & \textbf{00.09} & \textbf{03.24}  & \textbf{47.94} & \textbf{95.15} & \textbf{44.88} & 00.12 & \textbf{00.69} & \textbf{36.86} \\
\bottomrule
\end{tabular}
}
\caption{[Comparison of reliability metrics of proposed \refcal vs. \sota in a multi-class classification setting on \TINY and \TINYLT datasets]: We use \texttt{ResNet-50} backbone on \texttt{\TINY and \TINYLT(imbalance factor=$0.1$)} , ``\textbf{ + \refcal}" in bold are our approach variants. Note that \refcal surpasses \sota for train-time calibrators and refinement methods when considered independently in terms of \auroc, and \texttt{Top $1\%$ accuracy} except when paired with \texttt{Adafocal} and \texttt{AUCM}. In terms of calibration, the performance generally is better or comparable, with only occasional decrease in calibration metrics (\ece,\sce,\smce, \ace). For e.g., observe the rows (a) \texttt{CE} and the corresponding \textbf{ `+ \refcal'}, (b) \texttt{CE+TS}\cite{guo2017calibration} and the corresponding \textbf{ `+ \refcal'} and (c)  {MMCE}\cite{kumarpaper} and the corresponding \textbf{ `+ \refcal'} which jointly offer both refinement and calibration. For \ece/\sce computation, $15$ bins were used in accordance with prior work. \ace uses an adaptive binning strategy. \textbf{Numbers in bold:} best performance; The term \textbf{`Ref'} denotes the Refinement method, while \textbf{`Cal'} indicates the Calibration method. Note: \NLL: Negative log likelihood (cross-entropy loss). \textbf{Note:} For all methods, we have used the code provided by the authors. 
All models used for inference were chosen based on best top $1\%$ accuracy as per norm.}
%\midrule
%ACLS~\cite{park2023acls} & CVPR'23& \cmark          & \cmark &  &         &           &  &        &           &  &         \\

%\textbf{ \refcal} &  \textbf{Ours}  &   \cmark       & \cmark &          &       &           &          &       &       &   \\
%\rowcolor{LightCyan}\textbf{\emph{RefCal} (SupCon with AdaFocal) }     &    \cmark      & \cmark &          &       &           &          &       &       &   \\

%\ramya{doing grid search to improve values. Last resort maybe we comment out Adafocal and LogitNorm}. \soumya{Minor: Why after ECE(\%) there is a huge gap between SCE(\%). Is it intentional.}}
\label{tab:ls_tiny}
\end{table*}

% and for ACLS\cite{park2023acls} - the code was implemented by us and verified by the original authors. 

\begin{table*}[t]
\centering
\resizebox{0.9\linewidth}{!}{%
\begin{tabular}{l|llllll}
\toprule

\multirow{3}{*}{\textbf{Method}} &  & &  \textbf{\large CIFAR-10-LT} &&&   \\
%\midrule
 &  \textbf{Top1 (\%)~$\uparrow$} & \textbf{AUROC~$\uparrow$}  & \textbf{ECE (\%) ~$\downarrow$}    & \textbf{SCE (\%)~$\downarrow$} & \textbf{ACE(\%)~$\downarrow$}  & \textbf{smECE(\%)~$\downarrow$} \\ 
& &  & $\times 10^{-2}$   & $\times 10^{-2}$ & $\times 10^{-2}$    & $\times 10^{-2}$     \\ 
\bottomrule
%NLL (CE) & \cmark         &  \xmark   &          &       &           &          &       &       &   \\
NLL (CE) &  85.30  & 98.40  & 09.35  & 02.22 & 02.02 & 08.91   \\
\textbf{+  \refcal} \textbf{(Ours)}&  \textbf{89.70}  & \textbf{98.82} & \textbf{08.60}   & \textbf{01.83}  & \textbf{01.45}  & \textbf{06.85}\\

\midrule
LS~\cite{originallabelsmoothing} & 86.58 & 97.33 & 05.00 & 01.47 & 01.64 & \textbf{03.92}\\
\textbf{+  \refcal} \textbf{(Ours)}    &  \textbf{91.17}  & \textbf{99.28}  & \textbf{04.84} & \textbf{01.03}  & \textbf{00.70}  & 04.55\\
\midrule
CE+TS~\cite{guo2017calibration} & 85.30  & 98.90  & 03.42  & 01.84 & 01.86 & 03.41 \\
\textbf{+  \refcal} \textbf{(Ours)}     & \textbf{89.70}  & \textbf{99.12} & \textbf{03.09}  & \textbf{01.04}  & \textbf{01.01}   & \textbf{02.96}\\
\midrule
MMCE~\cite{kumarpaper} & 86.47 & 98.55 & \textbf{09.30}   & 02.08 & 01.85 & 08.47\\
\textbf{+  \refcal} \textbf{(Ours)} & \textbf{90.02} & \textbf{98.88} & 09.71  & \textbf{01.84}  & \textbf{01.22}  & \textbf{05.62}   \\
\midrule
MixUp~\cite{thulasidasan2019mixup} & 88.92 & \textbf{98.75} & \textbf{07.97}  & 02.54 & 02.50  & 07.51\\
\textbf{+  \refcal} \textbf{(Ours)}    &  \textbf{89.49} & 98.69 & 10.16 & \textbf{02.54}  & \textbf{02.42}  & \textbf{07.42}  \\
\midrule
CRL~\cite{moon2020confidence} & 83.16 & 97.71 & 10.90  & 02.33 & 02.08 & 10.41\\
\textbf{+  \refcal} \textbf{(Ours)}    & \textbf{89.70}  & \textbf{98.82} & \textbf{08.59}  & \textbf{01.83}  & \textbf{01.45}  & \textbf{06.85} \\
\midrule
AUCM~\cite{yuan2021large} & 88.35 & 98.62 & 05.75  & \textbf{01.38} & \textbf{01.11} & 05.79 \\
\textbf{+  \refcal (Ours)}     &  \textbf{89.39} & \textbf{99.10}  & \textbf{05.41}  & 01.91  & 01.89  & \textbf{05.41}\\
\midrule
FL+ MDCA  ~\cite{mdca}  & 87.60  & 98.41 & 05.45  & 01.42 & 01.28 & 05.40 \\
\textbf{+  \refcal (Ours)}   &  \textbf{89.12} & \textbf{99.11} & \textbf{03.16}  & \textbf{01.36}  & \textbf{01.28}  & \textbf{03.14}\\
\midrule
AdaFocal~\cite{ghosh2023adafocal} & 81.51 & 98.30  & 11.26 & 02.54 & 02.35 & 11.05\\
\textbf{+  \refcal} \textbf{(Ours)}    & \textbf{89.57} & \textbf{98.49} & \textbf{10.43} & \textbf{01.06}  & \textbf{01.02}  & \textbf{05.27}  \\
\midrule
MbLS~\cite{liu2022devil} & 87.82 & 98.59 & 07.05  & 01.58 & 01.24 & 07.05\\
\textbf{+  \refcal} \textbf{(Ours)}  &  \textbf{91.16} & \textbf{99.28} & \textbf{05.52}  & \textbf{01.16}   & \textbf{00.78}  & \textbf{05.09}         \\
\midrule

LogitNorm~\cite{wei2022mitigating} & 87.13 & \textbf{99.02} & \textbf{03.14}  & \textbf{01.31} & \textbf{01.21} & \textbf{03.11}\\

\textbf{+  \refcal} \textbf{(Ours)} & \textbf{89.02} & 98.50  & 74.19 & 10.24 & 15.03 & 50.29 \\
% \midrule
%ACLS~\cite{park2023acls} & CVPR'23& \cmark          & \cmark &  &         &           &  &        &           &  &         \\

%\textbf{ +  \refcal} &  \textbf{Ours}  &   \cmark       & \cmark &          &       &           &          &       &       &   \\
%\rowcolor{LightCyan}\textbf{\refcal (SupCon with AdaFocal) }     &    \cmark      & \cmark &          &       &           &          &       &       &   \\
\bottomrule
\end{tabular}
}
\caption{[\textbf{Comparison of reliability metrics in multi-class classification: Accuracy, Calibration, and Refinement of our approach, \refcal  vs. \sota}] on \texttt{ImageNet-LT} \cite{openlongtailrecognition} using \texttt{ResNet-50} \cite{he2016deep} feature extractor. Notice that \refcal variants outperform contemporary methods in majority of cases.}
%Additional binary classification experiments are a part of supplementary.  } %\textcolor{red}{CPC and ACLS if time permits. Yet to check bold in entries in row-pairs. Rationale for choosing imbalanced dataset??}}
\label{tab:ImageNetLT}
\end{table*}

% {\hrefhttps://crossminds.ai/video/self-distillation-as-instance-specific-label-smoothing-606fec3ff43a7f2f827c10c9/}
In the main manuscript, we show the performance of \refcal training regime on the datasets: \CIFARH, \CIFARHLT, \texttt{ImageNet-1K}, \texttt{ImageNet-LT}, and \STLT. In the supplemental material, we include large-scale datasets where \refcal variants are consistently superior to the \sota on both regular and long-tailed versions of the datasets: \TINY and \TINYLT as shown in Tab. \ref{tab:ls_tiny}. From Tab. \ref{tab:ls_tiny}, we infer \refcal variants namely, \refcal + \texttt{NLL}, \refcal + \texttt{LS}, \refcal + \texttt{CE+TS}, \refcal + \texttt{MMCE}, \refcal + \texttt{CRL}, \refcal + \texttt{Adafocal}, \refcal + \texttt{LogitNorm} offer performance better than the methods considered independently. Observe a similar trend on \CIFARHLT in Tab. \ref{tab:cifar10lt}. %\refcal+\texttt{MbLS}.

\subsection{Robustness of \textbf{\refcal}}

\subsubsection{Robustness to natural corruptions} 
\dnns lack robustness to out-of-distribution data or natural corruptions such as noise, blur, etc. Hendrycks \etal ~\cite{hendrycks2018benchmarking} benchmark robustness of a \dnn using $15$  algorithmically generated image corruptions that mimic natural corruptions. Each corruption severity ranges from $1$ to $5$ based on the intensity of corruption, where $5$ is the most severe. The corruption can be seen as perturbing a sample point in its local neighborhood while remaining in the support space of the probability distribution of valid images. We use \CIFARHC dataset to measure trained model performance on such corruptions. Tab. \ref{tab:CIFAR100_C} reports the results on \CIFARHC and in particular, results are for Gaussian noise with varying degrees of corruption. Observe that our training not only yields higher accuracy, \auc, but also favorable calibration values.

\subsubsection{Robustness to out-of-distribution (\ood) samples (semantic shift)} 
A well-calibrated model should not only exhibit low confidence whenever it misclassifies but also in situations when it encounters data that belongs to a new/different class different than the training classes. In literature, these are referred to as \ood samples \cite{hebbalaguppe2022novel,liang2017enhancing}. We investigate if \refcal can acquire representations that demonstrate increased resilience to \ood samples by assigning them low confidence. We compare the performance of various approaches using \texttt{FPR@0.95TPR}, \auc, and \texttt{Detection Error}, as reported in \cite{liang2017enhancing}. See the Supplementary for a description of evaluation metrics. Tab. \ref{tab:ood} summarizes the \ood detection performance of our model and highlights the capability to reject such unknown samples by increasing \auc between the in-distribution and \ood samples and decreasing both \texttt{FPR@0.95TPR}, and \texttt{Detection Error} . Notice with and without \refcal rows in Tab. \ref{tab:ood} which indicates we have the best values in comparison to \sota calibrators and refinement techniques.
% This is the corruption Table

\begin{table*}[t]
\centering
\resizebox{\linewidth}{!}{%
\begin{tabular}{l|llllll}
\toprule
% \multicolumn{12}{l}{\textbf{\texttt{ResNet50} | \texttt{CIFAR100-LT dataset (imbalance factor=0.1)} \quad \quad\quad\quad \quad  \texttt{ResNet50}   \texttt{CIFAR10-LT dataset (imbalance factor=0.1)}}}   
\multirow{3}{*}{\textbf{Method}} & && \textbf{\large CIFAR-100C} \\
%\midrule
 &  \textbf{Top1 (\%)~$\uparrow$} & \textbf{AUROC~$\uparrow$}  & \textbf{ECE (\%) ~$\downarrow$}    & \textbf{SCE (\%)~$\downarrow$} & \textbf{ACE(\%)~$\downarrow$}  & \textbf{smECE(\%)~$\downarrow$} \\ 
 &          & &  &$\times 10^{-2}$   & $\times 10^{-2}$ & $\times 10^{-2}$ \\ 
\bottomrule
%NLL (CE) & \cmark         &  \xmark   &          &       &           &          &       &       &   \\
NLL (CE) & 24.81	& 84.08	& 24.73	& 00.80	& 00.79 & 23.84\\
\textbf{+  \refcal (Ours)} & \textbf{32.97} &	\textbf{91.72} & \textbf{16.93} &	\textbf{00.73} &	\textbf{00.71} &	\textbf{16.79} \\

\midrule
LS~\cite{originallabelsmoothing} & 23.53 & 	80.08 & 	\textbf{03.36} & 	\textbf{00.35}	& \textbf{00.49} & 	\textbf{03.32}\\
\textbf{+  \refcal (Ours)} &  \textbf{30.01}	& \textbf{91.75}	& 14.86	& 00.56	& 00.93	& 14.78\\
\midrule
CE+TS~\cite{guo2017calibration} & 24.81  &	85.09 &	05.30 &	00.57 &	00.62 &	05.22   \\
\textbf{+  \refcal (Ours)} & \textbf{32.97} &	\textbf{92.54} &	\textbf{02.17}	& \textbf{00.59} &	\textbf{00.69} &	\textbf{02.18}\\
\midrule
MMCE~\cite{kumarpaper} & 23.54	& 84.15	& 24.85 &	00.77 &	00.78 &	23.90\\
\textbf{+  \refcal (Ours)} & \textbf{32.22}	& \textbf{93.21} &	\textbf{02.00}	& \textbf{00.65} &	\textbf{00.76}	& \textbf{02.00} \\
\midrule
%MixUp~\cite{thulasidasan2019mixup} & \\
%\textbf{+  \refcal (Ours)}  &    \\
%\midrule
CRL~\cite{moon2020confidence} &  23.34	& 86.62	& \textbf{07.21} &	\textbf{00.62} &	\textbf{00.64} &	\textbf{07.11}\\
\textbf{+  \refcal (Ours)}  & \textbf{32.96}	& \textbf{91.72} &	16.94 &	00.73 &	00.71	&16.79  \\
\midrule
%Trust Score~\cite{jiang2018trust} (NeurIPS'19)                    &  \cmark        &   \xmark     & &          &       &           &          &       &       &   \\
%\textbf{\refcal (Ours)  +  TrustScore}    &   \cmark       & \cmark &          &       &           &          &       &       &   \\
%\midrule
%PSKD ~\cite{Kim_2021_ICCV} (ICCV'21)                   &    \xmark       &    \cmark       & &          &       &           &          &       &       &   \\
AUCM~\cite{yuan2021large} &  19.02	& 83.71	 &  \textbf{05.75} & 	00.68	& \textbf{00.68}	 & \textbf{05.57} \\
\textbf{+  \refcal (Ours)}  &   \textbf{31.36} & 	\textbf{92.95} & 	29.47	&  \textbf{00.08} & 	01.34 & 	27.36\\
\midrule
FL+ MDCA  ~\cite{mdca}  & 22.83	& 85.03	& 11.90	& \textbf{00.69} & 00.72	& 11.89\\
\textbf{+  \refcal (Ours)}  &  \textbf{32.25} &	\textbf{91.98} & \textbf{10.98} &	\textbf{00.69}	& \textbf{00.70}	& \textbf{10.96} \\
\midrule
AdaFocal~\cite{ghosh2023adafocal}  & \textbf{34.76} &91.69 &37.89 &00.94 &00.82& 32.83    \\
\textbf{+  \refcal (Ours)}  & 32.43 &\textbf{92.61} &\textbf{07.78} &\textbf{00.68}&\textbf{00.70}&\textbf{07.78}\\
\midrule
%CPC~\cite{cpc} & CVPR'22   &   \xmark        & \cmark                                       & &          &       &           &          &       &       &   \\
%\textbf{+  \refcal (Ours)}  & \textbf{Ours}    &   \cmark       & \cmark &          &       &           &          &       &       &   \\
%\midrule
MbLS~\cite{liu2022devil} &  23.47	& 0.80	&  19.12	& 00.64	& \textbf{00.66} & 	18.97\\
\textbf{+  \refcal (Ours)}   & \textbf{29.95} & \textbf{92.11} &	\textbf{09.30} & 	\textbf{00.62}	& 00.85	& \textbf{09.29} \\
\midrule

LogitNorm~\cite{wei2022mitigating} &  25.56	& 84.34	& \textbf{24.03} &	00.05	& \textbf{01.01}	& \textbf{23.18}          \\

\textbf{+  \refcal (Ours)}  &   \textbf{33.93} &	\textbf{90.67} &	32.61	& \textbf{00.03}	& 01.28	& 29.55\\
%\midrule
%ACLS~\cite{park2023acls} & CVPR'23& \cmark          & \cmark &  &         &           &  &        &           &  &         \\

%\textbf{ +  \refcal (Ours)} &  \textbf{Ours}  &   \cmark       & \cmark &          &       &           &          &       &       &   \\
%\rowcolor{LightCyan}\textbf{\refcal (Ours) (SupCon with AdaFocal) }     &    \cmark      & \cmark &          &       &           &          &       &       &   \\
\bottomrule
\end{tabular}
}
\caption{[\textbf{Robustness to corruptions]: Comparison of reliability metrics of \refcal (Ours)  vs. \sota}]: We use \texttt{ResNet-50} \cite{he2016deep} backbone on \texttt{CIFAR100C} \cite{hendrycks2018benchmarking} The results are averaged with 5 degrees of severity of Gaussian noise. We observe that \refcal scores are consistently on-par/superior over train-time calibrators and refinement methods.}
\label{tab:CIFAR100_C}
\end{table*}

% \textcolor{red}{to be expanded} 

\begin{table*}[t]
\centering
\resizebox{0.8\linewidth}{!}{%
\begin{tabular}{l|llllll}
\toprule
% \multicolumn{12}{l}{\textbf{\texttt{ResNet50} | \texttt{CIFAR100-LT dataset (imbalance factor=0.1)} \quad \quad\quad\quad \quad  \texttt{ResNet50}   \texttt{CIFAR10-LT dataset (imbalance factor=0.1)}}}   
\multirow{3}{*}{\textbf{Method}} & && \textbf{\large CIFAR-10} \\
%\midrule
 &  \textbf{FPR@TPR95 ~$\downarrow$} & \textbf{Det Error ~$\downarrow$}  & \textbf{AUROC~$\uparrow$}    & 
 \textbf{AUPR In ~$\uparrow$} & \textbf{AUPR Out ~$\uparrow$}  \\ 
\bottomrule
%NLL (CE) & \cmark         &  \xmark   &          &       &           &          &       &       &   \\
NLL (CE) & 30.50 &08.80 & 95.90& 96.90 & 94.70 \\
\textbf{+  \refcal (Ours)} & \textbf{16.30} &\textbf{07.10} & \textbf{97.50} & \textbf{97.70} &\textbf{97.00} \\

\midrule
LS~\cite{originallabelsmoothing} & 29.30 & 09.30 & 92.60 & 86.80 & 92.40\\
\textbf{+  \refcal (Ours)} & \textbf{17.00} & \textbf{08.20} & \textbf{97.10} & \textbf{97.60} & \textbf{96.30}\\
\midrule
CE+TS~\cite{guo2017calibration} &  22.20 & 08.10 &97.00 &97.60 &96.40\\
\textbf{+  \refcal (Ours)} & \textbf{14.10} & \textbf{06.80} &\textbf{97.90}& \textbf{98.20} &\textbf{97.60}\\
\midrule
MMCE~\cite{kumarpaper} & \textbf{12.10} &\textbf{06.90} &\textbf{98.30} &\textbf{98.40} &\textbf{98.20} \\
\textbf{+  \refcal (Ours)} &20.80 &10.30&97.30&85.80&95.60  \\
\midrule
%MixUp~\cite{thulasidasan2019mixup} & \\
%\textbf{+  \refcal (Ours)}  &    \\
%\midrule
CRL~\cite{moon2020confidence} & 27.90&08.20&96.30&97.20&95.30 \\
\textbf{+  \refcal (Ours)}  & \textbf{16.30}&\textbf{07.10}&\textbf{97.50}&\textbf{97.70}&\textbf{97.00}  \\
\midrule
%Trust Score~\cite{jiang2018trust} (NeurIPS'19)                    &  \cmark        &   \xmark     & &          &       &           &          &       &       &   \\
%\textbf{\refcal (Ours)  +  TrustScore}    &   \cmark       & \cmark &          &       &           &          &       &       &   \\
%\midrule
%PSKD ~\cite{Kim_2021_ICCV} (ICCV'21)                   &    \xmark       &    \cmark       & &          &       &           &          &       &       &   \\
AUCM~\cite{yuan2021large} & 38.30&12.10&93.40&93.60&92.40  \\
\textbf{+  \refcal (Ours)}  & \textbf{10.20}&\textbf{06.60}&\textbf{98.20}&\textbf{98.50}&\textbf{98.00}  \\
\midrule
FL+ MDCA  ~\cite{mdca}  & 20.50&08.00& 97.10&\textbf{97.70}&\textbf{96.60}\\
\textbf{+  \refcal (Ours)}  & \textbf{18.00}&\textbf{07.20}&\textbf{97.30}&97.40&96.50  \\
\midrule
%AdaFocal~\cite{ghosh2023adafocal}  &    \\
%\textbf{+  \refcal (Ours)}  & \\
%\midrule
%CPC~\cite{cpc} & CVPR'22   &   \xmark        & \cmark                                       & &          &       &           &          &       &       &   \\
%\textbf{+  \refcal (Ours)}  & \textbf{Ours}    &   \cmark       & \cmark &          &       &           &          &       &       &   \\
%\midrule
MbLS~\cite{liu2022devil} & 25.00&08.20&96.40&97.10&95.10  \\
\textbf{+  \refcal (Ours)}   &\textbf{16.90} &\textbf{07.80}&\textbf{97.20}&\textbf{97.70}&\textbf{96.60}  \\
\bottomrule
\end{tabular}
}
\caption{[\textbf{Robustness to \ood data]: Comparison of reliability metrics of \refcal (Ours)  vs. \sota}]: We use \texttt{ResNet-50} backbone on  in-distribution dataset \texttt{CIFAR10} and tested on \ood dataset \texttt{SVHN} for classification. }
\label{tab:ood}
\end{table*}

\subsection{Performance of \refcal to natural corruptions and OOD data}
\begin{figure*}[]
\centering
\begin{subfigure}{\linewidth}
\centering
\includegraphics[width=0.6\linewidth]{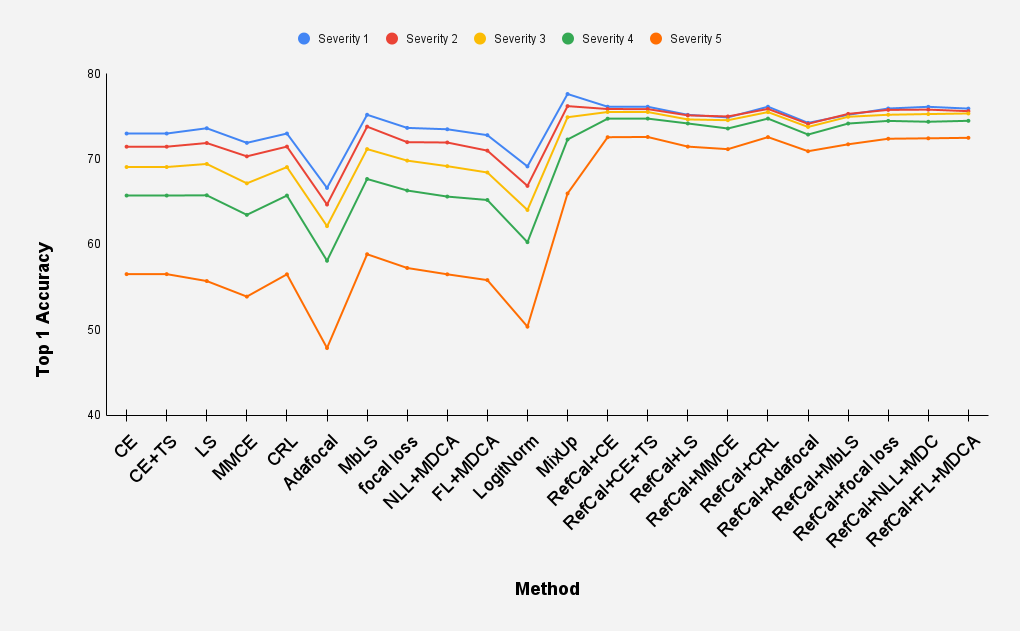}
\caption{Variation of Top 1\% accuracy }
\label{fig:sub1}
\end{subfigure}

\begin{subfigure}{\linewidth}
\centering
\includegraphics[width=0.6\linewidth]{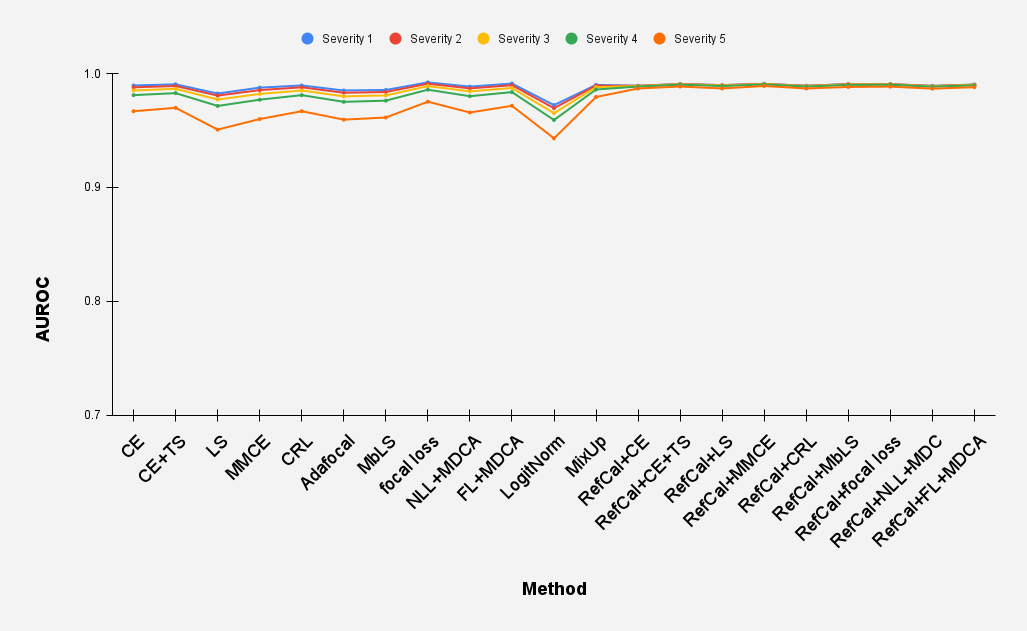}
\caption{Variation of AUROC}
\label{fig:sub2}
\end{subfigure}

\begin{subfigure}{\linewidth}
\centering
\includegraphics[width=0.6\linewidth]{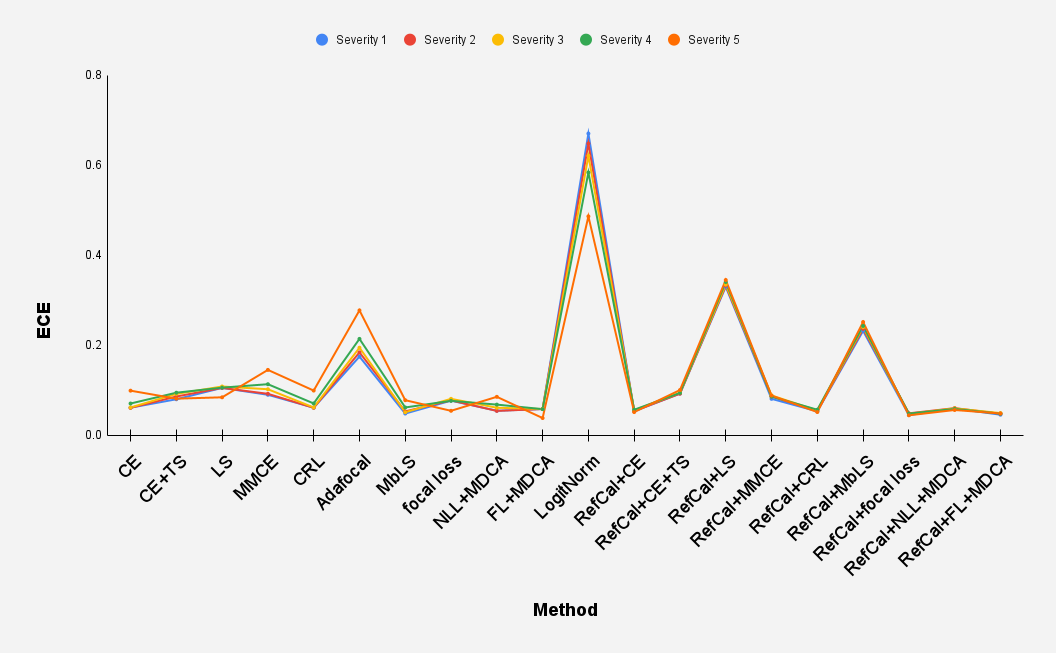}
\caption{Variation of calibration error (ECE)}
\label{fig:sub2}
\end{subfigure}
\caption{\textbf{[Robustness of \refcal to varying degree of corruptions in case of \texttt{Brightness} distorsion from \cite{hendrycks2018benchmarking}]}: We report results  on \CIFARHC using \texttt{ResNet50} \cite{he2016deep}. We observe that \refcal variants offer superior performance in both Top 1\% accuracy and \auroc in comparison with any train-time calibration and \sota refinement methods taken independently. Our approach shows relatively lower degradation in comparison to \sota. \ece exhibits lower variance with train-time calibrators which have higher \ece and wider variance showing lower reliability except FL+MDCA. }
%We recommend the following configurations (\refcal + \texttt{FL+MDCA}, \refcal + \texttt{focal loss}) as both metrics: \auroc and \ece are optimized.
\label{fig:brightnessDist}
\end{figure*}

Tab. \ref{fig:brightnessDist} and Tab. \ref{fig:frostDist} report the results on \CIFARHC when trained on \CIFARH and in particular, results for Brightness and frost distorsions with varying degrees of corruption. Here we provide results for a few types of corruption to study the robustness of \refcal training regime. 

\noindent Tab. \ref{fig:brightnessDist} illustrates the effect of varying the \texttt{brightness} distortion whose severity is varied from 1 to 5 and reliability metrics are recorded. \refcal variants, demonstrates an improvement over the state-of-the-art (\sota) in terms of both accuracy ($\uparrow$) and AUROC ($\uparrow$). Notice low variance in the metrics across different corruption severities, indicating better reliability as the degradation is relatively low. Moreover, similar trends are observed in the case of \ece (expected calibration error), where \refcal not only has lower or comparable calibration error to \sota but also exhibits lower variance with increased corruption.

%\noindent Fig. \ref{fig:frostDist} illustrates the effect of varying the \texttt{frost} distortion whose severity is varied from 1 to 5 and reliability metrics are recorded. \refcal variants, demonstrates an improvement over the state-of-the-art (\sota) in terms of both accuracy ($\uparrow$) and AUROC ($\uparrow$). Notice low variance in \auroc across different corruption severities, indicating better reliability. Moreover, similar trends are observed in the case of \ece (expected calibration error), where \refcal not only has lower or comparable calibration error to \sota but also exhibits lower variance with increased corruption severity. 

\begin{figure*}[]
	%\vspace{1cm}
	\centering
	\hspace*{-0.6cm}\includegraphics[angle=0, width=\linewidth]{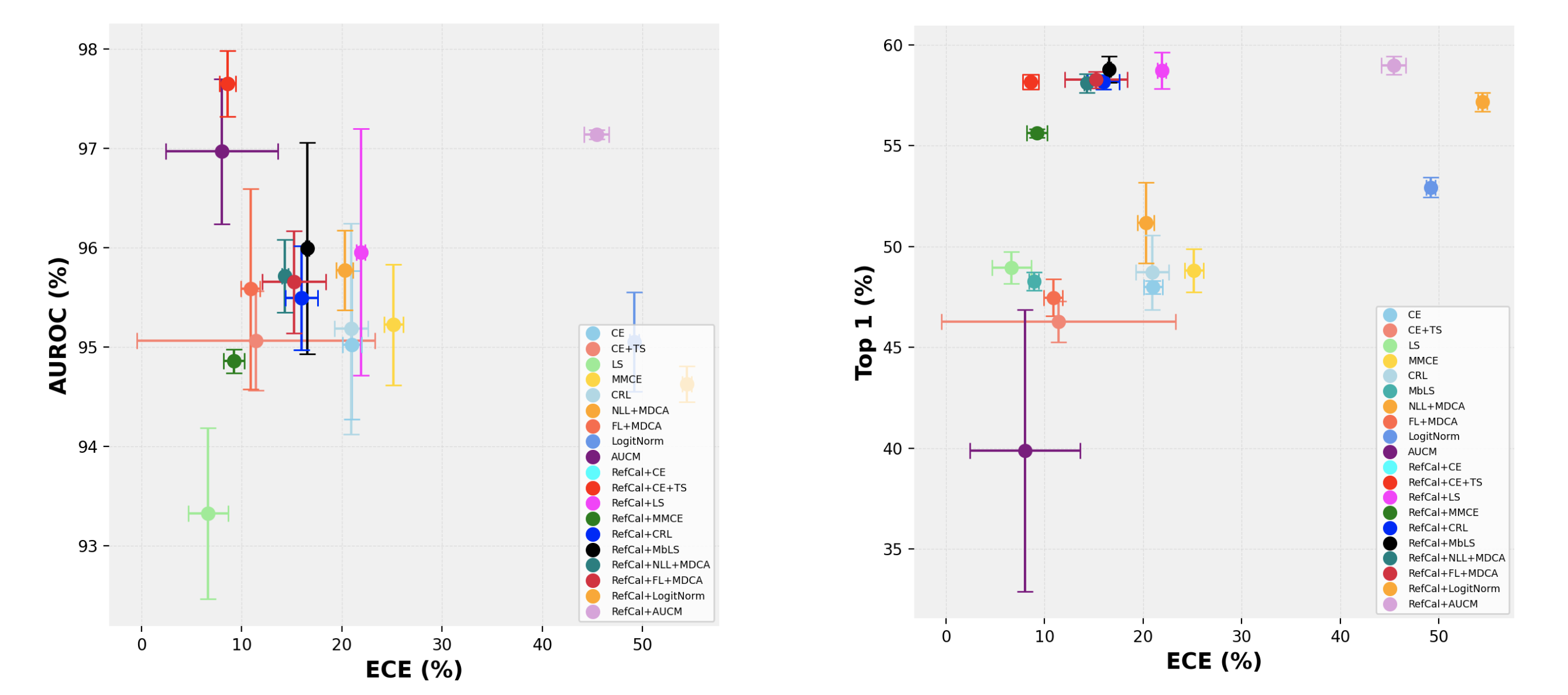}
% \vspace{-0.5cm}
\caption{\textbf{(Left)} \textbf{\auc vs. \ece trade-off} and \textbf{(Right)} \textbf{\texttt{Top 1\% accuracy} vs. \ece trade-off}. We use ResNet-50 trained on CIFAR100-LT (imbalance factor 10\%). Higher \auc and lower ECE is better. Higher Top $1\%$ accuracy with lower calibration error (Top-left location) is desirable. The lower variance in \auc and \ece in the plot emphasizes the reliability of \refcal variants. 
%Note the variant \texttt{RefCal+CE+TS} (ours) offers the highest \auc-low \ece with the next best being \texttt{AUCM}\cite{yuan2021large}. 
Note that \texttt{RefCal+CE+TS} performs the best. We have run every technique $3$ times for fair comparison.  } 
\label{fig:error_bar_suppl}
\end{figure*}

\section{Rationale on the superior performance of our \refcal framework} 
\label{sec:rationale}

The supervised contrastive loss is used as a surrogate for refinement. Thus refinement loss is structured so that gradients provide an intrinsic mechanism for hard positive/negative mining during training. For losses that are aimed to improve refinement like ranking losses, or max-margin type of losses, hard mining is known to be crucial to their performance. The authors of \cite{khosla2020supervised} show analytically that hard mining is intrinsic and thus removes the need for complicated hard mining algorithms. We conjecture this as the reason for strong performance on refinement across various datasets.

To enhance the understanding of effect of refinement and why it is necessary, we consider a case of \texttt{MbLS} and \texttt{MbLS}

\begin{comment}

\begin{figure*}[]

\begin{subfigure}{\linewidth}
\centering
\includegraphics[width=0.4\linewidth]{figs_plots/Addnl_WhyRefCal/MBLS_auroc_93.00_ece_0.0826.pdf}
\caption{\texttt{MbLS}\cite{liu2022devil}}
\label{fig:sub1}
\end{subfigure}
\begin{subfigure}{\linewidth}
\centering
\includegraphics[width=0.4\linewidth]{figs_plots/Addnl_WhyRefCal/RefCal+MBLS_auroc_96.62_ece_0.08629.pdf}
\caption{\refcal+\texttt{MbLS}\cite{liu2022devil}}
\label{fig:sub2}
\end{subfigure}
\end{comment}
\begin{figure}[t]
\centering

\begin{subfigure}{0.48\linewidth}
    \centering
    \includegraphics[width=\linewidth]{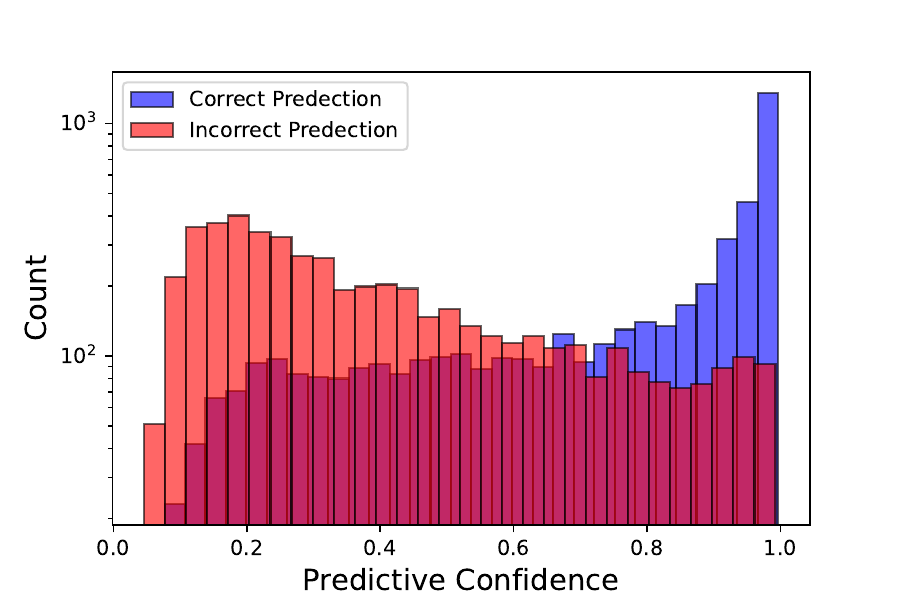}
    \caption{\texttt{MbLS}\cite{liu2022devil}}
    \label{fig:sub1}
\end{subfigure}
\hfill
\begin{subfigure}{0.48\linewidth}
    \centering
    \includegraphics[width=\linewidth]{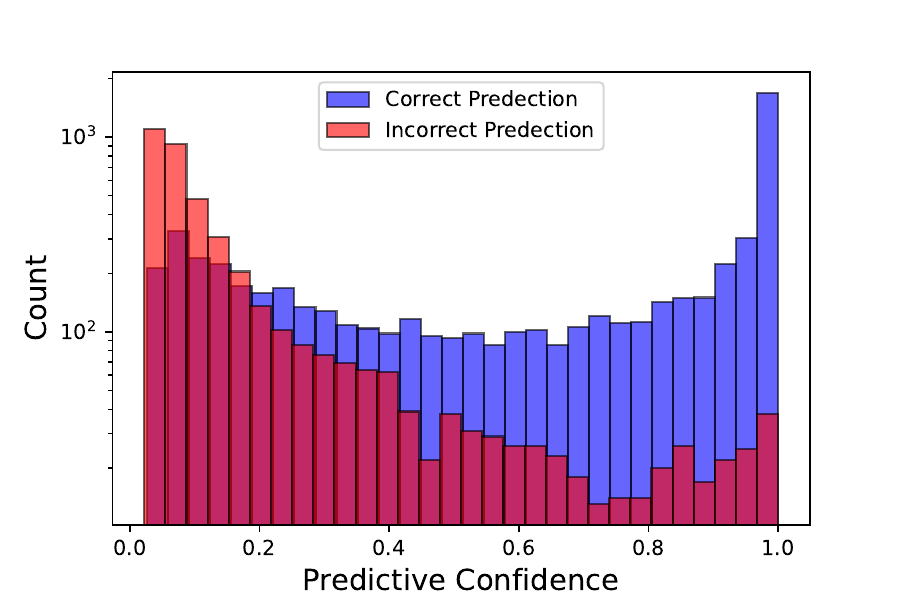}
    \caption{\refcal+\texttt{MbLS}\cite{liu2022devil}}
    \label{fig:sub2}
\end{subfigure}
\caption{\textbf{[Why RefCal? Example 2]}: Our proposed training regime, \refcal allows for joint optimization of calibration and refinement. \textbf{(a)} shows when we calibrate a \texttt{ResNet-50} model using \texttt{MbLS}\cite{liu2022devil} calibration on \CIFARHLT, it results in lower separation between the confidence values of the two classes (red and blue). To overcome this, we jointly train using the proposed proxy refinement loss along with calibration using \texttt{MbLS}\cite{liu2022devil}. It can be seen that the separability is enhanced as shown in \textbf{(b)}, indicating better refinement. Quantitatively, (\auc($\uparrow$), \ece($\downarrow$)) for (refinement, calibration) improve from ($93.00$, $0.087$) for \texttt{MbLS}\cite{liu2022devil} to ($96.62$, $0.086$) for \texttt{MbLS}\cite{liu2022devil} + proposed refinement. Note in the main text we gave an example teaser image of \refcal+\mmce vs. \mmce\cite{kumarpaper}}
\end{figure}

%\end{figure*}

\subsection{Grad-CAM visualization}
\label{subsec:gradcam}

\begin{figure*}[]
    \centering
    \includegraphics[width=1\linewidth]{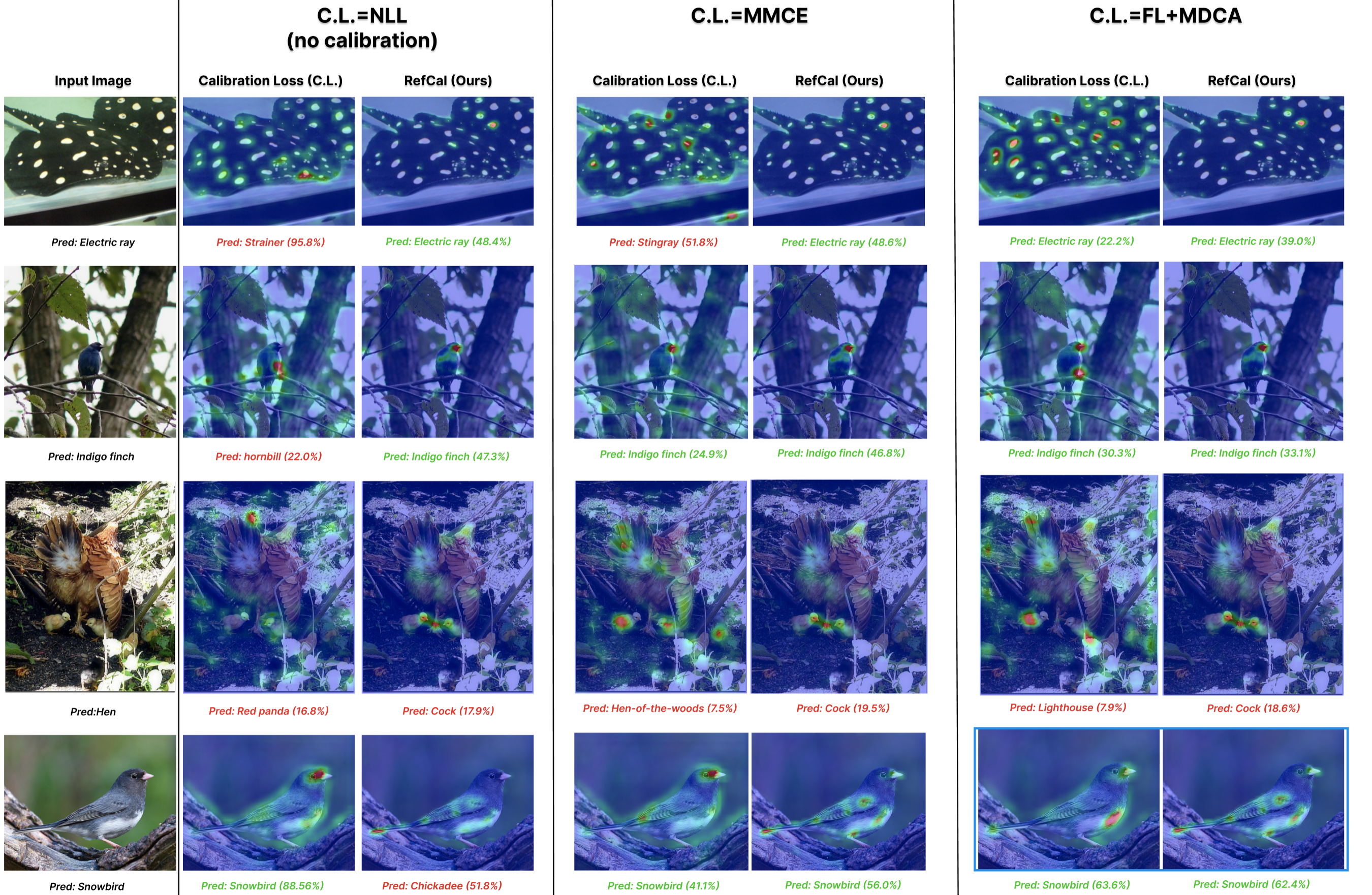}
    \caption{\textbf{[Grad-CAM]}: Our proposed training regime, \refcal allows for joint optimization of calibration and refinement. Grad-CAM visualizations for \texttt{Resnet-18} trained on \texttt{ImageNet-LT}, using a particular calibration technique (column bearing title Calibration Loss (C.L.)), and then by jointly optimizing with the same calibration but adding our refinement loss (columns bearing title \refcal i.e., columns 3,5, and 7). Here, the Calibration Loss tested are: NLL (no calibration) in column 2, MMCE\cite{kumarpaper} in column 4, and FL+MDCA\cite{hebbalaguppe2022novel} in column 6. \textbf{Note:} Numbers in parentheses indicate predictive confidence for each method. Since, the refinement loss forces a model to separate its confidence for positive, and negative samples, it also leads the model to focus its attention on the salient object features, instead of the background.}
    \label{fig:gradcamsuppl1}
\end{figure*}

Our observations from Fig. \ref{fig:gradcamsuppl1} reveal that the model trained with \refcal exhibits a more focus on semantic attributes of the foreground object critical for accurate classification.

\section{Training and compute details}
\label{sec:trainingDetails}

In this section, we provide a detailed summary of the hyperparameters and training techniques used, in order to ensure reproducibility. All models have been trained on 40GB Nvidia A100 GPU. The code was written using the PyTorch framework. We make use of automatic mixed precision training in order to reduce training time.  Details of  various hyperparameters used to record both the calibration and accuracy/\auroc are shown in Tabs. \ref{tab:hyperparam_cifar10}, \ref{tab:hyperparam_cifar10_imb},
\ref{tab:hyperparam_cifar100},
\ref{tab:hyperparam_cifar100_imb},
\ref{tab:hyperparam_tiny_imagenet}, \ref{tab:hyperparam_tiny_imagenet_imb},
\ref{tab:hyperparam_stl10_binary}. These hyperparameters enable us to reproduce all the tables in main and supplemental material. We used a fixed seed value of $1234$ across all datasets and architectures for \refcal.

%\todo{INSERT 6 tables here}

\subsection{Baselines}

\begin{itemize}
\item \textbf{CE+TS} \cite{guo2017calibration} The authors find that contemporary neural networks exhibit poor calibration, a departure from neural networks developed a decade ago. Extensive experiments reveal that factors such as depth, width, weight decay, and Batch Normalization play crucial roles in influencing calibration. Surprisingly, they observe that Temperature Scaling (\ts), a variant of Platt Scaling with a single parameter, is highly effective in calibrating predictions. \texttt{CE+TS} applies \texttt{TS} with \texttt{cross entropy loss}.
\item \textbf{Focal Loss} \cite{lin2017focal} A regular focal
loss (with fixed $\gamma$) improves the overall calibration by preventing samples from being over-confident. \cite{focallosspaper} is defined as $\texttt{FL} = - \sum_{i=1}^N (1-\hat{p}_{i,y_i})^{\gamma} \log (\hat{p}_{i,y_i})$, where $\gamma$ is a hyper-parameter. We trained using $\gamma \in \{1,2,3\}$ and reported it as \texttt{FL} in the results.

\item \textbf{AdaFocal} \cite{ghosh2023adafocal} Training with focal loss leads to better calibration than cross-entropy while achieving a similar level of accuracy \cite{focallosspaper}. This success stems from focal loss regularizing the entropy of the model's prediction (controlled by the parameter $\gamma$), thereby reining in the model's overconfidence. Further improvement is expected if $\gamma$ is selected independently for each training sample (Sample-Dependent Focal Loss (FLSD-53). However, FLSD-53 is based on heuristics and does not generalize well. In this paper, we propose a calibration-aware adaptive focal loss called AdaFocal that utilizes the calibration properties of focal (and inverse-focal) loss and adaptively modifies $\gamma_t$ for different groups of samples based on $\gamma_ {t-1}$ from the previous step and the knowledge of model's under/over-confidence on the validation set.

\item \textbf{Label Smoothing}\cite{originallabelsmoothing}: Label Smoothing (\LS) serves as a regularization technique designed to address potential inaccuracies within datasets. It recognizes that maximizing the likelihood directly, denoted as ${P(y|\mathbf{x})}$, may be detrimental due to the possibility of errors in the training labels. To mitigate this issue, \LS introduces controlled noise into the labeling process. In essence, \LS operates as follows: Given a small constant value $\epsilon$, it considers the training label $y$ to be correct with a probability of $(1-\epsilon)$ and incorrect otherwise. Specifically, in the context of a softmax model with $k$ outputs, it replaces the traditional binary classification targets of $0$ and $1$ with modified targets. These modified targets consist of $\frac{\epsilon}{(k-1)}$ for incorrect labels and $(1-\epsilon)$ for correct labels~\cite{originallabelsmoothing,labelsmoothinghelp}. This approach ensures that all output probabilities undergo uniform regularization, thereby helping to combat overfitting and improve model generalization.
\item \textbf{MbLS} \cite{liu2022devil} Miscalibration in odern \dnns may be intensified by overfitting, which occurs as a result of minimizing cross-entropy during training. This minimization encourages the predicted softmax probabilities to closely align with the one-hot label assignments. Consequently, the pre-softmax activation for the correct class becomes notably larger than the activations for the other classes. MbLS is a simple and flexible generalization based on inequality constraints, which imposes a controllable margin on logit distances.
\item \textbf{CRL} The novel loss function introduced by \cite{moon2020confidence} is called Correctness Ranking Loss (CRL). It explicitly enforces regularization on class probabilities to enhance their role as reliable confidence estimates, particularly in the context of ordinal ranking based on confidence. In other words, the loss is minimized when the confidence estimates for samples with high probabilities of correctness surpass those for samples with lower probabilities of correctness.
\item \textbf{MMCE} \cite{kumarpaper} MMCE represents a calibration measure relying on a Reproducing Kernel Hilbert Space (RKHS) kernel, and it can be trained efficiently in conjunction with the negative likelihood loss without the need for meticulous hyperparameter tuning. In theoretical terms, MMCE is a reliable calibration measure, reaching its minimum at perfect calibration. Moreover, its finite sample estimates are consistent, and the method exhibits rapid convergence rates.
\item \textbf{NLL+MDCA} \cite{mdca}
and \textbf{FL+MDCA} \cite{mdca}: The authors suggest a new auxiliary loss function called Multi-class Difference in Confidence and Accuracy (MDCA), which enhances both calibration and classification when combined with classification losses like focal loss or cross-entropy loss. The amalgamation is denoted as ``FL+MDCA," when paired with focal loss, and when paired with cross-entropy loss, it is labeled as ``CE+MDCA" in the tables.
\item \textbf{AUCM} \cite{yuan2021large} Authors propose a new margin-based min-max surrogate loss function for the AUC score (named AUC min-max-margin loss or simply AUC margin loss (\AUCM) for short). It is more robust than the commonly used AUC square loss while enjoying the same advantage in terms of large-scale stochastic optimization
\item \textbf{LogitNorm}\cite{wei2022mitigating}: DNNs are known to suffer from the overconfidence issue, where they produce abnormally high confidence for both in- and out-of-distribution inputs. Logit Normalization is \Logitnorm in short fixes the cross-entropy loss by enforcing a constant vector norm on the logits during training. 
\item \textbf{MixUp} \cite{thulasidasan2019mixup} mixup training benefits in calibration and predictive uncertainty of models trained with mixup data augmentation. During mixup training, the additional samples are generated during training by convexly combining random pairs of images and their associated labels.

\end{itemize}

\subsection{Datasets}

This section contains a description of the different image classification datasets used in our experiments. The choice of datasets is inspired by previous works in Calibration and refinement \cite{mdca,yang2021learning,kumarpaper}

\begin{itemize}

	\item \textbf{CIFAR-10:} The CIFAR-10 dataset \cite{Krizhevsky_2009_17719-cifar10-100} contains $32\times32$ images belonging to ten general object classes like airplane, cat, horse, truck, etc. The training set contains $50,000$ images ($5000$ from each class), and the test set contains $10,000$ images.
 	\item \textbf{CIFAR-10LT} \cite{Krizhevsky_2009_17719-cifar10-100} is an imbalanced dataset is comprised of under 60,000 color images, each measuring 32x32 pixels, distributed across 10 distinct classes. The number of samples within each class decreases exponentially with factors of 10 and 100. The dataset includes 10,000 test images, with 1000 images per class, and fewer than 50,000 training images. Each image is assigned one label.

  \item \textbf{STL-10}\cite{coates2011analysis} STL10 is a dataset inspired by the CIFAR-10 dataset, featuring certain modifications. Notably, each class in STL10 has fewer labeled training examples compared to CIFAR-10. The higher resolution of images in this dataset (96x96) poses a challenge, making it a rigorous benchmark for both supervised and the development of more scalable unsupervised learning methods.

  	\item \textbf{SVHN:} Street View House Numbers (SVHN) \cite{yuval2011reading} is a digit classification dataset, containing $32\times32$ images of digits from house numbers plates. The images are cropped around a single character of interest, albeit with some distractors on the sides. The dataset contains 10 classes, one for each digit from 0 to 9. There are $73,527$  images in the training set and $26,032$ images in the test set.

    	\item \textbf{CIFAR-100:} The CIFAR-100 dataset \cite{Krizhevsky_2009_17719-cifar10-100}  has 100 classes containing 600 images each. There are 500 training images and 100 testing images per class. The 100 classes in the CIFAR-100 are grouped into 20 superclasses. Each image comes with a ``fine" label (the class to which it belongs) and a ``coarse" label (the superclass to which it belongs).

	\item \textbf{CIFAR-100-LT:} This is a Long-tailed Version of CIFAR100, an  imbalanced dataset that consists of fewer than 60,000 color images, each sized 32x32 pixels, and is categorized into 100 different classes. The distribution of samples within each class follows an exponential decrease, with factors of 10 and 100. The dataset comprises fewer than 50,000 training images and 10,000 test images, with 100 images per class in the test set. These 100 classes are further grouped into 20 overarching superclasses. Each image is labeled with two annotations: a fine label indicating the specific class, and a coarse label representing the corresponding superclass.

  \item \textbf{CIFAR-100-C} is a dataset from Benchmarking Neural Network Robustness to Common Corruptions and Perturbations\cite{hendrycks2018benchmarking}

 	\item \textbf{TinyImageNet:}\cite{le2015tiny} Tiny ImageNet contains 100000 images of 200 classes (500 for each class) downsized to 64×64 colored images. Each class has 500 training images, 50 validation images and 50 test images.

  	\item \textbf{TinyImageNet-LT:}\cite{le2015tiny}  This is a TinyImageNet dataset with an imbalance factor of 10\% used in our experiments.
   
   %Notably, the additional classes introduced in ImageNet-2010 serve as the open set within this dataset.
%	\item \textbf{ImageNet:}: This dataset spans 1000 object classes and contains 1,281,167 training images, 50,000 validation images and 100,000 test images. 
 % \item \textbf{ImageNet-LT:} The ImageNet Long-Tailed dataset is a specific subset extracted from the larger /dataset/imagenet dataset. It comprises 115.8K images distributed across 1000 categories. Each category is represented by a maximum of 1280 images and a minimum of 5 images follwing a protocol.
\end{itemize}

\subsection{Metrics for Evaluation of Calibration and Refinement} 
\label{sec:metrics}

\subsubsection{Background on Calibration}
Recall, a calibrated classifier outputs confidence scores that match the empirical frequency of correctness.  Formally, we define calibration in a classical supervised setting as follows. Let $\D=\{(\textbf{x}_i,y_i)\}_{i=1}^N$ denote a dataset consisting of $N$ samples from a joint distribution $\mathcal{P}(\X,\Y)$, where for each sample $\textbf{x}_i \in \X$ is the input and $y_i \in \Y = \{1, 2, ..., K\}$ is the ground-truth class label. Let $\sbf \in \mathbb{R}^K$, and $\sbf_i[y] = f_\theta(\bx_i)$ be the confidence that a \dnn, $f$, with model parameters $\theta$ predicts for a class $y$ on a given input $\bx_i$. The class, $\yh_i$, predicted by $f$ for a sample $\textbf{x}_i$ is computed as: 
\begin{equation}
    \yh_i = \argmax_{y \in \mathcal{Y}} \quad \mathbf{s}_i[y].
\end{equation}
The confidence for the predicted class is correspondingly computed as $\sh_i = \max_{y\in \Y} s_i[y]$. A model is said to be \emph{perfectly calibrated} \cite{guo2017calibration} when, for each sample $(x, y) \in \D$: 
\begin{equation}
    \mathbb{P}(y = y^* ~ | ~ \sbf[y] = s ) = s.
    \label{equ:calib2}
\end{equation}
Note that the perfect calibration requires each score value (and not only the $\widehat{s}$ ) to be calibrated. On the other hand, most calibration techniques focus only on the predicted class. That is, they only ensure that: $\mathbb{P}(\widehat{y}_i = y_i^* ~ | ~ \sh_i) = \sh_i$. 

\mypara{Expected Calibration Error (\ece)}
\ece is calculated by computing a weighted average of the differences in the confidence of the predicted class, and the accuracy of the samples, predicted with a particular confidence score \cite{ecepaper}:
\begin{equation}
    \ece = \sum_{i=1}^M \frac{B_i}{N} \Big\lvert A_i - C_i \Big\rvert.
    \label{equ:ece}
\end{equation}
%Here $N$ is the total number of samples, and the weighting is done on the basis of the fraction of samples in a given confidence bin/interval. Since the confidence values are in a continuous interval, for the computation of \ece, we divide the confidence range $[0,1]$ into $M$ equidistant bins, where $i^\text{th}$ bin is the interval $(\frac{i-1}{M}, \frac{i}{M}]$ in the confidence range, and $B_i$, represents the number of samples in the $i^\text{th}$ bin. Further, $A_i =\frac{1}{|B_i|}\sum_{j\in B_i} \mathbb{I} (\hat{y}_j=y_j)$, denotes accuracy for the samples in bin $B_i$, and $C_i=\frac{1}{|B_i|}\sum_{j: \sh_j \in B_i} \sh_j$, is the average predicted confidence of the samples, such that $\sh_j \in B_i$. The evaluation of \dnn calibration via \ece suffers from the following shortcomings: (a) \ece does not measure the calibration of all score values in the confidence vector, and (b) the metric is not differentiable, and hence can not be incorporated as a loss term during training procedure itself. Specifically, non-differentiablity arises due to binning samples into bins $B_i$.

Here $N$ is the total number of samples, and the weighting is done on the basis of the fraction of samples in a given confidence bin/interval. Since the confidence values are in a continuous interval, for the computation of \ece, we divide the confidence range $[0,1]$ into $M$ equidistant bins, where $i^{th}$ bin is the interval $(\frac{i-1}{M}, \frac{i}{M}]$ in the confidence range, and $B_i$, represents the number of samples in the $i^{th}$ bin. Further, $A_i =\frac{1}{|B_i|}\sum_{j \in B_i} \mathbb{I} (\hat{y}_j=y_j)$, denotes accuracy for the samples in bin $B_i$, and $C_i=\frac{1}{|B_i|}\sum_{j: \sh_j \in B_i} \sh_j$, is the average predicted confidence of the samples, such that $\sh_j \in B_i$. The evaluation of \dnn calibration via \ece suffers from the following shortcomings: (a) \ece does not measure the calibration of all score values in the confidence vector, and (b) the metric is not differentiable, and hence can not be incorporated as a loss term during training procedure itself. Specifically, non-differentiablity arises due to binning samples into bins $B_i$.

\subsubsection{Static Calibration Error (\sce)}
\sce is a recently proposed metric to measure calibration by \cite{nixon2019measuring}: 
\begin{equation}
    \sce = \frac{1}{K} \sum_{i=1}^M \sum_{j=1}^K \frac{B_{i,j}}{N} \big\lvert A_{i,j} - C_{i,j} \big\rvert,
    \label{equ:sce}
\end{equation}
where, $K$ denotes the number of classes, and $B_{{i,j}}$ denotes number of samples of the $j^{th}$ class in the $i^{th}$ bin. Further, $A_{i,j} = \frac{1}{B_{{i,j}}} \sum_{k \in B_{{i,j}}} \mathbb{I} (j = y_{k})$ is the accuracy for the samples of $j^{th}$ class in the $i^{th}$ bin, and $ C_{i,j} = \frac{1}{B_{i,j}} \sum_{k \in B_{i,j}} \mathbf{s}_k[j]$ or average confidence for the $j^{th}$ class in the $i^{th}$ bin. 
%\subsection{Classwise-\ece} \cite{Dirichlet} is another metric for measuring calibration in a multi-class setting, but is identical to Static Calibration Error (\sce). It is easy to see that \sce is a simple class-wise extension to \ece. Since \sce takes into account the whole confidence vector, it allows us to measure calibration of the non-predicted classes as well. Note that, similar to \ece, the metric \sce is also non-differentiable, and can not be used as a loss term during training. 

%\subsection{Class-$j$-\ece} 
%
%\cite{Dirichlet} has proposed to evaluate the calibration error of each class independent of other classes. This allows one to capture the contribution of a single class $j$ to the overall \sce (or classwise-\ece) error. We refer to this metric as class-$j$-\ece in our results/discussion. 

\subsubsection{Adaptive Calibration Error (\ace)}
\label{subsec:ace}
Adaptive Calibration Error (ACE) was proposed in \cite{nixon2019measuring} with the motivation to obtain the best estimate of the overall calibration error the metric should focus
on the regions where the predictions are made and focus less on regions with few predictions. 

\begin{equation}
 \ace =\frac{1}{K R} \sum_{k=1}^K \sum_{r=1}^R|\operatorname{acc}(r, k)-\operatorname{con} f(r, k)|
\end{equation}

\ace takes as input the predictions P (usually out of a softmax), correct labels, and a
number of ranges R. Here, $acc(r, k)$ and $conf(r, k)$ are the accuracy and confidence of adaptive calibration range $r$ for
class label $k$, respectively; and $N$ is the total number of data points. Calibration range $r$ defined by
the $\lfloor{N/R}\rfloor$ index of the sorted and thresholded predictions
\subsubsection{AUROC}
\label{subsec:auroc}

AUROC is the Area Under the Receiver Operating Characteristic curve and is computed as the area under the FPR against TPR curve by varying
the threshold. It has been widely used for assessing the performance of a classifier and also for \textbf{refinement}.  A traditional performance metric of a classifier is accuracy which measures the proportion of examples that are predicted correctly. However, in imbalanced data, accuracy can be misleading as the number of data points from one class is much larger
than the number of data points from another class \cite{auroc_survey}. AUROC is crucial for enhancing trust in classification as it measures the separability of two/more classes. 
\subsubsection{Smooth Calibration error (smECE)}
  ECE is highly dependant on different settings of bin-edges and the number of bins, making it a biased estimator of the true value. SmECE addresses non differentiability of ECE\cite{wang2023meta,minderer2021revisiting}. In order to obtain an effective representation of accuracy within a narrow confidence interval (corresponding to a single-example bin), we utilize the accuracy values of neighboring points around the selected example. These values are weighted based on their distances in the confidence space. Specifically, we express the soft estimate of accuracy for a single-example bin in a mathematical form of \texttt{smECE}:
 
 \begin{align}
 SmECE = \frac{1}{M} \sum_i^M |SACC(i) - conf(i)|
 \end{align}
where,
\begin{align}
    SACC (b_i) = \sum_{j=1}^M \pi(x_i). K(z_i,z_j) \nonumber
\end{align}
and
\begin{align}
     K(x_i,x'_j) = exp(\frac{-\|x_i -x'_j\|^2}{2h^2})
          \nonumber
\end{align}

where  $b_i$ is the bin housing example $x_i$ and $K(.,.)$ is a
chosen distance measure, for example, a Gaussian kernel
and $h$ is bandwidth, $z_i$ and $\pi(x_i)$ are the confidence and accuracy of $i^{th}$ example, $M$ is the class .

\subsection{Metrics for Out-of-distribution detection}
\label{sec:ood_metrics}

We denote a sample as positive if it is In distribution (\id), and negative for Out-of-distribution (\ood) samples. Accordingly, we use  TP, TN, FP, FN to denote true positives, true negatives, false positives, and false negatives, respectively. The `R' suffix in any of the above refers to the `rate' of the quantity. We also report test accuracy on \id classes. We compare the performance of various approaches using the following metrics: 
\begin{enumerate}%[leftmargin=*, topsep=0.2em, itemsep=-0.25em, label=\textbf{(\arabic*)}]
	\item \textbf{FPR@TPR95:} can be interpreted as the probability that a negative (out-of-distribution (\ood)) example is misclassified as positive (in-distribution (\id)) when the true positive rate (TPR) is as high as 95\%. The true positive rate can be computed by TPR = TP / (TP+FN), where TP and FN denote true positives and false negatives respectively. The false positive rate (FPR) can be computed by FPR = FP / (FP+TN), where FP and TN denote false positives and true negatives respectively.
	\item \textbf{AUROC:} is the Area Under the Receiver Operating Characteristic curve and computed as the area under the FPR against TPR curve.
	\item \textbf{Detection Error:} measures the minimum misclassification probability under the assumption that positive and negative examples have equal probability in the test set. 
\end{enumerate}

\section{Reproducibility}
 In the spirit of reproducible research, we intend to make the source code available post-acceptance. To aid reviewers, the source code for our approach is attached along with the supplemental material. Details of our setup and implementation of the baselines can be found at: \texttt{Code/README.md} folder. 

\section{Limitations}
Our approach employs supervised contrastive loss for refinement and calibration during training, necessitating a longer training duration compared to post hoc methods. Nonetheless, this trade-off is justifiable given the substantial improvements in reliability. In terms of inference, our approach achieves performance comparable to \sota while enhancing reliability.

\section{Broader Impact}
\label{sec:broaderImpact}

 Our proposed algorithm has the potential to be employed generically to enhance the reliability of \dnns. Through this exploration in proposing the framework \refcal, we have discovered a novel approach to calibrating and as well as refining models effectively without tradeoffs in accuracy. We present, for the first time, the evidence that model calibration and refinement can be achieved without sacrificing accuracy through contrastive learning. We provide mathematical insights into the proposed refinement loss and link it to supervised contrastive loss.

 %\FloatBarrier

\begin{table*}[]
\centering

  \resizebox{\textwidth}{!}{ 
\begin{tabular}{l|rrrrlrrrl}
\hline
Method  &
  \multicolumn{1}{l}{\begin{tabular}[c]{@{}l@{}}Stage 1 \\ epochs\end{tabular}} &
  \multicolumn{1}{l}{\begin{tabular}[c]{@{}l@{}}Stage 2 \\ epochs\end{tabular}} &
  \multicolumn{1}{l}{\begin{tabular}[c]{@{}l@{}}Stage 1 \\ BS\end{tabular}} &
  \multicolumn{1}{l}{\begin{tabular}[c]{@{}l@{}}Stage 2 \\ BS\end{tabular}} &
  Backbone &
  \multicolumn{1}{l}{\begin{tabular}[c]{@{}l@{}}Stage 1 \\ LR\end{tabular}} &
  \multicolumn{1}{l}{\begin{tabular}[c]{@{}l@{}}Stage 1 \\ $\tau$\end{tabular}} &
  \multicolumn{1}{l}{\begin{tabular}[c]{@{}l@{}}Stage 2 \\ LR\end{tabular}} &
  Parameters \\ \hline
\refcal+\texttt{CE}         & 1000 & 100 & 1024 & 512 & \texttt{Resnet50} & 0.5 & 0.1 & 5       & \default            \\ \hline
\refcal+\texttt{CE+TS} \cite{guo2017calibration}     & 1000 & 100 & 1024 & 512 & \texttt{Resnet50} & 0.5 & 0.1 & 5       & $\tau$ = 1.8  \\ \hline
\refcal+\texttt{LS} \cite{originallabelsmoothing} & 1000 & 100 & 1024 & 512 & \texttt{Resnet50} & 0.5 & 0.1 & 5       & $\alpha$ =0.1         \\ \hline
\refcal+\texttt{MMCE} \cite{kumarpaper}      & 1000 & 100 & 1024 & 512 & \texttt{Resnet50} & 0.5 & 0.1 & 5       & $\beta$ =2            \\ \hline
\refcal+\texttt{CRL} \cite{moon2020confidence}        & 1000 & 100 & 1024 & 512 & \texttt{Resnet50} & 0.5 & 0.1 & 5       & \default            \\ \hline
\refcal+\texttt{Adafocal} \cite{ghosh2023adafocal}  & 1000 & 100 & 1024 & 512 & \texttt{Resnet50} & 0.5 & 0.1 & 5       & \default            \\ \hline
\refcal+\texttt{MbLS} \cite{liu2022devil}      & 1000 & 100 & 1024 & 512 & \texttt{Resnet50} & 0.5 & 0.1 & 5       & $\alpha$ =0.1         \\ \hline
\refcal+\texttt{FL+MDCA} \cite{mdca}   & 1000 & 100 & 1024 & 512 & \texttt{Resnet50} & 0.5 & 0.1 & 5       & $\gamma$=2            \\ \hline
\refcal+\texttt{AUCM} \cite{yuan2021large}      & 1000 & 100 & 1024 & 512 & \texttt{Resnet50} & 0.5 & 0.1 & 5       & \default            \\ \hline
\refcal+\texttt{LogitNorm} \cite{wei2022mitigating} & 1000 & 100 & 1024 & 512 & \texttt{Resnet50} & 0.5 & 0.1 & 0.00005 & $\tau$ = 0.01 \\ \hline
\refcal+\texttt{CE+Mixup} \cite{thulasidasan2019mixup}  & 1000 & 100 & 1024 & 512 & \texttt{Resnet50} & 0.5 & 0.1 & 5       & \default            \\ \hline
\end{tabular}
}
\caption{The hyperparameters for the proposed \refcal training details for the \CIFART dataset. \textbf{Note:} BS: Batch Size, LR: Learning rate: $\tau$: Temperature parameter. \texttt{default}: default values used by the calibration/refinement technique; $\alpha$: smoothing weight \cite{originallabelsmoothing,liu2022devil}; $\beta$: calibration weight \cite{kumarpaper}; $\gamma$: focusing parameter \cite{focallosspaper}}
\label{tab:hyperparam_cifar10}

\end{table*}

%\FloatBarrier

% \usepackage{graphicx}
%\FloatBarrier
\begin{table*}[]
\centering

  \resizebox{\textwidth}{!}{ 
\begin{tabular}{l|rrrrlrrrl}
\hline
Methods &
  \multicolumn{1}{l}{\begin{tabular}[c]{@{}l@{}}Stage 1 \\ epochs\end{tabular}} &
  \multicolumn{1}{l}{\begin{tabular}[c]{@{}l@{}}Stage 2 \\ epochs\end{tabular}} &
  \multicolumn{1}{l}{\begin{tabular}[c]{@{}l@{}}Stage 1 \\ BS\end{tabular}} &
  \multicolumn{1}{l}{\begin{tabular}[c]{@{}l@{}}Stage 2 \\ BS\end{tabular}} &
  Backbone &
  \multicolumn{1}{l}{\begin{tabular}[c]{@{}l@{}}Stage 1 \\ LR\end{tabular}} &
  \multicolumn{1}{l}{\begin{tabular}[c]{@{}l@{}}Stage 1 \\ $\tau$\end{tabular}} &
  \multicolumn{1}{l}{\begin{tabular}[c]{@{}l@{}}Stage 2 \\ LR\end{tabular}} &
  Parameters \\ \hline
\refcal+\texttt{CE}         & 1000 & 100 & 1024 & 512 & \texttt{Resnet50} & 0.5 & 0.1 & 5       & \default           \\ \hline
\refcal+\texttt{CE+TS} \cite{guo2017calibration}       & 1000 & 100 & 1024 & 512 & \texttt{Resnet50} & 0.5 & 0.1 & 5       & $\tau$ = 3.5 \\ \hline
\refcal+\texttt{LS} \cite{originallabelsmoothing}         & 1000 & 100 & 1024 & 512 & \texttt{Resnet50} & 0.5 & 0.1 & 5       & $\alpha$ =0.01       \\ \hline
\refcal+\texttt{MMCE} \cite{kumarpaper}       & 1000 & 100 & 1024 & 512 & \texttt{Resnet50} & 0.5 & 0.1 & 5       & $\beta$ =2           \\ \hline
\refcal+\texttt{CRL} \cite{moon2020confidence}        & 1000 & 100 & 1024 & 512 & \texttt{Resnet50} & 0.5 & 0.1 & 5       & \default           \\ \hline
\refcal+\texttt{Adafocal} \cite{ghosh2023adafocal}   & 1000 & 100 & 1024 & 512 & \texttt{Resnet50} & 0.5 & 0.1 & 5       & \default           \\ \hline
\refcal+\texttt{MbLS} \cite{liu2022devil}       & 1000 & 100 & 1024 & 512 & \texttt{Resnet50} & 0.5 & 0.1 & 5       & $\alpha$ =0.01       \\ \hline
\refcal+\texttt{FL+MDCA} \cite{mdca}      & 1000 & 100 & 1024 & 512 & \texttt{Resnet50} & 0.5 & 0.1 & 5       & $\gamma$=2           \\ \hline
\refcal+\texttt{AUCM} \cite{yuan2021large}       & 1000 & 100 & 1024 & 512 & \texttt{Resnet50} & 0.5 & 0.1 & 5       & \default           \\ \hline
\refcal+\texttt{LogitNorm} \cite{wei2022mitigating} & 1000 & 100 & 1024 & 512 & \texttt{Resnet50} & 0.5 & 0.1 & 0.00005 & $\tau$ = 0.01       \\ \hline
\refcal+\texttt{CE+Mixup} \cite{thulasidasan2019mixup}   & 1000 & 100 & 1024 & 512 & \texttt{Resnet50} & 0.5 & 0.1 & 5       & \default           \\ \hline
\end{tabular}%
}
\caption{The hyperparameters for the proposed \refcal training details for the \CIFARTLT dataset. \textbf{Note:} BS: Batch Size, LR: Learning rate: $\tau$: Temperature parameter. \texttt{default}: default values used by the calibration/refinement technique; $\alpha$: smoothing weight \cite{originallabelsmoothing,liu2022devil}; $\beta$: calibration weight \cite{kumarpaper}; $\gamma$: focusing parameter \cite{focallosspaper}}
\label{tab:hyperparam_cifar10_imb}
\end{table*}
%\FloatBarrier

% \usepackage{graphicx}
%\FloatBarrier

\begin{table*}[]
\centering
%\resizebox{\columnwidth}{!}{%
  \resizebox{\textwidth}{!}{ 
\begin{tabular}{l|rrrrlrrrl}
\hline
Methods &
  \multicolumn{1}{l}{\begin{tabular}[c]{@{}l@{}}Stage 1 \\ epochs\end{tabular}} &
  \multicolumn{1}{l}{\begin{tabular}[c]{@{}l@{}}Stage 2 \\ epochs\end{tabular}} &
  \multicolumn{1}{l}{\begin{tabular}[c]{@{}l@{}}Stage 1 \\ BS\end{tabular}} &
  \multicolumn{1}{l}{\begin{tabular}[c]{@{}l@{}}Stage 2 \\ BS\end{tabular}} &
  Backbone &
  \multicolumn{1}{l}{\begin{tabular}[c]{@{}l@{}}Stage 1 \\ LR\end{tabular}} &
  \multicolumn{1}{l}{\begin{tabular}[c]{@{}l@{}}Stage 1 \\ $\tau$\end{tabular}} &
  \multicolumn{1}{l}{\begin{tabular}[c]{@{}l@{}}Stage 2 \\ LR\end{tabular}} &
  Parameters \\ \hline
\refcal+\texttt{CE}         & 1000 & 100 & 512 & 512 & \texttt{Resnet50} & 0.5 & 0.1 & 5   & \default           \\ \hline
\refcal+\texttt{CE+TS} \cite{guo2017calibration}      & 1000 & 100 & 512 & 512 & \texttt{Resnet50} & 0.5 & 0.1 & 5   & $\tau$ = 1.1 \\ \hline
\refcal+\texttt{LS} \cite{originallabelsmoothing}         & 1000 & 100 & 512 & 512 & \texttt{Resnet50} & 0.5 & 0.1 & 5   & $\alpha$ =0.01       \\ \hline
\refcal+\texttt{MMCE} \cite{kumarpaper}       & 1000 & 100 & 512 & 512 & \texttt{Resnet50} & 0.5 & 0.1 & 5   & $\beta$ =2           \\ \hline
\refcal+\texttt{CRL} \cite{moon2020confidence}        & 1000 & 100 & 512 & 512 & \texttt{Resnet50} & 0.5 & 0.1 & 5   & \default           \\ \hline
\refcal+\texttt{Adafocal} \cite{ghosh2023adafocal}   & 1000 & 100 & 512 & 512 & \texttt{Resnet50} & 0.5 & 0.1 & 5   & \default           \\ \hline
\refcal+\texttt{MbLS} \cite{liu2022devil}      & 1000 & 100 & 512 & 512 & \texttt{Resnet50} & 0.5 & 0.1 & 5   & $\alpha$ =0.01       \\ \hline
\refcal+\texttt{FL+MDCA} \cite{mdca}    & 1000 & 100 & 512 & 512 & \texttt{Resnet50} & 0.5 & 0.1 & 5   & $\gamma$=2           \\ \hline
\refcal+\texttt{AUCM} \cite{yuan2021large}       & 1000 & 100 & 512 & 512 & \texttt{Resnet50} & 0.5 & 0.1 & 5   & \default           \\ \hline
\refcal+\texttt{LogitNorm} \cite{wei2022mitigating} & 1000 & 100 & 512 & 512 & \texttt{Resnet50} & 0.5 & 0.1 & 0.5 & $\tau$ = 0.01       \\ \hline
\refcal+\texttt{CE+Mixup} \cite{thulasidasan2019mixup}   & 1000 & 100 & 512 & 512 & \texttt{Resnet50} & 0.5 & 0.1 & 5   & \default           \\ \hline
\end{tabular}%
}
\caption{The hyperparameters for the proposed \refcal training details for the \CIFARH dataset. \textbf{Note:} BS: Batch Size, LR: Learning rate: $\tau$: Temperature parameter. \texttt{default}: default values used by the calibration/refinement technique; $\alpha$: smoothing weight \cite{originallabelsmoothing,liu2022devil}; $\beta$: calibration weight \cite{kumarpaper}; $\gamma$: focusing parameter \cite{focallosspaper}}
\label{tab:hyperparam_cifar100}
\end{table*}

%\FloatBarrier

% \usepackage{graphicx}
%\FloatBarrier

\begin{table*}[]
\centering

  \resizebox{\textwidth}{!}{ 
\begin{tabular}{l|rrrrlrrrl}
\hline
Methods &
  \multicolumn{1}{l}{\begin{tabular}[c]{@{}l@{}}Stage 1 \\ epochs\end{tabular}} &
  \multicolumn{1}{l}{\begin{tabular}[c]{@{}l@{}}Stage 2 \\ epochs\end{tabular}} &
  \multicolumn{1}{l}{\begin{tabular}[c]{@{}l@{}}Stage 1 \\ BS\end{tabular}} &
  \multicolumn{1}{l}{\begin{tabular}[c]{@{}l@{}}Stage 2 \\ BS\end{tabular}} &
  Backbone &
  \multicolumn{1}{l}{\begin{tabular}[c]{@{}l@{}}Stage 1 \\ LR\end{tabular}} &
  \multicolumn{1}{l}{\begin{tabular}[c]{@{}l@{}}Stage 1 \\ $\tau$\end{tabular}} &
  \multicolumn{1}{l}{\begin{tabular}[c]{@{}l@{}}Stage 2 \\ LR\end{tabular}} &
  Parameters \\ \hline
\refcal+\texttt{CE}         & 1000 & 100 & 512 & 512 & \texttt{Resnet50} & 0.5 & 0.1 & 5     & \default           \\ \hline
\refcal+\texttt{CE+TS} \cite{guo2017calibration}      & 1000 & 100 & 512 & 512 & \texttt{Resnet50} & 0.5 & 0.1 & 5     & $\tau$ = 1.5 \\ \hline
\refcal+\texttt{LS} \cite{originallabelsmoothing}         & 1000 & 100 & 512 & 512 & \texttt{Resnet50} & 0.5 & 0.1 & 5     & $\alpha$ =0.01       \\ \hline
\refcal+\texttt{MMCE} \cite{kumarpaper}       & 1000 & 100 & 512 & 512 & \texttt{Resnet50} & 0.5 & 0.1 & 5     & $\beta$ =2           \\ \hline
\refcal+\texttt{CRL} \cite{moon2020confidence}        & 1000 & 100 & 512 & 512 & \texttt{Resnet50} & 0.5 & 0.1 & 5     & \default           \\ \hline
\refcal+\texttt{Adafocal} \cite{ghosh2023adafocal}   & 1000 & 100 & 512 & 512 & \texttt{Resnet50} & 0.5 & 0.1 & 0.005 & \default           \\ \hline
\refcal+\texttt{MbLS} \cite{liu2022devil}       & 1000 & 100 & 512 & 512 & \texttt{Resnet50} & 0.5 & 0.1 & 5     & $\alpha$ =0.01       \\ \hline
\refcal+\texttt{FL+MDCA} \cite{mdca}    & 1000 & 100 & 512 & 512 & \texttt{Resnet50} & 0.5 & 0.1 & 5     & $\gamma$ =2           \\ \hline
\refcal+\texttt{AUCM} \cite{yuan2021large}       & 1000 & 100 & 512 & 512 & \texttt{Resnet50} & 0.5 & 0.1 & 5     & \default           \\ \hline
\refcal+\texttt{LogitNorm} \cite{wei2022mitigating} & 1000 & 100 & 512 & 256 & \texttt{Resnet50} & 0.5 & 0.1 & 0.05  & $\tau$ = 0.01       \\ \hline
\refcal+\texttt{CE+Mixup} \cite{thulasidasan2019mixup}   & 1000 & 100 & 512 & 512 & \texttt{Resnet50} & 0.5 & 0.1 & 5     & \default           \\ \hline
\end{tabular}%
}
\caption{The hyperparameters for the proposed \refcal training details for the \CIFARHLT dataset. \textbf{Note:} BS: Batch Size, LR: Learning rate: $\tau$: Temperature parameter. \texttt{default}: default values used by the calibration/refinement technique; $\alpha$: smoothing weight \cite{originallabelsmoothing,liu2022devil}; $\beta$: calibration weight \cite{kumarpaper}; $\gamma$: focusing parameter \cite{focallosspaper}}
\label{tab:hyperparam_cifar100_imb}
\end{table*}
%\FloatBarrier

% \usepackage{graphicx}
\begin{table*}[]
\centering
  \resizebox{\textwidth}{!}{ 
\begin{tabular}{l|rrrrlrrrl}
\hline
Methods &
  \multicolumn{1}{l}{\begin{tabular}[c]{@{}l@{}}Stage 1 \\ epochs\end{tabular}} &
  \multicolumn{1}{l}{\begin{tabular}[c]{@{}l@{}}Stage 2 \\ epochs\end{tabular}} &
  \multicolumn{1}{l}{\begin{tabular}[c]{@{}l@{}}Stage 1 \\ BS\end{tabular}} &
  \multicolumn{1}{l}{\begin{tabular}[c]{@{}l@{}}Stage 2 \\ BS\end{tabular}} &
  Backbone &
  \multicolumn{1}{l}{\begin{tabular}[c]{@{}l@{}}Stage 1 \\ LR\end{tabular}} &
  \multicolumn{1}{l}{\begin{tabular}[c]{@{}l@{}}Stage 1 \\ $\tau$\end{tabular}} &
  \multicolumn{1}{l}{\begin{tabular}[c]{@{}l@{}}Stage 2 \\ LR\end{tabular}} &
  Parameters \\ \hline
\refcal+\texttt{CE}         & 1000 & 100 & 512 & 512 & \texttt{Resnet50} & 0.5 & 0.1 & 0.5 & \default           \\ \hline
\refcal+\texttt{CE+TS} \cite{guo2017calibration}      & 1000 & 100 & 512 & 512 & \texttt{Resnet50} & 0.5 & 0.1 & 5   & $\tau$ = 1.5 \\ \hline
\refcal+\texttt{LS} \cite{originallabelsmoothing}         & 1000 & 100 & 512 & 512 & \texttt{Resnet50} & 0.5 & 0.1 & 5   & $\alpha$ =0.01       \\ \hline
\refcal+\texttt{MMCE} \cite{kumarpaper}       & 1000 & 100 & 512 & 512 & \texttt{Resnet50} & 0.5 & 0.1 & 5   & $\beta$ =2           \\ \hline
\refcal+\texttt{CRL} \cite{moon2020confidence}        & 1000 & 100 & 512 & 512 & \texttt{Resnet50} & 0.5 & 0.1 & 0.5 & \default           \\ \hline
\refcal+\texttt{Adafocal} \cite{ghosh2023adafocal}   & 1000 & 100 & 512 & 512 & \texttt{Resnet50} & 0.5 & 0.1 & 5   & \default           \\ \hline
\refcal+\texttt{MbLS} \cite{liu2022devil}       & 1000 & 100 & 512 & 512 & \texttt{Resnet50} & 0.5 & 0.1 & 5   & $\alpha$ =0.01       \\ \hline
\refcal+\texttt{FL+MDCA} \cite{mdca}    & 1000 & 100 & 512 & 512 & \texttt{Resnet50} & 0.5 & 0.1 & 5   & $\gamma$=2           \\ \hline
\refcal+\texttt{AUCM} \cite{yuan2021large}       & 1000 & 100 & 512 & 512 & \texttt{Resnet50} & 0.5 & 0.1 & 5   & \default           \\ \hline
\refcal+\texttt{LogitNorm} \cite{wei2022mitigating} & 1000 & 100 & 512 & 256 & \texttt{Resnet50} & 0.5 & 0.1 & 0.5 & $\tau$ = 0.01       \\ \hline
\refcal+\texttt{CE+Mixup} \cite{thulasidasan2019mixup}   & 1000 & 100 & 512 & 512 & \texttt{Resnet50} & 0.5 & 0.1 & 5   & \default           \\ \hline
\end{tabular}%
}
\caption{The hyperparameters for the proposed \refcal training details for the \TINY dataset. \textbf{Note:} BS: Batch Size, LR: Learning rate: $\tau$: Temperature parameter. \texttt{default}: default values used by the calibration/refinement technique; $\alpha$: smoothing weight \cite{originallabelsmoothing,liu2022devil}; $\beta$: calibration weight \cite{kumarpaper}; $\gamma$: focusing parameter \cite{focallosspaper}}
\label{tab:hyperparam_tiny_imagenet}
\end{table*}
% \usepackage{graphicx}

%\FloatBarrier

\begin{table*}[]
\centering
  \resizebox{\textwidth}{!}{ 
\begin{tabular}{l|rrrrlrrrl}
\hline
Methods &
  \multicolumn{1}{l}{\begin{tabular}[c]{@{}l@{}}Stage 1 \\ epochs\end{tabular}} &
  \multicolumn{1}{l}{\begin{tabular}[c]{@{}l@{}}Stage 2 \\ epochs\end{tabular}} &
  \multicolumn{1}{l}{\begin{tabular}[c]{@{}l@{}}Stage 1 \\ BS\end{tabular}} &
  \multicolumn{1}{l}{\begin{tabular}[c]{@{}l@{}}Stage 2 \\ BS\end{tabular}} &
  Backbone &
  \multicolumn{1}{l}{\begin{tabular}[c]{@{}l@{}}Stage 1 \\ LR\end{tabular}} &
  \multicolumn{1}{l}{\begin{tabular}[c]{@{}l@{}}Stage 1 \\ $\tau$\end{tabular}} &
  \multicolumn{1}{l}{\begin{tabular}[c]{@{}l@{}}Stage 2 \\ LR\end{tabular}} &
  Parameters \\ \hline
\refcal+\texttt{CE}         & 1000 & 100 & 512 & 512 & \texttt{Resnet50} & 0.5 & 0.1 & 5     & \default           \\ \hline
\refcal+\texttt{CE+TS} \cite{guo2017calibration}      & 1000 & 100 & 512 & 512 & \texttt{Resnet50} & 0.5 & 0.1 & 5     & $\tau$ = 1.5 \\ \hline
\refcal+\texttt{LS} \cite{originallabelsmoothing}         & 1000 & 100 & 512 & 512 & \texttt{Resnet50} & 0.5 & 0.1 & 5     & $\alpha$ =0.01       \\ \hline
\refcal+\texttt{MMCE} \cite{kumarpaper}       & 1000 & 100 & 512 & 512 & \texttt{Resnet50} & 0.5 & 0.1 & 5     & $\beta$ =2           \\ \hline
\refcal+\texttt{CRL} \cite{moon2020confidence}        & 1000 & 100 & 512 & 512 & \texttt{Resnet50} & 0.5 & 0.1 & 0.5   & \default           \\ \hline
\refcal+\texttt{Adafocal} \cite{ghosh2023adafocal}   & 1000 & 100 & 512 & 512 & \texttt{Resnet50} & 0.5 & 0.1 & 0.005 & \default           \\ \hline
\refcal+\texttt{MbLS} \cite{liu2022devil}      & 1000 & 100 & 512 & 512 & \texttt{Resnet50} & 0.5 & 0.1 & 5     & $\alpha$ =0.01       \\ \hline
\refcal+\texttt{FL+MDCA} \cite{mdca}    & 1000 & 100 & 512 & 256 & \texttt{Resnet50} & 0.5 & 0.1 & 0.5   & $\gamma$=2           \\ \hline
\refcal+\texttt{AUCM} \cite{yuan2021large}       & 1000 & 100 & 512 & 512 & \texttt{Resnet50} & 0.5 & 0.1 & 5     & \default           \\ \hline
\refcal+\texttt{LogitNorm} \cite{wei2022mitigating} & 1000 & 100 & 512 & 256 & \texttt{Resnet50} & 0.5 & 0.1 & 0.05  & $\tau$ = 0.01       \\ \hline
\refcal+\texttt{CE+Mixup} \cite{thulasidasan2019mixup}   & 1000 & 100 & 512 & 512 & \texttt{Resnet50} & 0.5 & 0.1 & 5     & \default           \\ \hline
\end{tabular}%
}
\caption{The hyperparameters for the proposed \refcal training details for the \TINYLT dataset. \textbf{Note:} BS: Batch Size, LR: Learning rate: $\tau$: Temperature parameter. \texttt{default}: default values used by the calibration/refinement technique; $\alpha$: smoothing weight \cite{originallabelsmoothing,liu2022devil}; $\beta$: calibration weight \cite{kumarpaper}; $\gamma$: focusing parameter \cite{focallosspaper}}
\label{tab:hyperparam_tiny_imagenet_imb}
\end{table*}

%\FloatBarrier

% \usepackage{graphicx}
%\FloatBarrier

\begin{table*}[]
\centering
  \resizebox{\textwidth}{!}{ 
\begin{tabular}{l|rrrrlrrrl}
\hline
Methods &
  \multicolumn{1}{l}{\begin{tabular}[c]{@{}l@{}}Stage 1 \\ epochs\end{tabular}} &
  \multicolumn{1}{l}{\begin{tabular}[c]{@{}l@{}}Stage 2 \\ epochs\end{tabular}} &
  \multicolumn{1}{l}{\begin{tabular}[c]{@{}l@{}}Stage 1 \\ BS\end{tabular}} &
  \multicolumn{1}{l}{\begin{tabular}[c]{@{}l@{}}Stage 2 \\ BS\end{tabular}} &
  Backbone &
  \multicolumn{1}{l}{\begin{tabular}[c]{@{}l@{}}Stage 1 \\ LR\end{tabular}} &
  \multicolumn{1}{l}{\begin{tabular}[c]{@{}l@{}}Stage 1 \\ $\tau$\end{tabular}} &
  \multicolumn{1}{l}{\begin{tabular}[c]{@{}l@{}}Stage 2 \\ LR\end{tabular}} &
  Parameters \\ \hline
\refcal+\texttt{CE}         & 1000 & 100 & 1024 & 512 & \texttt{Resnet50} & 0.05 & 0.1 & 0.005 & \default           \\ \hline
\refcal+\texttt{CE+TS} \cite{guo2017calibration}      & 1000 & 100 & 1024 & 512 & \texttt{Resnet50} & 0.05 & 0.1 & 0.005 & $\tau$ = 3.5 \\ \hline
\refcal+\texttt{LS} \cite{originallabelsmoothing}         & 1000 & 100 & 1024 & 512 & \texttt{Resnet50} & 0.05 & 0.1 & 0.005 & $\alpha$ =0.01       \\ \hline
\refcal+\texttt{MMCE} \cite{kumarpaper}    & 1000 & 100 & 1024 & 512 & \texttt{Resnet50} & 0.05 & 0.1 & 0.005 & $\beta$ =2           \\ \hline
\refcal+\texttt{CRL} \cite{moon2020confidence}        & 1000 & 100 & 1024 & 512 & \texttt{Resnet50} & 0.05 & 0.1 & 0.005 & \default           \\ \hline
\refcal+\texttt{Adafocal} \cite{ghosh2023adafocal}   & 1000 & 100 & 1024 & 512 & \texttt{Resnet50} & 0.05 & 0.1 & 0.005 & \default           \\ \hline
\refcal+\texttt{MbLS} \cite{liu2022devil}       & 1000 & 100 & 1024 & 512 & \texttt{Resnet50} & 0.05 & 0.1 & 0.005 & $\alpha$ =0.01       \\ \hline
\refcal+\texttt{FL+MDCA} \cite{mdca}    & 1000 & 100 & 1024 & 256 & \texttt{Resnet50} & 0.05 & 0.1 & 0.005 & $\gamma$=2           \\ \hline
\refcal+\texttt{AUCM} \cite{yuan2021large}        & 1000 & 100 & 1024 & 512 & \texttt{Resnet50} & 0.05 & 0.1 & 0.005 & \default           \\ \hline
\refcal+\texttt{LogitNorm} \cite{wei2022mitigating} & 1000 & 100 & 1024 & 256 & \texttt{Resnet50} & 0.05 & 0.1 & 0.005 & $\tau$ = 0.01       \\ \hline
\refcal+\texttt{CE+Mixup} \cite{thulasidasan2019mixup}   & 1000 & 100 & 1024 & 512 & \texttt{Resnet50} & 0.05 & 0.1 & 0.005 & \default           \\ \hline
\end{tabular}%
}
\caption{The hyperparameters for the proposed RefCal architecture for the \STLT converted into binary setting as mentioned in main text. \textbf{Note:} BS: Batch Size, LR: Learning rate: $\tau$: Temperature parameter. \texttt{default}: default values used by the calibration/refinement technique; $\alpha$: smoothing weight \cite{originallabelsmoothing,liu2022devil}; $\beta$: calibration weight \cite{kumarpaper}; $\gamma$: focusing parameter \cite{focallosspaper}}
\label{tab:hyperparam_stl10_binary}
\end{table*}

%\FloatBarrier

\clearpage
\newpage
\end{document}